\documentclass[twoside]{article}

\usepackage[preprint]{aistats2026}
\usepackage[round]{natbib}

\usepackage{amsmath,amsfonts,bm}

\def\eqref#1{equation~\ref{#1}}

\def\1{\bm{1}}
\newcommand{\train}{\mathcal{D}}

\def\eps{{\epsilon}}
\def\inn#1{\langle #1 \rangle}

\def\rt{{\textnormal{t}}}

\def\rvx{{\mathbf{x}}}

\def\rmR{{\mathbf{R}}}

\def\rmW{{\mathbf{W}}}

\def\vw{{\bm{w}}}

\def\vy{{\bm{y}}}
\def\vz{{\bm{z}}}

\def\mI{{\bm{I}}}

\def\mX{{\bm{X}}}
\def\mY{{\bm{Y}}}

\DeclareMathAlphabet{\mathsfit}{\encodingdefault}{\sfdefault}{m}{sl}
\SetMathAlphabet{\mathsfit}{bold}{\encodingdefault}{\sfdefault}{bx}{n}
\newcommand{\tens}[1]{\bm{\mathsfit{#1}}}

\def\tP{{\tens{P}}}
\def\tQ{{\tens{Q}}}

\def\gD{{\mathcal{D}}}
\def\gE{{\mathcal{E}}}
\def\gF{{\mathcal{F}}}

\def\gH{{\mathcal{H}}}

\def\gJ{{\mathcal{J}}}

\def\gL{{\mathcal{L}}}
\def\gM{{\mathcal{M}}}
\def\gN{{\mathcal{N}}}

\def\gP{{\mathcal{P}}}
\def\gQ{{\mathcal{Q}}}
\def\gR{{\mathcal{R}}}
\def\gS{{\mathcal{S}}}
\def\gT{{\mathcal{T}}}

\def\sP{{\mathbb{P}}}
\def\sQ{{\mathbb{Q}}}

\newcommand{\E}{\mathbb{E}}
\newcommand{\Ls}{\mathcal{L}}
\newcommand{\R}{\mathbb{R}}

\DeclareMathOperator*{\argmin}{arg\,min}
\DeclareMathOperator*{\arginf}{arg\,inf}

\usepackage[utf8]{inputenc} 
\usepackage[T1]{fontenc}    
\usepackage[colorlinks=true,citecolor=blue]{hyperref}       
\usepackage{url}            
\usepackage{booktabs}       
\usepackage{amsfonts}       
\usepackage{nicefrac}       
\usepackage{microtype}      
\usepackage[svgnames]{xcolor}         
\usepackage[breakable, theorems, skins]{tcolorbox}
\usepackage{scalerel}
\usepackage{stackengine,wasysym}

\newcommand\reallywidetilde[1]{\ThisStyle{%
  \setbox0=\hbox{$\SavedStyle#1$}%
  \stackengine{-.1\LMpt}{$\SavedStyle#1$}{%
    \stretchto{\scaleto{\SavedStyle\mkern.2mu\AC}{.5150\wd0}}{.6\ht0}%
  }{O}{c}{F}{T}{S}%
}}
\newcommand{\pmval}[2]{#1\,{\scriptsize$\pm$\,#2}}

\DeclareRobustCommand{\mybox}[2][gray!20]{%
\begin{tcolorbox}[   
        breakable,
        left=0pt,
        right=0pt,
        top=0pt,
        bottom=0pt,
        colback=#1,
        colframe=#1,
        width=\dimexpr\columnwidth\relax, 
        enlarge left by=0mm,
        boxsep=5pt,
        arc=0pt,outer arc=0pt,
        ]
        #2
\end{tcolorbox}
}
\usepackage{hyperref}
\usepackage{url}
\usepackage{breqn}

\usepackage{amsthm}
\usepackage{url}
\usepackage{lipsum}
\usepackage{graphicx}
\usepackage{adjustbox}
\usepackage{wrapfig}
\usepackage{thmtools}
\usepackage{algorithm}
\usepackage{algpseudocode}

\renewcommand*{\thefootnote}{\fnsymbol{footnote}}

\newcommand{\tablestyle}[2]{\setlength{\tabcolsep}{#1}\renewcommand{\arraystretch}{#2}\centering\footnotesize}
\graphicspath{{./figures/}}
\newtheorem{theorem}{Theorem}

\newtheorem{proposition}{Proposition}
\newtheorem{lemma}{Lemma}

\newcommand{\fsl}{N$\mathcal{M}$}

\begin{document}

%

%

\twocolumn[

\aistatstitle{Flows and Diffusions on the Neural Manifold}

\aistatsauthor{ Daniel Saragih${}^\dagger$ \And Deyu Cao${}^{*,\, \dagger}$ \And Tejas Balaji${}^*$ }

\aistatsaddress{ Queen's University \and Vector Institute \And  University of Tokyo \And University of Toronto } ]

\def\thefootnote{*}\footnotetext{Equal contribution.}\def\thefootnote{\arabic{footnote}}
\def\thefootnote{$\dagger$}\footnotetext{Work done while at the University of Toronto.}
\def\thefootnote{\arabic{footnote}}

\begin{abstract}
 Diffusion and flow-based generative models have achieved remarkable success in domains such as image synthesis, video generation, and natural language modeling. In this work, we extend these advances to \textit{weight space learning} by leveraging recent techniques to incorporate structural priors derived from optimization dynamics. Central to our approach is modeling the trajectory induced by gradient descent as a trajectory inference problem. We unify several trajectory inference techniques towards matching a gradient flow, providing a theoretical framework for treating optimization paths as inductive bias. We further explore architectural and algorithmic choices, including reward fine-tuning by adjoint matching, the use of autoencoders for latent weight representation, conditioning on task-specific context data, and adopting informative source distributions such as Kaiming uniform. Experiments demonstrate that our method matches or surpasses baselines in generating in-distribution weights, improves initialization for downstream training, and supports fine-tuning to enhance performance. Finally, we illustrate a practical application in safety-critical systems: detecting harmful covariate shifts, where our method outperforms the closest comparable baseline.
\end{abstract}

\section{Introduction}
\label{sec:intro}
Flow matching (FM) \citep{albergo2023building, lipmanFlowMatchingGenerative2023, liu2023flow} is a prominent fixture in generative modeling tasks from imaging \citep{lipmanFlowMatchingGenerative2023, tongImprovingGeneralizingFlowbased2024, esser2024scalingrectifiedflowtransformers, liu2024instaflow} to language \citep{gat2024discrete, shaul2024flowmatchinggeneraldiscrete, campbell2024generative}. However, its application to neural network weights has not been explored. By leveraging the principled, yet versatile training of FM, we aim to generate task-specific weights on novel tasks.

The natural question is: why generate task-specific weights instead of relying on conventional training methods? One compelling reason is efficiency. If we can train a meta-model to produce classifiers conditioned only on the evaluation dataset, then generating weights reduces to a single inference pass of our flow or diffusion model. This motivation parallels recent work in zero- and few-shot learning \citep{zhang2024metadiff, soroDiffusionbasedNeuralNetwork2024}, where generalization to new tasks is achieved with minimal or no training. Further on efficiency, conditionally generated weights can also serve as a strong, head-start initialization for downstream fine-tuning, which we later evaluate on corrupted datasets. This approach is especially practical when training a large number of smaller networks, such as in applications involving implicit neural representations \citep{essakine2025where}. Finally, we argue that learning to generate neural network weights opens a new perspective: it allows us to reinterpret diverse problems as questions on weight space. We exemplify this in Sec. \ref{sec:experiments} through an application to detecting harmful covariate shift.

\begin{figure*}
    \centering
    \label{fig:pipeline}
    \includegraphics[width=0.8\textwidth]{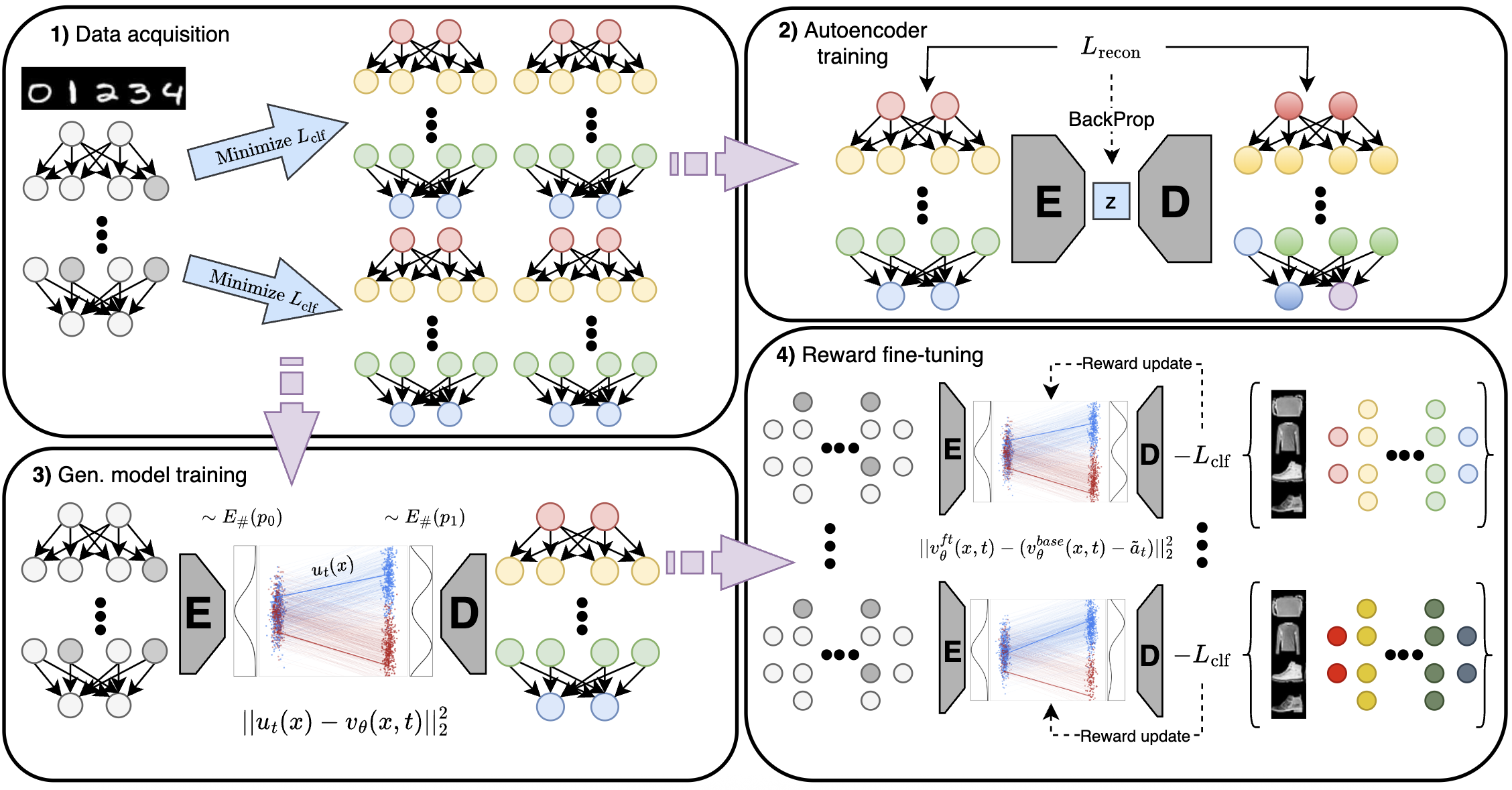}
    \caption{\textbf{Example unconditional pipeline.} \textbf{(1)} Base model pre-training, shown here on MNIST, producing checkpoints across epochs. \textbf{(2) Optional:} variational autoencoder training with a weight-space reconstruction objective. \textbf{(3)} Generative meta-model training; here we illustrate \textit{unconditional} \fsl-CFM w/ (trained) VAE (our default \fsl-CFM is on weight space directly) using the weight initialization from \textbf{(1)} as $p_0$. \textbf{(4) Optional:} reward fine-tuning via adjoint matching where $r(\boldsymbol{\cdot}) = -L_{\text{clf}}(\mX_{\mathrm{FashionMNIST}}; \, \boldsymbol{\cdot})$, steering the \textit{trained} meta-model towards generating FashionMNIST classifiers.}
    \vspace{-1em}
\end{figure*}

In this paper, we introduce flow matching as a new class of methods for generating neural network weights, designed to incorporate structural priors such as training trajectories and source distributions. Under this framework, we may cast our goal as one of \textit{trajectory inference} \citep{lavenant2024towardstrajectory}, reconstructing the continuous-time dynamics $t \mapsto p_t^*$ given easy sampling from the marginal distributions $(p_{t_k}^*)_{k=0}^K$. In practice, temporal observations are sparse, necessitating methods that can sensibly interpolate between observed timepoints, often leveraging biases in data. Indeed, we further ground our approach in the \textit{manifold hypothesis} \citep{bengio2013representation}, which posits that natural data lies on a low-dimensional submanifold of the ambient space, representing a bias that may be incorporated as a prior. Drawing on the Lottery Ticket Hypothesis \citep{frankle2018the, zhang2021validatinglottery, liu2024surveylotterytickethypothesis} as well as the body of work on pruning \citep{cheng2024asurveyonpruning}, we extend this intuition to weight space: neural networks themselves tend to lie on a low-dimensional structure, which we refer to as the \textit{neural manifold} (\fsl)\footnote{
We note that modern deep neural network architectures exhibit parameter symmetries \citep{hechtnielsen1990algebraicstructure, chen1993geometryfeedforward} that may be exploited to reduce the size of weight space. In particular, the neural manifold can be viewed as a \textit{quotient set} where we identify distinct weights perturbed by symmetric transformations. As our focus is on generative techniques, the main paper limits itself to the un-identified space, deferring a discussion and modest experiments to App. \ref{app:futher-exps}.
}. We will make use of these observations to motivate the various experiments conducted later.

In this work, we take preliminary steps toward understanding flows and diffusions on the neural manifold. Our contributions are: \textbf{(1)} unifying and characterizing methods for approximating gradient descent trajectories, which enables principled modeling of structural priors; \textbf{(2)} leveraging these insights to design flow- and diffusion-based approaches for generating neural network weights that not only match or exceed conventionally trained models on in-distribution tasks but also provide stronger initializations for downstream training and allow conditioning on context data to retrieve weights from a distribution pre-trained across datasets; \textbf{(3)} incorporating a fine-tuning mechanism grounded in adjoint matching \citep{domingo-enrichAdjointMatchingFinetuning2024} to adapt pretrained flows efficiently; and \textbf{(4)} demonstrating an application to covariate shift detection, where our approach outperforms the closest comparable baseline. While we view these contributions as a step toward reinterpreting learning problems on weight space, we do not claim to resolve every aspect of generative modeling in this domain. In particular, our scope is limited to three directions—exploiting structural priors (e.g., training trajectories and source distributions), exploring informative priors, and incorporating reward fine-tuning for downstream tasks—implemented within the flow matching framework. Broader challenges such as scaling to very large models or comprehensive accounting of parameter symmetries remain important directions for future work.

\section{Background \& Related Work}
\label{sec:prelims}

\paragraph{Conditional flow models.} \citet{chenNeuralOrdinaryDifferential2019} first introduced continuous normalizing flows as an effective data generation process through modeling dynamics. Simulation-free methods improve on this concept by simplifying the training objective \citep{albergo2023building, lipmanFlowMatchingGenerative2023, liu2023flow}. Following the formulation of \citet{lipmanFlowMatchingGenerative2023}, given random variables $\bar \rvx_0 \sim p_0$ and $\bar \rvx_1 \sim p_1$ a data distribution, define a reference flow $\bar \rvx = (\bar \rvx_t)_{t \in [0, 1]}$ where $\bar \rvx_t = \beta_t \bar \rvx_0 + \alpha_t \bar \rvx_1$ with the constraint that $\alpha_0 = \beta_1 = 0$ and $\alpha_1 = \beta_0 = 1$. The aim of flow modeling is to learn a path $\rvx = (\rvx_t)_{t \in [0, 1]}$ which has the same marginal distribution as $\bar \rvx$. To make this a feasible task, we describe this process as an ODE: $d\rvx_t = v(\rvx_t, t) dt$ where $\rvx_0 \sim \mathcal{N} (0 , \mI)$. Training proceeds by first parameterizing $v(\rvx_t, t)$ by a neural network $\theta$ and matching the reference flow velocity, i.e. $u(\rvx_t, t) := \frac{d}{d t} \bar \rvx_t$. This would, however, be an unfeasible training objective, therefore, we condition on samples from the distribution $\rvx_1 \sim p_1$ and train
\begin{equation}
    \label{eq:cfm}
    \Ls_{\text{cfm}}(\theta) = \E_{\rt, \rvx_1, \rvx_t}||v_\theta(\rvx_t, t) - u(\rvx_t, t \mid \rvx_1)||.
\end{equation}
\citet{lipmanFlowMatchingGenerative2023} proved that this loss produces the same gradients as the marginal loss, thus optimizing it will result in convergence to the reference $u(\rvx_t, t)$. Moreover, we can always marginalize an independent conditioning variable $\vy$ on $v_\theta, u$ -- this will serve as our context conditioning vector.

\paragraph{Modeling probability paths.} The straight line paths in the CFM framework can be situated in the broader framework of matching interpolants $(\rvx_t)_{t \in [0, 1]}$ to minimize a given energy function. Typically, the kinetic energy $\gE(\rvx_t, \dot \rvx_t) = \E_{\rt \sim U[0, 1]} ||{\dot \rvx_t}||^2$ is minimized \citep{shaul2023kineticoptimal}, however, more general energies have been considered \citep{neklyudov2023actionmatching, neklyudov2024wassersteinlagrangian, kapusniakMetricFlowMatching2024, liu2024generalized} to broaden the class of learnable paths. Of interest to us is the concept of finding interpolants that depend on prior knowledge as provided through samples. For instance, Metric Flow Matching \citep{kapusniakMetricFlowMatching2024} learns parametric interpolants that minimize 
    $\E_{(\rvx_0, \rvx_1) \sim \pi} \gE_g(\rvx_t) = \E_{\rt, (\rvx_0, \rvx_1)} ||\dot \rvx_t||_{g(\rvx_t)}^2$,
where $\pi$ is some distribution on the product, typically $p_0 \otimes p_1$, and $g$ is a data-dependent Riemannian metric on the ambient space $\R^d$. Alternatively, \citet{rohbeck2025modeling} proposed a more stable option using cubic splines for multi-marginal flow matching. Most related to our setting is the task of modeling distributions evolved via gradient flow: $\frac{d}{d t} p_t = -\nabla \cdot (p_t \nabla s_t)$. We explore this task and prior methods in the following Section.

\paragraph{Harmful covariate shift detection.} Covariate shifts refers to changes in the test data distribution $p_{\text{test}}(x)$ as compared to the training distribution $p_{\text{train}}(x)$ while the relation between inputs and outputs remain fixed, i.e. $p_{\text{test}}(y|x) = p_{\text{train}}(y|x)$. Importantly, we do not require labels to determine this shift, thus it is practical to do so in a standard deployment setting. Prior work in this domain include deep kernel MMD \citep{liu2020learningdeepkernels}, H-divergences \citep{zhao2022comparing}, and Detectron \citep{ginsberg2023harmfulcov}. As Detectron requires minimal tuning and is most performant in low-data regimes ($N < 100$ samples), we emphasize the use of this approach. In particular, Detectron \citep{ginsberg2023harmfulcov} builds off of \textit{selective classification}---building classifiers that accept/reject test data depending on closeness to the training distribution---and \textit{PQ-learning} \citep{goldwasser2020beyondperturbations} that extends the conventional theory of PAC learning to arbitrary test distributions by employing selective classification. The main idea considers the generalization set $\gR$ of a classifier $f_\theta$ and samples $\gQ$ from an unknown distribution. The strategy is to fine-tune constrained disagreement classifiers (CDCs) to agree with $f_\theta$ on $\gR$ but disagree on $\gQ$. If $\gQ \subset \gR$, then it will be difficult to disagree on $\gQ$, but if the CDCs behave inconsistently on $\gQ$, that suggests a covariate shift. Notably, this method is sample-efficient, agnostic to classifier architecture, and may be used in tests of statistical significance.

See Appendix \ref{app:rel-works} for further related works.

\section{Modeling Weight Trajectories}
\label{sec:modeling}

In this section, we build towards our approach. Proofs are provided in App. \ref{app:proofs}. 

\paragraph{The continuity equation on neural network parameters.}
For the purpose of our analysis, let us restrict our view to neural networks that are optimized by gradient descent (GD) algorithms to minimize a loss $\gL(\theta_k) := d(\gM_\theta(\mX) - \mY)$, where $\gM_\theta$ is a neural network parameterized by trainable weights $\theta \in \R^p$, $\mX \in \R^{N \times D}$ are inputs, $\mY \in \R^{N \times c}$ are labels, and $d$ is some differentiable distance function, such as cross-entropy. To minimize via GD, parameter updates are done by $\theta_{k + 1} = \theta_k - \alpha \nabla \gL(\theta_k)$ given some learning rate $\alpha > 0$. Taking the learning rate to zero, we can view parameter evolution as a \textit{gradient flow}. For simplicity, we assume that the updates are deterministic (contrary to stochastic gradient descent which randomly selects training batches), and defer to App. \ref{app:sb-formulation} an approach incorporating stochasticity via Schr\"odinger bridges. For now, the sole source of randomness is the initialization $\theta_0 \sim p_0$. Within this setting, we can write down a continuity equation. In later sections, we show how this result underpins the choice of modeling framework.

\mybox[orange!20!white]{
\begin{theorem}[name=Informal; follows Ch. 8.3 of \citet{santambrogio2015optimal}]
\label{thm:ce}
    Let $\theta_0 \sim p_0$ be initialized network parameters and the loss $\gL$ is $C^1$ in $\theta$. If $(\theta_t)_{t \geq 0}$ is the gradient descent curve, we have $p_t = \mathrm{Law}(\theta_t)$ with 
    \begin{equation}
        \label{eq:ce}
        \partial_t p_t - \nabla\cdot(p_t\nabla \gL) = 0.
    \end{equation}
\end{theorem}
}

\subsection{Modeling Eqn. \ref{eq:ce}}
\label{sec:proxies}
The problem of learning the continuous dynamics of a system governed by a continuity equation has been studied in many forms in existing literature. In our setting, Theorem \ref{thm:ce} establishes a link between the practical dynamics of SGD and the continuity equation (Eqn. \ref{eq:ce}), which provides a more tractable theoretical framework. Building on this connection, we study methods that realize parameterized solutions to Eqn. \ref{eq:ce}: $\partial_tp_t + \nabla \cdot (p_tv_t^\Theta) = 0$, thereby providing a common lens of interpretation. \citet[Theorem 1]{lipmanFlowMatchingGenerative2023} showed that CFM may be viewed through this lens, hence motivating its use. However, its training objective assumes affine Gaussian paths, which oversimplifies the non-terminal distributions along a GD path. Therefore, we turn to a natural generalization: multi-marginal flow matching (MMFM), which has been shown \citep[Prop. 2]{rohbeck2025modeling} to correspond to training multiple CFMs. To add to our selection, Theorem \ref{thm:mmfm-action-gap} shows a correspondence between MMFM and JKOnet${}^*$ \citep{terpin2024learning} via the \textit{action gap}; this connection, its formulation towards modeling Eqn. \ref{eq:jko-ce}, and its efficient scalar potential parameterization motivates its consideration. Below, we expound on the action gap lens, provide some background, and present a generalization to non-affine regression targets (App. \ref{app:extended-proxy-learning}).

\subsubsection{The action gap} 

Frameworks which match gradients have a learnable function $\Psi_\theta(x,t)$ which is trained to match a regression target. Two representatives of this approach are the works on Action Matching \citep{neklyudov2023actionmatching, neklyudov2024wassersteinlagrangian} and the JKOnet family \citep{bunne2022proximal, terpin2024learning}. We found the theoretical framework of the \textit{action gap} from Action Matching to be most suitable as a reference; we recall it here.

\paragraph{Action matching.} The action matching setup presumes an initial distribution $q_0$, a velocity field $v:[0, 1] \times \Omega \to \R^p$, and a continuity equation which describes the dynamics: $\frac{d}{dt}q_t = -\nabla\cdot(q_tv_t)$. \citet[Theorem 2.1]{neklyudov2023actionmatching} showed that, under mild conditions on $q_t$, a unique function $s_t^*(x)$ termed the \textit{action} may be defined such that $v_t(x) = \nabla s_t^*(x)$ and the continuity equation $\frac{d}{dt}q_t = -\nabla\cdot(q_t\nabla s_t^*)$ is satisfied. One can readily see the connection with Eqn. \ref{eq:ce}: $s_t^* \equiv -\Ls$ up to a constant. Therefore, the \textit{action gap} is 
    $AG(s, s^*) = \frac{1}{2}\int_0^1 \E_{q_t(x)} ||\nabla s_t(x) - \nabla s_t^*(x)||^2 \, dt = \frac{1}{2}\int_0^1 \E_{q_t(x)} ||\nabla s_t(x) + \nabla \Ls (x)||^2 \, dt.$
This optimization is clearly intractable because of the required access to $\nabla \Ls$, therefore the authors computed a more tractable variational objective for optimization \citep[Theorem 2.2.]{neklyudov2023actionmatching}. To close this exposition, we recall a bound on the 2-Wasserstein distance which will be a recurring theme in this section.

\mybox[orange!20!white]{
\begin{proposition}[name=Prop. A.1 of \citet{neklyudov2023actionmatching}]
\label{prop:w2-bound-action-gap}
    Suppose the curl-free vector field $\nabla s_t$ is continuously differentiable in $(t, x)$, and uniformly Lipschitz in $x$ throughout $[0, 1] \times \R^p$ with Lipschitz constant $K$. Let $\hat q_t$ denote the density path induced by $\nabla s_t$. Then,
    \begin{equation}
        \label{eq:w2-bound-action-gap}
        W_2^2(\hat q_\tau, q_\tau) \leq e^{\frac{1+2K}{\tau}} \int_0^\tau \E_{q_t(x)}||\nabla s_t(x) + \nabla \Ls (x)||^2 \, dt
    \end{equation}
\end{proposition}
}

\subsubsection{Approximating drift in practice}
\label{sec:approx-ce-in-practice}

Although we made use of the action gap as our theoretical framework, we opt for the simpler objectives of JKOnet${}^*$ and MMFM in practice. In this section, we present some background and show their objectives may be cast as minimizing the action gap in the discretization limit. 

\paragraph{JKOnet.} The JKOnet family considers the problem of modeling the Fokker-Planck equation 
\begin{equation}
\label{eq:jko-ce}
\partial_tp_t(x) = \nabla \cdot (p_t(x) \nabla V(x)) + \beta\Delta p_t(x),
\end{equation}
given some potential $V$. The seminal work \citet{jko1998variationalformulation} (namesake of the JKO scheme) related such equations to a variational objective in Wasserstein space, namely 
$\mu_{t + 1} = \argmin_{\mu \in \gP(\R^p)} J(\mu) +  W_2^2(\mu, \mu_t) / 2\tau$ where $J(\mu) := \int_{\R^p} V(x) p_t(x) \, dx + \beta \int_{\R^p} p_t(x) \log (p_t(x)) \, dx$,
and step size $\tau > 0$.
Focusing on the most recent presentation, termed JKOnet${}^*$ \citep{terpin2024learning}, consider the Euclidean analog of the scheme and its first-order optimality condition: $\nabla J(x_{t + 1}) + (x_{t + 1} - x_t) /\tau = 0$, where $x_t \sim \mu_t$. If we let $\beta = 0$, then by \citet[Prop. 3.1]{terpin2024learning}, we have the minimization objective
\begin{equation}
    \label{eq:jkonet-pot-loss}
    \int_{\R^p \times \R^p} \Bigg\vert\Bigg\vert \nabla V(x_{t+1}) + \frac{1}{\tau}(x_{t + 1}- x_t)\Bigg\vert\Bigg\vert ^ 2 \, d\pi_t(x_t, x_{t + 1}) = 0 ,
\end{equation}
where $\pi_t$ is the optimal coupling between $\mu_t$ and $\mu_{t + 1}$. It can be readily seen that when $\beta = 0$, Eqn. \ref{eq:jko-ce} matches the form of Eqn. \ref{eq:ce}, and the requisite setup for action matching. The issue with applying the bound Eqn. \ref{eq:w2-bound-action-gap} stems from the time-discretization. To resolve this, we show in Theorem \ref{thm:mmfm-action-gap} (replacing $u_t$ with $(x_{t+1}-x_t)/\tau$) that the action gradient $\nabla s_t^*$ may be approached in the limit. Moreover, the objective Eqn. \ref{eq:jkonet-pot-loss} is an intuitive description of matching the best linear approximation of the gradient $\nabla \Ls$. Motivated by its simplicity, we opt towards JKOnet${}^*$ as the representative approach. 

\paragraph{Multi-marginal flow matching.}
The MMFM objective is similar to that of CFM in Eqn. \ref{eq:ce}. The difference lies in the definition of the regression target. Suppose we wish to generate marginal densities $p_{t_0}^*, p_{t_1}^*, \ldots, p_{t_K}^*$. Instead of sampling $x \sim p_1$, we sample $z = (x_0, \ldots, x_K)$ independently from each of the marginal densities. To align with the CFM objective, the reference path $p_t(x|z)$ must be a piecewise affine Gaussian path, and its mean a linear interpolation between the $K+1$ samples. Formally, define the interpolant
\begin{equation}
    \label{eq:mmfm-mean}
    \mu_t(z) = \sum_{ k = 0}^{K - 1} \left( x_k + \frac{ t- t_k}{t_{ k + 1} - t_k} (x_{k + 1} - x_k) \right) \cdot \mathbf{1}_{[t_k, t_{ k + 1})}(t)
\end{equation}
and the regression target ought to be
\begin{equation}
    \label{eq:mmfm-regression-target}
    u_t(x\mid z) = \sum_{ k = 0}^{K - 1} \frac{x_{k + 1} - x_k}{t_{ k + 1} - t_k} \cdot \mathbf{1}_{[t_k, t_{ k + 1})}(t).
\end{equation}
With the usual marginalization of $u_t$, i.e. $u_t(x)p_t(x) = \E_{q(z)}[u_t(x\mid z) p_t(x \mid z)]$, we can argue by checking the continuity equation \citet[Theorem 3.1]{tongImprovingGeneralizingFlowbased2024} that $p_t$ is generated by $u_t$. 

Additionally, it is natural to think that if the timepoints $(t_0, \ldots, t_K)$ were dense enough, then its limit is the true probability path $t \mapsto p_t^*$. \citet[Prop. 2]{rohbeck2025modeling} proved that the MMFM objective is equivalent to solving $K$ CFM objectives. By analogy, we consider a MMFM setup equivalent to $K$ OT-CFM objectives, and we show that the reference path $p_t$ approaches $p_t^*$ in the sense of the action gap (cf. Prop. \ref{prop:w2-bound-action-gap}).

\mybox[orange!20!white]{
\begin{theorem}
    \label{thm:mmfm-action-gap}
    Suppose the true marginals evolve according to $\frac{d}{dt}p_t^* = -\nabla \cdot(p_t^* \nabla s_t^*)$ and $t \mapsto p_t^*$ is an absolutely continuous curve. Define $q(z)$ such that marginalizing $q$ with respect to all variables except $x_k, x_{k + 1}$ yields the coupling $p_{t_k} \otimes (T_k^{k +1})_\# p_{t_k}$, where $T_k^{k + 1}$ is the transport map from $p_{t_k}$ to $p_{t_{k + 1}}$. Then,
    \begin{equation}
        \label{eq:mmfm-action-gap}
        \lim_{|t_k - t_{k + 1}| \to 0} \int_0^1 \E_{p_t(x)} ||u_t(x) - \nabla s_t^*(x)||_2^2 \, dt = 0.
    \end{equation}

    Replacing $u_t$ with $\frac{x_{t+1}-x_t}{\tau}$, this shows that $\nabla V$ (Eqn. \ref{eq:jkonet-pot-loss}) regresses to the reference action in the limit.
\end{theorem}
}

\section{Methods}
\label{sec:methods}

\subsection{Architectural modules}
\label{sec:methods-arch}
We describe the components of our approach below and leave more details to Appendix \ref{app:arch}, \ref{app:experimental-details}. Throughout, we use the \fsl- prefix to denote our methods, e.g. \fsl-CFM to denote conditional flow matching. Our framework is designed to be modular, with different components that can be instantiated in various ways. In this work, we prioritize simplicity in order to highlight the generative framework and the proposed reward fine-tuning mechanism. Figure \ref{fig:pipeline} provides an example of how these components connect.

\paragraph{Weight encoder.} Due to the often intractable size of weight space, it is sometimes necessary for modeling to happen in latent space (see App. \ref{app:practice-and-scaling} for scaling remarks). We justify this design by appealing to work on the Lottery Ticket Hypothesis \citep{frankle2018the, zhang2021validatinglottery, liu2024surveylotterytickethypothesis} as well as the body of work on pruning \citep{cheng2024asurveyonpruning}, which suggests that, like natural data, neural networks live on a low-dimensional manifold within its ambient space. There are a variety of encoders to choose from, such as the variational autoencoder (VAE) \citep{kingma2022autoencodingvariationalbayes}, and specialized encoders for neural network parameters \citep{kofinasGraphNeuralNetworks2024, putterman2024learningloras, schuerholt2024sane}. As we are mostly focused on the generative aspect, we use the simple VAE following \citet{soroDiffusionbasedNeuralNetwork2024}.

\paragraph{Generative meta-model.} The backbone of our meta-learning framework is a CFM following \citet{tongImprovingGeneralizingFlowbased2024}, or MMFM that matches piecewise-linear interpolants; see App. \ref{app:metric-learning} for a detailed discussion of interpolants. We also experiment with using the JKOnet${}^*$ method in a few tests; see details in App. \ref{app:arch}. Preliminary experiments were also done with learned interpolants such as Metric Flow Matching, but we found them to be unstable, likely due to the sparsity of data and the large ambient space favoring simpler interpolants. Exploiting the flexibility of FM to use a non-Gaussian prior, we use the Kaiming uniform or normal initializations \citep{he2015delvingdeep} as the source $p_0$; see App. \ref{app:futher-exps} for an ablation.

\paragraph{Reward fine-tuning.} FM models lend themselves to the recently proposed reward fine-tuning method, based on the adjoint ODE \citep{domingo-enrichAdjointMatchingFinetuning2024}, which casts stochastic optimal control as a regression problem. This allows us to tune pre-trained flow meta-models for downstream applications, exemplified in this work by detecting harmful covariate shifts, and improved generative performance (see App. \ref{app:reft-results}). Specifically, this method modifies the base generative distribution $\hat p_1^{\text{base}}$ to generate the reward-tilted distribution $p^*_1(x) \propto p_1^{\text{base}}(x) \exp(r(x))$ via the \textit{Adjoint Matching} (AM) algorithm. Naturally, in our setting, we suppose $p_1^{\text{base}}$ is obtained from meta-training and governs classifiers that predict on $\train_1$, but we wish to modify the meta-model to generate base models that predict on $\train_2$. Therefore, we set the reward $r(X_1) := -\Ls_2(X_1)$ where $\Ls_2$ is a loss on $\train_2$ such as cross-entropy and proceed with AM. See App. \ref{app:reft} for further details.

\subsection{Detecting harmful covariate shifts}
\label{sec:methods-cdc}

\paragraph{Training CDCs.} Continuing the exposition in Section \ref{sec:prelims}, we specify the training regiment of CDCs. Let $g(\cdot)$ represent one CDC and $f(\cdot)$ the base classifier; further, let $\tP = \{(x_i, y_i)\}_{i = 1}^n$ be samples from the generalization set $\gR$ and $\tQ = \{\tilde x_i\}_{i =1}^m$ from the unknown distribution. Our objective is for $g$ to maintain performance on $\tP$, but to disagree with $f$ on $\tQ$. Naturally, we use the cross-entropy loss $\ell_{ce}$ on $\tP$, but on $\tQ$, \citet{ginsberg2023harmfulcov} introduces the \textit{disagreement-cross-entropy}: \(
\ell_{dce}(\hat y, f(x_i)) = \frac{1}{1-N} \sum_{c = 1}^N \mathbf{1}_{f(x_i) \neq c} \log (\hat y_c),
\)
where $N$ denotes the total number of classes. We combine these objectives by minimizing the CDC loss:
\begin{multline}
    \label{eq:cdc-loss}
    \ell_{cdc}(\tP, \tQ) := \frac{1}{|\tP \cup \tQ|} \biggl( \sum_{(x_i, y_i) \in \tP} \ell_{ce}(g(x_i), y_i) + \\
    \hspace{1em}\lambda \sum_{\tilde x_i \in \tQ} \ell_{dce}(g(\tilde x_i), f(\tilde x_i))
    \biggr)
\end{multline}
To test for shift, Detectron compares $g$ trained with $\tQ$ sampled from the unknown distribution ($g_{\tQ}$) against $\tQ = \tP^*$ ($g_{\tP}$), i.e. samples from the generalization set. In particular, the disagreement rate or the class entropy for each case is obtained and hypothesis tested. In both cases, the disagreement and entropy are higher if $\tQ$ represents a significant shift.

\paragraph{Motivation.} The problem of detecting covariate shift is not just about the data---a lot of modern neural networks are robust to such changes. The essence of the problem is whether or not \textit{the classifier weights} required to predict on $\tQ$ differs from the current classifier, motivating a method that is sensitive to changes in the weights required to predict on a new set. Building on the finding that the support of CNN3 classifiers is narrow (see Sec. \ref{sec:reft}) and the fact that the reward-tilted distribution (obtained from reward fine-tuning) has the same support, if the ideal classifier required to predict on a new dataset lies far outside of the original support, then we would expect a noticeable performance difference after reward fine-tuning than if it were close to the original support (see corruption experiments in App. \ref{app:reft-results}).

\paragraph{Meta-detectron.} Our approach, termed \textit{meta-detectron}, builds on reward fine-tuning by adjoint matching. We start by meta-training a \fsl-CFM meta-model to learn classifier distributions on each of the datasets (this is identical to the unconditional generation setup). Next, we reward fine-tune, maintaining the procedure of sampling from the meta-model at each iteration to compute the reward, but now the reward function is $r(X_1) = -\ell_{cdc}(\tP, \tQ; X_1)$, where $X_1$ serves the role of $g$, and the original meta-model generates the base classifier $f$ in Eqn. \ref{eq:cdc-loss}. As the method requires training $g_{\tP}$ and $g_{\tQ}$, we likewise fine-tune two different meta-models depending on the disagreement set, and compare the disagreement rate and entropy of the generated $g_{\tP}$ and $g_{\tQ}$. Returning to our motivation, it is more likely for the support to lie closer to classifiers that disagree on an out-of-distribution set $\tQ$ than those disagreeing on $\tP^*$. That being said, we expect the disagreement rate to be more conservative overall due to the tightness of the support.

\section{Experiments}
\label{sec:experiments}
\begin{table}
\caption{Best validation accuracy of unconditional \fsl{} CNN3 generation for various datasets. \textit{Original} denotes base models trained conventionally by SGD and \textit{p-diff} those generated with p-diff \citep{wangNeuralNetworkDiffusion2024}. MMFM($k$) and JKO($k$) indicates the number of marginal distributions in addition to $p_0$ and $p_1$.}
\label{tab:uncond}
\tablestyle{2pt}{1.0}
\begin{center}
\begin{adjustbox}{max width=0.4\textwidth}
\begin{tabular}{l @{\hspace{20pt}}cccc}
\toprule
&  {CIFAR100} & {CIFAR10} & {MNIST} & {STL10} \\
\midrule
Original & 25.62 & 63.38 & 98.93 & 53.88 \\
p-diff & 25.99 & 63.37 & 98.93 & 53.86 \\
\fsl-CFM w/ VAE & 26.01 & 64.32 & 98.91 & 53.50 \\
\fsl-CFM        & 25.31 &   62.52 &   98.52 &   53.49 \\
\fsl-MMFM(2)     &  21.16 &   63.35 &   98.53 &   53.20  \\
\fsl-MMFM(3)     &  24.53 &   63.34 &   98.62 &   53.53 \\
\fsl-MMFM(4)     &  22.89 &   61.33 &   98.39 &   53.08 \\
\fsl-JKO(3)     & 24.84 &   62.60 &   98.59 &   53.44 \\
\fsl-JKO(4)     &  24.97 &   63.35 &   98.57 &   53.21 \\
\bottomrule
\end{tabular}
\end{adjustbox}
\vspace{-1em}
\end{center}
\end{table}

\begin{table}
\vspace{-1em}
\caption{Mean validation accuracy of top-5 \fsl{} model retrievals. A single meta-model is used for all base datasets, with a conditioning signal obtained from image samples used to distinguish between each set.}
\label{tab:model-retrieval}
\tablestyle{2.0pt}{1.1}
\begin{center}
\begin{adjustbox}{max width=0.4\textwidth}
\begin{tabular}{lcccc}
\toprule
 &  \multicolumn{1}{c}{{CIFAR10}} & \multicolumn{1}{c}{{STL10}} & \multicolumn{1}{c}{{MNIST}} & \multicolumn{1}{c}{{FMNIST}}\\
\midrule
Original & 63.38 & 53.88 & 98.93 & 89.77 \\
\fsl-CFM & 62.89 & 53.47 & 98.69 & 90.24 \\
\fsl-CFM w/ VAE & 62.79 & 53.41 & 98.54 & 90.59 \\
\fsl-MMFM & 53.45 & 49.62 & 92.86 & 76.55 \\
\fsl-MMFM w/ VAE & 63.87 & 52.86 & 98.36 & 89.73 \\
\fsl-JKO\footnotemark[1] & 62.04 & 51.16 & 97.44 & 87.75 \\
\bottomrule
\end{tabular}
\vspace{-2em}
\end{adjustbox}
\end{center}
\end{table}
\footnotetext[1]{JKOnet expects the trajectory to evolve via gradient flow. This is not guaranteed in the latent space, hence we only show the un-encoded variant.}

First, we confirm various properties that are to be expected of weight generation models. Next, we explore reward fine-tuning courtesy of adjoint matching, finishing with an application to detect harmful covariate shifts. Further experiments and details may be found in Apps. \ref{app:futher-exps} and \ref{app:experimental-details} respectively.

\subsection{Generative modeling desiderata}

\paragraph{Unconditional generation.} We first evaluate the modeling capacity of the flow meta-model. The target distribution $p_1$ is defined by training base models on CIFAR10, CIFAR100, STL10, and MNIST, and saving weight checkpoints across 100 epochs. For large models, we may choose to generate only a subset of the weights. For large models, we generate only a subset of weights: batch norm parameters for ResNet-18 \citep{He2015DeepRL}, ViT-base \citep{dosovitskiy2021animage}, and ConvNext-tiny \citep{liu2022convnet2020s}, and the full medium-CNN (CNN3) \citep{schurholtModelZoosDataset2022}. A separate meta-model is trained per dataset and validated by reconstructing its base model for test classification. Focusing on CNN3, Table~\ref{tab:uncond} shows performance matching conventionally trained models and p-diff \citep{wangNeuralNetworkDiffusion2024} (extra results in Table~\ref{tab:uncond-app}). Figure~\ref{fig:traj-losses} further shows that methods using trajectory information (MMFM and JKO) yield faster validation loss decrease over inference steps, consistent with gradient descent converging rapidly toward the final loss. For later experiments we therefore restrict to MMFM(3) and JKO(4) (see App.~\ref{app:experimental-details} for computational details).

\begin{figure}
    \centering
    \includegraphics[width=0.49\linewidth]{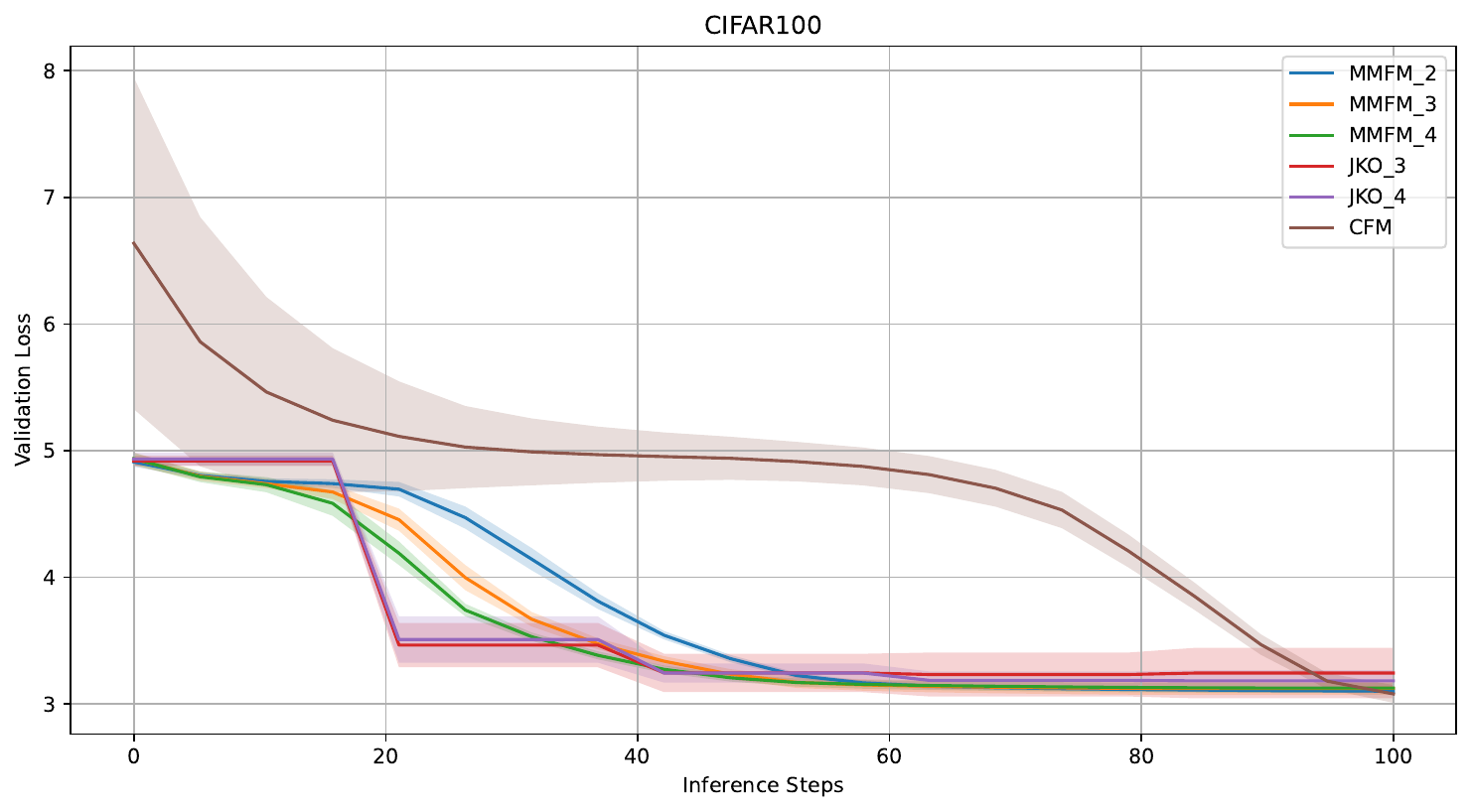}
    \includegraphics[width=0.49\linewidth]{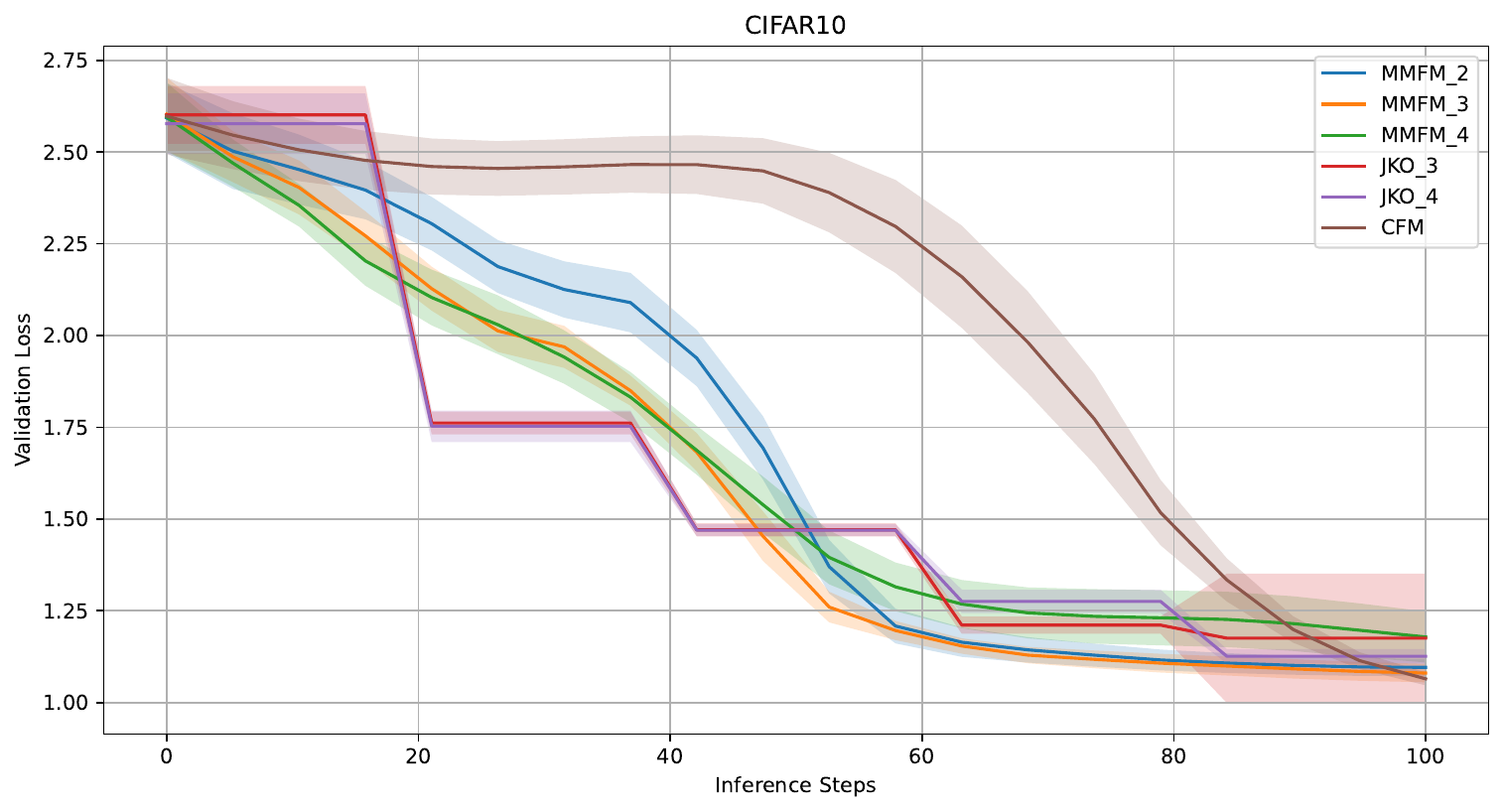}
    \includegraphics[width=0.49\linewidth]{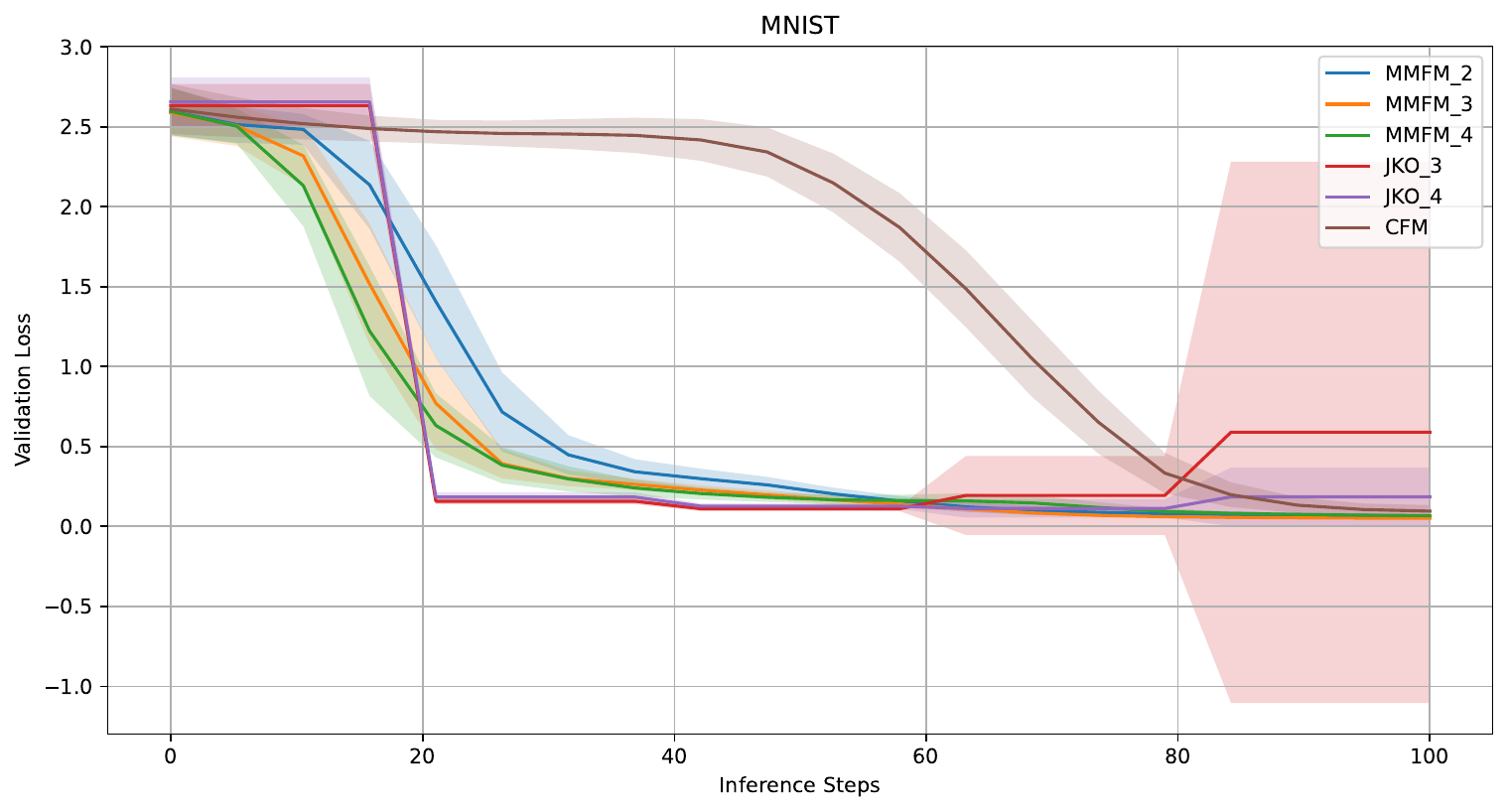}
    \includegraphics[width=0.49\linewidth]{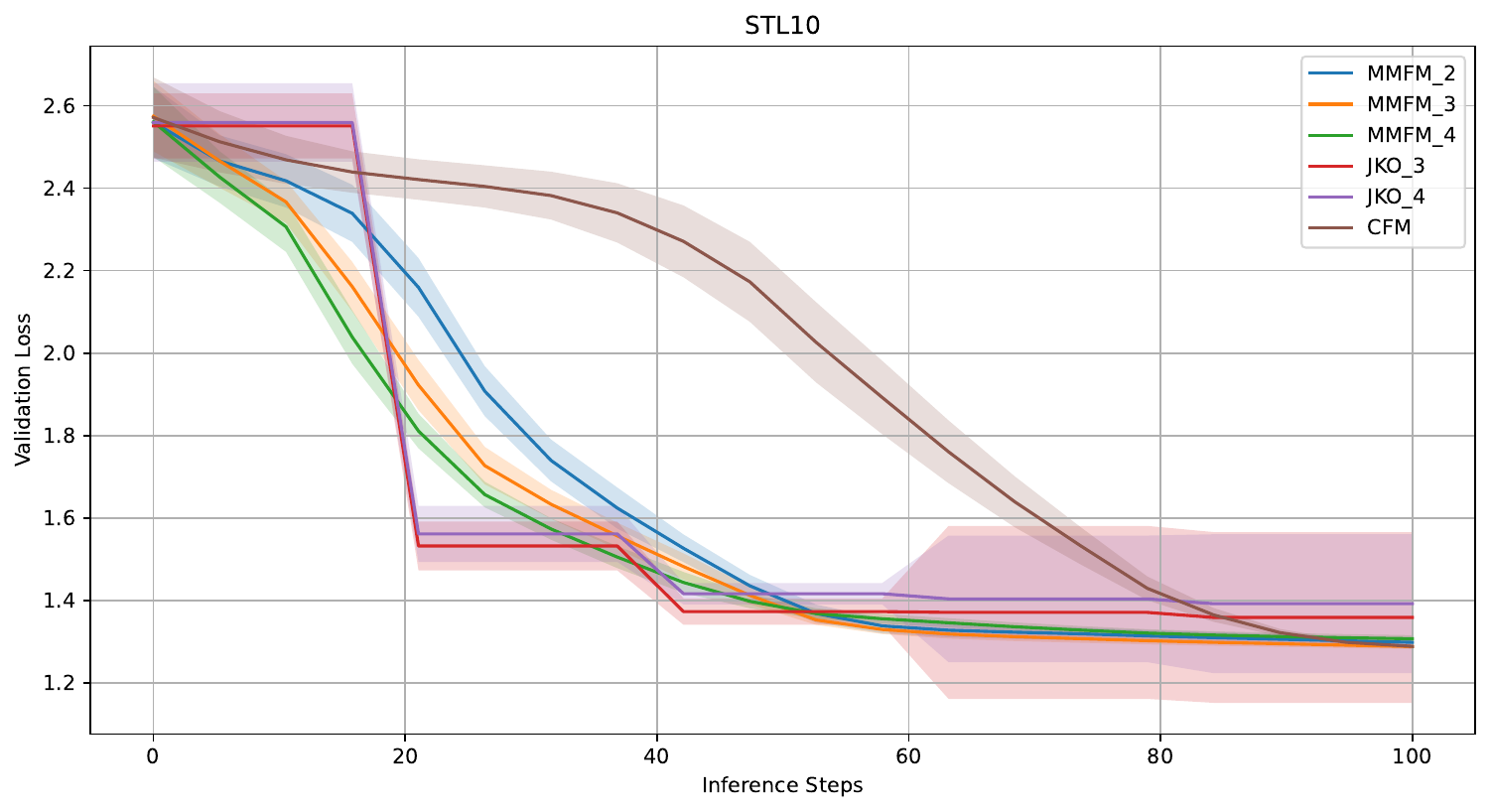}
    \vspace{-1.4em}
    \caption{Base model validation loss over the course of inference for various \fsl{} methods. The plots were computed on 20 out of 100 intermediate timepoints for MMFM and CFM, but restricted by design to $k$ timepoints for JKO($k$). MMFM\_k refers to MMFM with $k$ intermediate marginal distributions (distributions in addition to $p_0$ and $p_1$) and likewise for JKO.}
    \label{fig:traj-losses}
\end{figure}

\begin{figure}
    \vspace{-1em}
    \centering
    \includegraphics[width=0.49\linewidth]{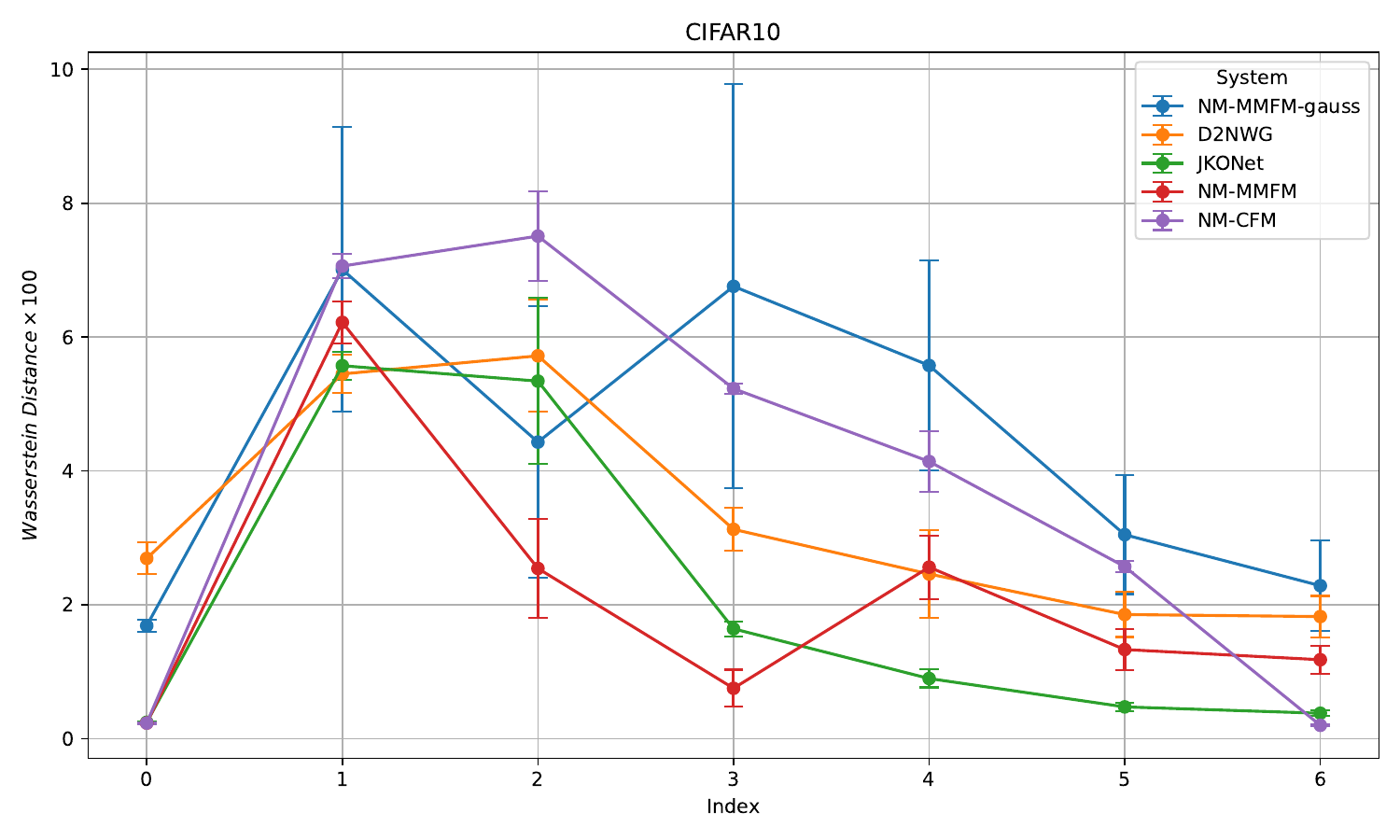}
    \includegraphics[width=0.49\linewidth]{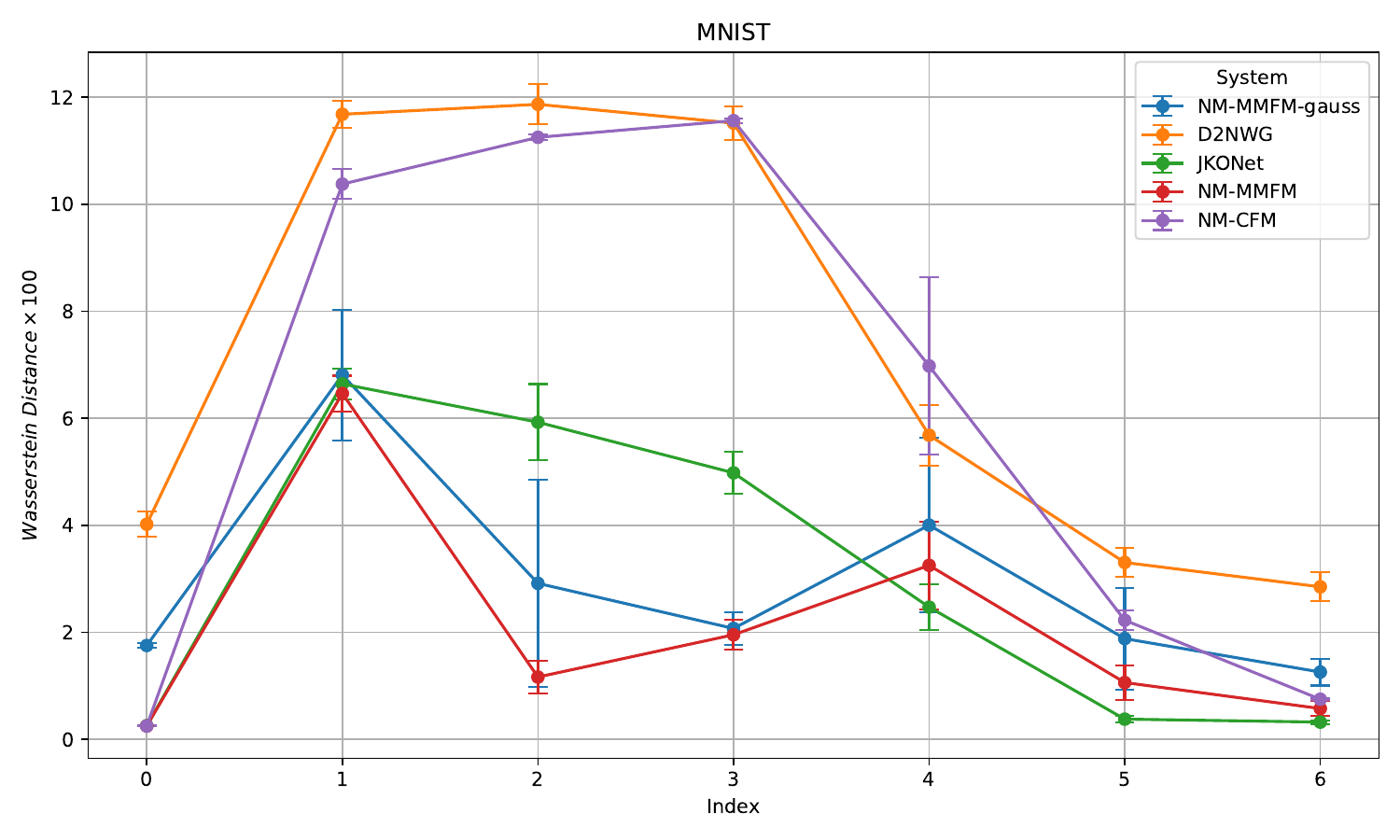}
    \vspace{-1.4em}
    \caption{Mean $W_1$-distance ($\times 100$) between reference and generated intermediate marginals over 5 seeds of unconditional generation. The horizontal axis corresponds to increasing indices of the intermediate marginals, i.e. $k$ in $p_{t_k}$ where $t_0 = 0, \, t_6 = 1$. The plots also show the effect of using a Gaussian prior with MMFM (denoted \fsl-MMFM-gauss), excluding \fsl-CFM-gauss due to its large $W_1$ deviation.}
    \label{fig:w1-zoomed-tests}
    \vspace{-1em}
\end{figure}

\paragraph{Trajectory modeling.}
In this experiment, we evaluate the ability of different approaches to model the weight trajectory. As decided above, we used \fsl-MMFM(3) and \fsl-JKO(4) as representatives for this method. Moreover, in the interest of fairness, we divide the trajectory into 5 buckets, and so the MMFM and JKO methods would need to interpolate between training distributions. Our baseline is D2NWG \citep{soroDiffusionbasedNeuralNetwork2024} where the VAE is trained on full trajectory weights at each batch iteration. See Figure \ref{fig:w1-zoomed-tests} for results and App. \ref{app:futher-exps} for more results, including an investigation of parameter symmetries. Interestingly, we find that the Wasserstein-1 ($W_1$) distance of the generated trajectory to be consistently lower in D2NWG vs. \fsl-CFM, but both methods lag behind MMFM and JKO which explicitly models the weight trajectory. The over-performance of MMFM and JKO is to be expected, and we suspect the D2NWG performance is due to latent space training on the full trajectory data; that is, even if the interpolated weights do not follow the expected trajectory, it still lands on the data manifold. See Figure \ref{fig:vae-latent-space} for a mock illustration.

\paragraph{Model retrieval.} We perform model retrieval to test whether the meta-model can distinguish weights of the base model given conditioning samples from the dataset the base model was trained on. The base model is the same CNN3 model and we obtain weight checkpoints trained on MNIST, Fashion-MNIST (FMNIST), CIFAR10, and STL10 across 50 epochs of conventional training. Unlike in the previous test, we will train just \textit{one} meta-model conditioned on context samples from their respective training sets passed through a CLIP \citep{Radford2021LearningTV} encoder (see Figure \ref{fig:f2sl}). At validation, we pass in a random support sample and generate the full CNN3. Table~\ref{tab:model-retrieval} shows that mean top-5 validation accuracies largely match base models, confirming capacity for conditional generation. We note that \fsl-MMFM underperforms \fsl-CFM, reflecting the higher complexity of its target vector field. Larger meta-models improve accuracy (App.~\ref{app:experimental-details}), but we retain these results to highlight the tradeoff.

\paragraph{Downstream initialization.}
Next, we repeat the model retrieval test, but instead we obtain weights across 10 epochs of conventional training. We opted for the encoded runs if possible for efficiency. The generated weights are used as initialization before fine-tuning another 20 epochs. As shown in Table \ref{tab:ft-retrieval}, our initialization enjoys faster convergence, even for corrupted datasets, highlighting generalization capability.

\subsection{Reward fine-tuning}
\label{sec:reft}

In this section, we investigate the use of reward fine-tuning for shifting a distribution of classifier weights. Detail of our modifications, and its extension to multiple marginals, can be seen in App. \ref{app:reft}.

\paragraph{Support of classifier weights.} Notably, the method of reward fine-tuning \textit{cannot} be applied for arbitrary meta-model fine-tuning since $\text{supp } p_1^{\text{ft}} = \text{supp } p_1^{\text{base}}$. For some loss $\gL$, a soft (due to discretization and random sampling) loss lower bound for weights $p_1^{\text{ft}}$ is $\arginf_\alpha \{\alpha > 0: \text{supp } p_1^{\text{base}} \cap \{x : \gL(x) \leq \alpha\} \neq \emptyset\}$. 
We stress that this property of the support is a function of both the downstream data \textit{and} the model architecture. Indeed, due to the small size of the CNN3, the parameters that predict on different datasets e.g. CIFAR10 and STL10, will differ considerably, but this may not be the case for larger neural networks which possess a larger generalization set. We hypothesize: \textit{the support set of CNN3 weights trained for different datasets are narrow and mostly disjoint, thus, small changes in the training data will noticeably affect the support w.r.t. validation accuracy.} Note how this ties back to the motivation of Meta-Detectron (Sec. \ref{sec:methods-cdc}).

\paragraph{Results.}
We found experimental evidence to support this hypothesis, but also to suggest that reward fine-tuning goes a long way towards improving validation accuracy on out-of-distribution data obtained by increasing image corruption. We defer results and discussion to App. \ref{app:reft-results}, specifically Tables \ref{tab:corrupt-ft} and \ref{tab:corrupt-ft-extra}. Given this finding, we use it to approach the problem of harmful covariate shifts.

\begin{table}[h]
\caption{TPR@5 and AUROC for detection of harmful covariate shift on CIFAR10.1 and Camelyon17. We test on both the disagreement rate (DAR) and the entropy, setting $\lambda = \kappa/(|\tQ| + 1)$. See App. \ref{app:meta-detectron-training} for details on choosing $\kappa$ and extra results; the runs here vary $\kappa$ between $|\tQ|$. The best results are \textbf{bolded}.}

\label{tab:meta-detectron}
\tablestyle{3pt}{1.0}
\begin{center}
\begin{adjustbox}{max width=\columnwidth}
\begin{tabular}{l @{\hspace{10pt}}ccc ccc}
\toprule
\textbf{TPR@5} & \multicolumn{3}{c}{CIFAR10} & \multicolumn{3}{c}{Camelyon}\\
\cmidrule(lr){2-4} \cmidrule(lr){5-7}
 $|\tQ|$ & 10 & 20 & 50 & 10 & 20 & 50 \\
\midrule
Det. (DAR) & 0 & 0 & \pmval{.10}{.10} & \pmval{.10}{.10} & \pmval{.20}{.13} & \pmval{.50}{.17} \\

Meta-det. (DAR) & \pmval{.53}{.13} & \pmval{.47}{.13} & \pmval{.53}{.13} & \pmval{.73}{.12} & \textbf{\pmval{.40}{.13}} & \textbf{\pmval{.68}{.10}} \\

Det. (Ent) & \pmval{.60}{.16} & \pmval{.10}{.10} & \pmval{.10}{.10} & 0 & 0 & 0 \\
Meta-det. (Ent) & \pmval{.47}{.13} & \textbf{\pmval{.93}{.07}} & \textbf{1.00} & \textbf{1.00}  & 0 & \pmval{.24}{.09} \\

\toprule

\textbf{AUROC} & \multicolumn{3}{c}{CIFAR10} & \multicolumn{3}{c}{Camelyon}\\
\cmidrule(lr){2-4} \cmidrule(lr){5-7}
$|\tQ|$ & 10 & 20 & 50 & 10 & 20 & 50 \\
\midrule
Det. (DAR) & 0.480 & 0.495 & 0.665 & 0.665 & 0.750 & 0.875 \\
Meta-det. (DAR) & \textbf{0.876} & 0.838 & 0.900 & 0.867 & 0.760 &  \textbf{0.930} \\

Det. (Ent) & 0.775 & 0.740 & 0.785 & 0.490 & 0.445 & 0.660\\
Meta-det. (Ent) & 0.809 &\textbf{0.987} & \textbf{1.000} & \textbf{1.000}  & \textbf{0.836} & 0.755 \\

\bottomrule
\end{tabular}
\end{adjustbox}
\vspace{-1em}
\end{center}
\end{table}

\subsection{Detecting harmful covariate shifts}

We evaluate Meta-detectron (Sec. \ref{sec:methods-cdc}) on CNN3 with experiments following \citet{ginsberg2023harmfulcov} on CIFAR10.1 \citep{recht2019imagenet}, where shift comes from the dataset pipeline, and Camelyon17 \citep{veeling2018rotation}, which consists of histopathological slides from multiple hospitals. Table \ref{tab:meta-detectron} shows the \textit{True Positive Rate at 5\% Significance Level (TPR@5)} and \textit{AUROC} aggregated over 10 randomly chosen seeds for sampling $\tP^*$ and $\tQ$ of varying sample sizes. In addition, we ablated over the weight $\lambda$; see App. \ref{app:meta-detectron-training} for details and further results. Compared to the original tests \citep[Table 1]{ginsberg2023harmfulcov} on ResNet-18, we observe that covariate shift is highly architecture dependent. This is expected as CNN3 underfits CIFAR10 ($\sim 63\%$ validation accuracy). Our approach accounts for this as the base classifiers are generated directly by the fine-tuned meta-models. We also observe--though not shown--lower disagreement rates overall, which pays off in the TPR@5 as the $\tP^*$ disagreement rates are close to zero in all cases, and confirms the conservative nature of our method. Importantly, we also observe in Table \ref{tab:meta-detectron-val-acc} that the validation accuracy on $\tP$ is mostly unchanged. Regarding meta-training behavior, Figure \ref{fig:auroc-and-loss} shows that the AUROC increases sharply early in the reward fine-tuning phase, requiring only about 50 batch iterations to reach its peak. This coincides with a marked decrease in $\ell_{cdc}$. However, we also note some instability in the AUROC throughout training, particularly in the Camelyon experiments, where fluctuations are more pronounced.

\begin{figure}
    \centering
    \includegraphics[width=0.49\linewidth]{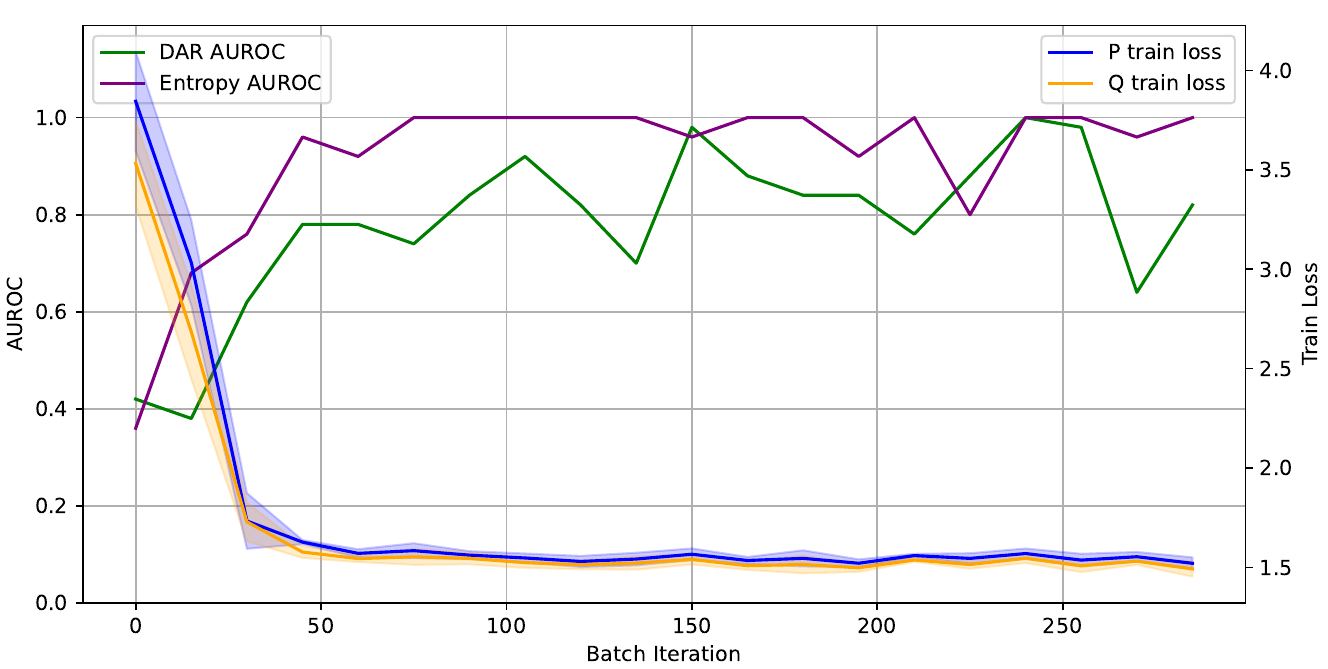}
    \includegraphics[width=0.49\linewidth]{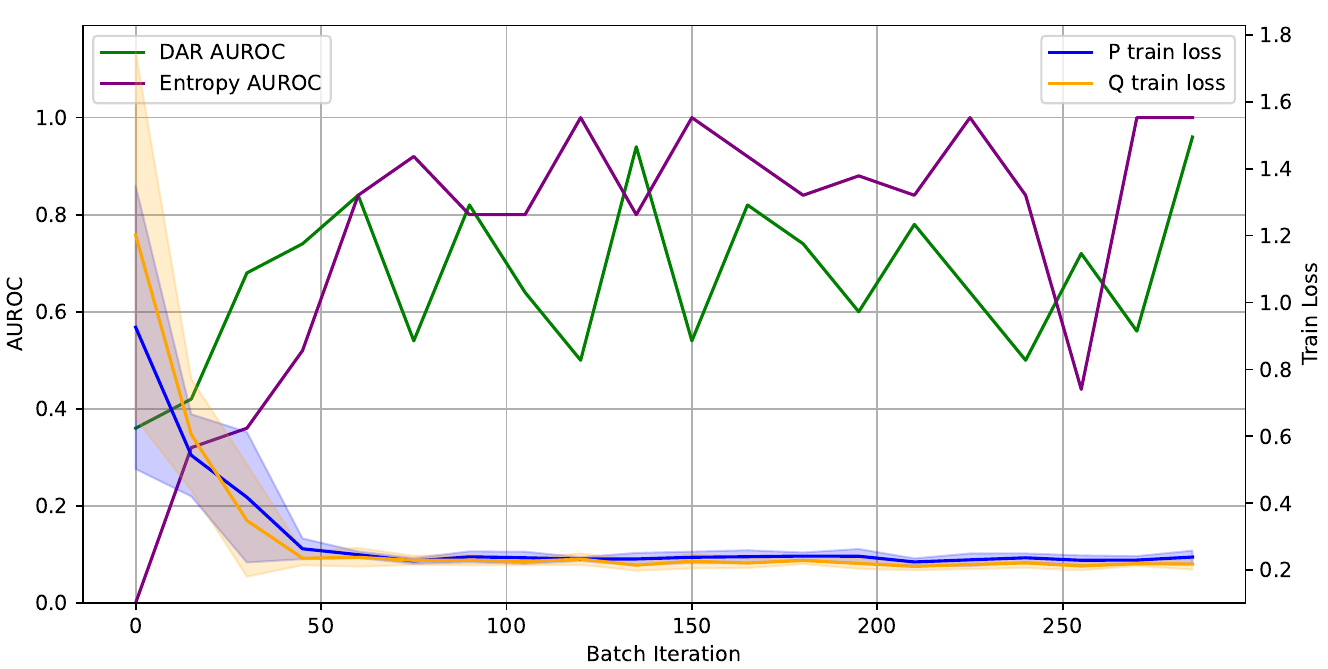}
    \vspace{-1em}
    \caption{Plots illustrating how AUROC and $\ell_{cdc}$ evolves over meta-detectron training iterations for CIFAR10 and Camelyon17 when $|\tQ|=20$. See App. \ref{app:meta-detectron-training} for more figures.}
    \label{fig:auroc-and-loss}
\end{figure}

\section{Conclusion}
\label{sec:discussions}
In this work, we have provided a preliminary investigation of the latest dynamical generative models for weight generation with applications to covariate shift detection. Due to the large size of modern neural network architectures, limited resources constrain our study to architectures with $<10^6$ parameters. To address this concern, as well as training dataset diversity and the lack of experiments incorporating stochastic weight evolution, future research directions include exploration of equivariant architectures to reduce dimensionality of weight space, and incorporating Schr\"odinger bridge matching. Moreover, the methods here open up a plethora of other applications. For instance, we may experiment with  \textbf{1)} more traditional meta-learning tests such as zero- and few-shot learning; \textbf{2)} model merging by superposition of the inference ODE/SDE \citep{skreta2025superpositiondiff}; or \textbf{3)} network constrained problems such as generating binary neural networks.

\subsubsection*{Acknowledgments}
The authors would like to thank Lazar Atanackovic, Kirill Neklyudov, and Kacper Kapusniak for valuable
discussions on methodology and implementation. We are also grateful to Soro Bedionita for assistance
in resolving issues related to the D2NWG implementation. Special thanks go to Edwin Chacko, Vedant
Swamy, and Ashwin Santhosh—developers in the UTMIST project group where this work was originally
conceived—for their early collaboration.

{
\bibliography{./main}

\begin{thebibliography}{88}
\providecommand{\natexlab}[1]{#1}
\providecommand{\url}[1]{\texttt{#1}}
\expandafter\ifx\csname urlstyle\endcsname\relax
  \providecommand{\doi}[1]{doi: #1}\else
  \providecommand{\doi}{doi: \begingroup \urlstyle{rm}\Url}\fi

\bibitem[Albergo and Vanden-Eijnden(2023)]{albergo2023building}
Michael~Samuel Albergo and Eric Vanden-Eijnden.
\newblock Building normalizing flows with stochastic interpolants.
\newblock In \emph{The Eleventh International Conference on Learning Representations}, 2023.
\newblock URL \url{https://openreview.net/forum?id=li7qeBbCR1t}.

\bibitem[Lipman et~al.(2023)Lipman, Chen, {Ben-Hamu}, Nickel, and Le]{lipmanFlowMatchingGenerative2023}
Yaron Lipman, Ricky T.~Q. Chen, Heli {Ben-Hamu}, Maximilian Nickel, and Matt Le.
\newblock Flow {{Matching}} for {{Generative Modeling}}, February 2023.

\bibitem[Liu et~al.(2023)Liu, Gong, and qiang liu]{liu2023flow}
Xingchao Liu, Chengyue Gong, and qiang liu.
\newblock Flow straight and fast: Learning to generate and transfer data with rectified flow.
\newblock In \emph{The Eleventh International Conference on Learning Representations}, 2023.
\newblock URL \url{https://openreview.net/forum?id=XVjTT1nw5z}.

\bibitem[Tong et~al.(2024)Tong, Fatras, Malkin, Huguet, Zhang, Rector-Brooks, Wolf, and Bengio]{tongImprovingGeneralizingFlowbased2024}
Alexander Tong, Kilian Fatras, Nikolay Malkin, Guillaume Huguet, Yanlei Zhang, Jarrid Rector-Brooks, Guy Wolf, and Yoshua Bengio.
\newblock Improving and generalizing flow-based generative models with minibatch optimal transport.
\newblock \emph{Transactions on Machine Learning Research}, 2024.
\newblock ISSN 2835-8856.
\newblock URL \url{https://openreview.net/forum?id=CD9Snc73AW}.

\bibitem[Esser et~al.(2024)Esser, Kulal, Blattmann, Entezari, Müller, Saini, Levi, Lorenz, Sauer, Boesel, Podell, Dockhorn, English, Lacey, Goodwin, Marek, and Rombach]{esser2024scalingrectifiedflowtransformers}
Patrick Esser, Sumith Kulal, Andreas Blattmann, Rahim Entezari, Jonas Müller, Harry Saini, Yam Levi, Dominik Lorenz, Axel Sauer, Frederic Boesel, Dustin Podell, Tim Dockhorn, Zion English, Kyle Lacey, Alex Goodwin, Yannik Marek, and Robin Rombach.
\newblock Scaling rectified flow transformers for high-resolution image synthesis, 2024.
\newblock URL \url{https://arxiv.org/abs/2403.03206}.

\bibitem[Liu et~al.(2024{\natexlab{a}})Liu, Zhang, Ma, Peng, and qiang liu]{liu2024instaflow}
Xingchao Liu, Xiwen Zhang, Jianzhu Ma, Jian Peng, and qiang liu.
\newblock Instaflow: One step is enough for high-quality diffusion-based text-to-image generation.
\newblock In \emph{The Twelfth International Conference on Learning Representations}, 2024{\natexlab{a}}.
\newblock URL \url{https://openreview.net/forum?id=1k4yZbbDqX}.

\bibitem[Gat et~al.(2024)Gat, Remez, Shaul, Kreuk, Chen, Synnaeve, Adi, and Lipman]{gat2024discrete}
Itai Gat, Tal Remez, Neta Shaul, Felix Kreuk, Ricky T.~Q. Chen, Gabriel Synnaeve, Yossi Adi, and Yaron Lipman.
\newblock Discrete flow matching.
\newblock In \emph{The Thirty-eighth Annual Conference on Neural Information Processing Systems}, 2024.
\newblock URL \url{https://openreview.net/forum?id=GTDKo3Sv9p}.

\bibitem[Shaul et~al.(2024)Shaul, Gat, Havasi, Severo, Sriram, Holderrieth, Karrer, Lipman, and Chen]{shaul2024flowmatchinggeneraldiscrete}
Neta Shaul, Itai Gat, Marton Havasi, Daniel Severo, Anuroop Sriram, Peter Holderrieth, Brian Karrer, Yaron Lipman, and Ricky T.~Q. Chen.
\newblock Flow matching with general discrete paths: A kinetic-optimal perspective, 2024.
\newblock URL \url{https://arxiv.org/abs/2412.03487}.

\bibitem[Campbell et~al.(2024)Campbell, Yim, Barzilay, Rainforth, and Jaakkola]{campbell2024generative}
Andrew Campbell, Jason Yim, Regina Barzilay, Tom Rainforth, and Tommi Jaakkola.
\newblock Generative flows on discrete state-spaces: Enabling multimodal flows with applications to protein co-design.
\newblock \emph{arXiv preprint arXiv:2402.04997}, 2024.

\bibitem[Zhang et~al.(2024)Zhang, Luo, Yu, Li, Lin, Ye, and Zhang]{zhang2024metadiff}
Baoquan Zhang, Chuyao Luo, Demin Yu, Xutao Li, Huiwei Lin, Yunming Ye, and Bowen Zhang.
\newblock Metadiff: Meta-learning with conditional diffusion for few-shot learning.
\newblock \emph{Proceedings of the AAAI Conference on Artificial Intelligence}, 38\penalty0 (15):\penalty0 16687--16695, Mar. 2024.

\bibitem[Soro et~al.(2025)Soro, Andreis, Lee, Jeong, Chong, Hutter, and Hwang]{soroDiffusionbasedNeuralNetwork2024}
Bedionita Soro, Bruno Andreis, Hayeon Lee, Wonyong Jeong, Song Chong, Frank Hutter, and Sung~Ju Hwang.
\newblock Diffusion-based neural network weights generation.
\newblock In \emph{The Thirteenth International Conference on Learning Representations}, 2025.
\newblock URL \url{https://openreview.net/forum?id=j8WHjM9aMm}.

\bibitem[Essakine et~al.(2025)Essakine, Cheng, Cheng, Zhang, Deng, Zhu, Sch{\"o}nlieb, and Aviles-Rivero]{essakine2025where}
Amer Essakine, Yanqi Cheng, Chun-Wun Cheng, Lipei Zhang, Zhongying Deng, Lei Zhu, Carola-Bibiane Sch{\"o}nlieb, and Angelica~I Aviles-Rivero.
\newblock Where do we stand with implicit neural representations? a technical and performance survey.
\newblock \emph{Transactions on Machine Learning Research}, 2025.
\newblock ISSN 2835-8856.
\newblock URL \url{https://openreview.net/forum?id=QTsJXSvAI2}.

\bibitem[Lavenant et~al.(2024)Lavenant, Zhang, Kim, and Schiebinger]{lavenant2024towardstrajectory}
Hugo Lavenant, Stephen Zhang, Young-Heon Kim, and Geoffrey Schiebinger.
\newblock {Toward a mathematical theory of trajectory inference}.
\newblock \emph{The Annals of Applied Probability}, 34\penalty0 (1A):\penalty0 428 -- 500, 2024.
\newblock \doi{10.1214/23-AAP1969}.
\newblock URL \url{https://doi.org/10.1214/23-AAP1969}.

\bibitem[Bengio et~al.(2013)Bengio, Courville, and Vincent]{bengio2013representation}
Yoshua Bengio, Aaron Courville, and Pascal Vincent.
\newblock Representation learning: A review and new perspectives.
\newblock \emph{IEEE Trans. Pattern Anal. Mach. Intell.}, 35\penalty0 (8):\penalty0 1798–1828, August 2013.
\newblock ISSN 0162-8828.
\newblock \doi{10.1109/TPAMI.2013.50}.
\newblock URL \url{https://doi.org/10.1109/TPAMI.2013.50}.

\bibitem[Frankle and Carbin(2019)]{frankle2018the}
Jonathan Frankle and Michael Carbin.
\newblock The lottery ticket hypothesis: Finding sparse, trainable neural networks.
\newblock In \emph{International Conference on Learning Representations}, 2019.
\newblock URL \url{https://openreview.net/forum?id=rJl-b3RcF7}.

\bibitem[Zhang et~al.(2021)Zhang, Jin, Zhang, Zhou, Zhao, Ren, Liu, Wu, Jin, and Dou]{zhang2021validatinglottery}
Zeru Zhang, Jiayin Jin, Zijie Zhang, Yang Zhou, Xin Zhao, Jiaxiang Ren, Ji~Liu, Lingfei Wu, Ruoming Jin, and Dejing Dou.
\newblock Validating the lottery ticket hypothesis with inertial manifold theory.
\newblock In \emph{Proceedings of the 35th International Conference on Neural Information Processing Systems}, NIPS '21, Red Hook, NY, USA, 2021. Curran Associates Inc.
\newblock ISBN 9781713845393.

\bibitem[Liu et~al.(2024{\natexlab{b}})Liu, Zhang, He, Wang, Xiao, Ye, Zhou, Ku, and Hui]{liu2024surveylotterytickethypothesis}
Bohan Liu, Zijie Zhang, Peixiong He, Zhensen Wang, Yang Xiao, Ruimeng Ye, Yang Zhou, Wei-Shinn Ku, and Bo~Hui.
\newblock A survey of lottery ticket hypothesis, 2024{\natexlab{b}}.
\newblock URL \url{https://arxiv.org/abs/2403.04861}.

\bibitem[Cheng et~al.(2024)Cheng, Zhang, and Shi]{cheng2024asurveyonpruning}
Hongrong Cheng, Miao Zhang, and Javen~Qinfeng Shi.
\newblock A survey on deep neural network pruning: Taxonomy, comparison, analysis, and recommendations.
\newblock \emph{IEEE Transactions on Pattern Analysis and Machine Intelligence}, 46\penalty0 (12):\penalty0 10558--10578, 2024.
\newblock \doi{10.1109/TPAMI.2024.3447085}.

\bibitem[Hecht-Nielsen(1990)]{hechtnielsen1990algebraicstructure}
Robert Hecht-Nielsen.
\newblock On the algebraic structure of feedforward network weight spaces.
\newblock In Rolf Eckmiller, editor, \emph{Advanced Neural Computers}, pages 129--135. North-Holland, Amsterdam, 1990.
\newblock ISBN 978-0-444-88400-8.
\newblock \doi{https://doi.org/10.1016/B978-0-444-88400-8.50019-4}.
\newblock URL \url{https://www.sciencedirect.com/science/article/pii/B9780444884008500194}.

\bibitem[Chen et~al.(1993)Chen, Lu, and Hecht-Nielsen]{chen1993geometryfeedforward}
An~Mei Chen, Haw-minn Lu, and Robert Hecht-Nielsen.
\newblock On the geometry of feedforward neural network error surfaces.
\newblock \emph{Neural Computation}, 5\penalty0 (6):\penalty0 910--927, 1993.
\newblock \doi{10.1162/neco.1993.5.6.910}.

\bibitem[Domingo-Enrich et~al.(2025)Domingo-Enrich, Drozdzal, Karrer, and Chen]{domingo-enrichAdjointMatchingFinetuning2024}
Carles Domingo-Enrich, Michal Drozdzal, Brian Karrer, and Ricky T.~Q. Chen.
\newblock Adjoint matching: Fine-tuning flow and diffusion generative models with memoryless stochastic optimal control.
\newblock In \emph{The Thirteenth International Conference on Learning Representations}, 2025.
\newblock URL \url{https://openreview.net/forum?id=xQBRrtQM8u}.

\bibitem[Chen et~al.(2019)Chen, Rubanova, Bettencourt, and Duvenaud]{chenNeuralOrdinaryDifferential2019}
Ricky T.~Q. Chen, Yulia Rubanova, Jesse Bettencourt, and David Duvenaud.
\newblock Neural {{Ordinary Differential Equations}}, December 2019.

\bibitem[Shaul et~al.(2023)Shaul, Chen, Nickel, Le, and Lipman]{shaul2023kineticoptimal}
Neta Shaul, Ricky T.~Q. Chen, Maximilian Nickel, Matt Le, and Yaron Lipman.
\newblock On kinetic optimal probability paths for generative models.
\newblock In \emph{Proceedings of the 40th International Conference on Machine Learning}, ICML'23. JMLR.org, 2023.

\bibitem[Neklyudov et~al.(2023)Neklyudov, Brekelmans, Severo, and Makhzani]{neklyudov2023actionmatching}
Kirill Neklyudov, Rob Brekelmans, Daniel Severo, and Alireza Makhzani.
\newblock Action matching: learning stochastic dynamics from samples.
\newblock In \emph{Proceedings of the 40th International Conference on Machine Learning}, ICML'23. JMLR.org, 2023.

\bibitem[Neklyudov et~al.(2024)Neklyudov, Brekelmans, Tong, Atanackovic, Liu, and Makhzani]{neklyudov2024wassersteinlagrangian}
Kirill Neklyudov, Rob Brekelmans, Alexander Tong, Lazar Atanackovic, Qiang Liu, and Alireza Makhzani.
\newblock A computational framework for solving wasserstein lagrangian flows.
\newblock In \emph{Proceedings of the 41st International Conference on Machine Learning}, ICML'24. JMLR.org, 2024.

\bibitem[Kapusniak et~al.(2024)Kapusniak, Potaptchik, Reu, Zhang, Tong, Bronstein, Bose, and Di~Giovanni]{kapusniakMetricFlowMatching2024}
Kacper Kapusniak, Peter Potaptchik, Teodora Reu, Leo Zhang, Alexander Tong, Michael Bronstein, Avishek~Joey Bose, and Francesco Di~Giovanni.
\newblock Metric {{Flow Matching}} for {{Smooth Interpolations}} on the {{Data Manifold}}, May 2024.

\bibitem[Liu et~al.(2024{\natexlab{c}})Liu, Lipman, Nickel, Karrer, Theodorou, and Chen]{liu2024generalized}
Guan-Horng Liu, Yaron Lipman, Maximilian Nickel, Brian Karrer, Evangelos Theodorou, and Ricky T.~Q. Chen.
\newblock Generalized schr\"odinger bridge matching.
\newblock In \emph{The Twelfth International Conference on Learning Representations}, 2024{\natexlab{c}}.
\newblock URL \url{https://openreview.net/forum?id=SoismgeX7z}.

\bibitem[Rohbeck et~al.(2025)Rohbeck, Bunne, Brouwer, Huetter, Biton, Chen, Regev, and Lopez]{rohbeck2025modeling}
Martin Rohbeck, Charlotte Bunne, Edward~De Brouwer, Jan-Christian Huetter, Anne Biton, Kelvin~Y. Chen, Aviv Regev, and Romain Lopez.
\newblock Modeling complex system dynamics with flow matching across time and conditions.
\newblock In \emph{The Thirteenth International Conference on Learning Representations}, 2025.
\newblock URL \url{https://openreview.net/forum?id=hwnObmOTrV}.

\bibitem[Liu et~al.(2020)Liu, Xu, Lu, Zhang, Gretton, and Sutherland]{liu2020learningdeepkernels}
Feng Liu, Wenkai Xu, Jie Lu, Guangquan Zhang, Arthur Gretton, and Danica~J. Sutherland.
\newblock Learning deep kernels for non-parametric two-sample tests.
\newblock In \emph{Proceedings of the 37th International Conference on Machine Learning}, ICML'20. JMLR.org, 2020.

\bibitem[Zhao et~al.(2022)Zhao, Sinha, He, Perreault, Song, and Ermon]{zhao2022comparing}
Shengjia Zhao, Abhishek Sinha, Yutong He, Aidan Perreault, Jiaming Song, and Stefano Ermon.
\newblock Comparing distributions by measuring differences that affect decision making.
\newblock In \emph{International Conference on Learning Representations}, 2022.
\newblock URL \url{https://openreview.net/forum?id=KB5onONJIAU}.

\bibitem[Ginsberg et~al.(2023)Ginsberg, Liang, and Krishnan]{ginsberg2023harmfulcov}
Tom Ginsberg, Zhongyuan Liang, and Rahul~G Krishnan.
\newblock A learning based hypothesis test for harmful covariate shift.
\newblock In \emph{The Eleventh International Conference on Learning Representations}, 2023.
\newblock URL \url{https://openreview.net/forum?id=rdfgqiwz7lZ}.

\bibitem[Goldwasser et~al.(2020)Goldwasser, Kalai, Kalai, and Montasser]{goldwasser2020beyondperturbations}
Shafi Goldwasser, Adam~Tauman Kalai, Yael~Tauman Kalai, and Omar Montasser.
\newblock Beyond perturbations: learning guarantees with arbitrary adversarial test examples.
\newblock In \emph{Proceedings of the 34th International Conference on Neural Information Processing Systems}, NIPS '20, Red Hook, NY, USA, 2020. Curran Associates Inc.
\newblock ISBN 9781713829546.

\bibitem[Santambrogio(2015)]{santambrogio2015optimal}
F.~Santambrogio.
\newblock \emph{Optimal Transport for Applied Mathematicians: {{Calculus}} of Variations, Pdes, and Modeling}.
\newblock Progress in Nonlinear Differential Equations and Their Applications. Springer International Publishing, 2015.
\newblock ISBN 978-3-319-20828-2.

\bibitem[Terpin et~al.(2024)Terpin, Lanzetti, Gadea, and Dorfler]{terpin2024learning}
Antonio Terpin, Nicolas Lanzetti, Mart{\'\i}n Gadea, and Florian Dorfler.
\newblock Learning diffusion at lightspeed.
\newblock In \emph{The Thirty-eighth Annual Conference on Neural Information Processing Systems}, 2024.
\newblock URL \url{https://openreview.net/forum?id=y10avdRFNK}.

\bibitem[Bunne et~al.(2022)Bunne, Meng-Papaxanthos, Krause, and Cuturi]{bunne2022proximal}
Charlotte Bunne, Laetitia Meng-Papaxanthos, Andreas Krause, and Marco Cuturi.
\newblock {Proximal Optimal Transport Modeling of Population Dynamics}.
\newblock In \emph{International Conference on Artificial Intelligence and Statistics (AISTATS)}, volume~25, 2022.

\bibitem[Jordan et~al.(1998)Jordan, Kinderlehrer, and Otto]{jko1998variationalformulation}
Richard Jordan, David Kinderlehrer, and Felix Otto.
\newblock The variational formulation of the fokker--planck equation.
\newblock \emph{SIAM Journal on Mathematical Analysis}, 29\penalty0 (1):\penalty0 1--17, 1998.
\newblock \doi{10.1137/S0036141096303359}.
\newblock URL \url{https://doi.org/10.1137/S0036141096303359}.

\bibitem[Kingma and Welling(2022)]{kingma2022autoencodingvariationalbayes}
Diederik~P Kingma and Max Welling.
\newblock Auto-encoding variational bayes, 2022.
\newblock URL \url{https://arxiv.org/abs/1312.6114}.

\bibitem[Kofinas et~al.(2024)Kofinas, Knyazev, Zhang, Chen, Burghouts, Gavves, Snoek, and Zhang]{kofinasGraphNeuralNetworks2024}
Miltiadis Kofinas, Boris Knyazev, Yan Zhang, Yunlu Chen, Gertjan~J. Burghouts, Efstratios Gavves, Cees G.~M. Snoek, and David~W. Zhang.
\newblock Graph {{Neural Networks}} for {{Learning Equivariant Representations}} of {{Neural Networks}}.
\newblock https://arxiv.org/abs/2403.12143v3, March 2024.

\bibitem[Putterman et~al.(2024)Putterman, Lim, Gelberg, Jegelka, and Maron]{putterman2024learningloras}
Theo Putterman, Derek Lim, Yoav Gelberg, Stefanie Jegelka, and Haggai Maron.
\newblock Learning on loras: Gl-equivariant processing of low-rank weight spaces for large finetuned models, 2024.
\newblock URL \url{https://arxiv.org/abs/2410.04207}.

\bibitem[Sch{\"u}rholt et~al.(2024)Sch{\"u}rholt, Mahoney, and Borth]{schuerholt2024sane}
Konstantin Sch{\"u}rholt, Michael~W. Mahoney, and Damian Borth.
\newblock Towards scalable and versatile weight space learning.
\newblock In \emph{Proceedings of the 41st International Conference on Machine Learning (ICML)}. PMLR, 2024.

\bibitem[He et~al.(2015{\natexlab{a}})He, Zhang, Ren, and Sun]{he2015delvingdeep}
Kaiming He, Xiangyu Zhang, Shaoqing Ren, and Jian Sun.
\newblock Delving deep into rectifiers: Surpassing human-level performance on imagenet classification.
\newblock In \emph{2015 IEEE International Conference on Computer Vision (ICCV)}, pages 1026--1034, 2015{\natexlab{a}}.
\newblock \doi{10.1109/ICCV.2015.123}.

\bibitem[Wang et~al.(2024)Wang, Tang, Zeng, Yin, Xu, Zhou, Zang, Darrell, Liu, and You]{wangNeuralNetworkDiffusion2024}
Kai Wang, Dongwen Tang, Boya Zeng, Yida Yin, Zhaopan Xu, Yukun Zhou, Zelin Zang, Trevor Darrell, Zhuang Liu, and Yang You.
\newblock Neural network diffusion, 2024.
\newblock URL \url{https://arxiv.org/abs/2402.13144}.

\bibitem[He et~al.(2015{\natexlab{b}})He, Zhang, Ren, and Sun]{He2015DeepRL}
Kaiming He, X.~Zhang, Shaoqing Ren, and Jian Sun.
\newblock Deep residual learning for image recognition.
\newblock \emph{2016 IEEE Conference on Computer Vision and Pattern Recognition (CVPR)}, pages 770--778, 2015{\natexlab{b}}.
\newblock URL \url{https://api.semanticscholar.org/CorpusID:206594692}.

\bibitem[Dosovitskiy et~al.(2021)Dosovitskiy, Beyer, Kolesnikov, Weissenborn, Zhai, Unterthiner, Dehghani, Minderer, Heigold, Gelly, Uszkoreit, and Houlsby]{dosovitskiy2021animage}
Alexey Dosovitskiy, Lucas Beyer, Alexander Kolesnikov, Dirk Weissenborn, Xiaohua Zhai, Thomas Unterthiner, Mostafa Dehghani, Matthias Minderer, Georg Heigold, Sylvain Gelly, Jakob Uszkoreit, and Neil Houlsby.
\newblock An image is worth 16x16 words: Transformers for image recognition at scale.
\newblock In \emph{International Conference on Learning Representations}, 2021.
\newblock URL \url{https://openreview.net/forum?id=YicbFdNTTy}.

\bibitem[Liu et~al.(2022)Liu, Mao, Wu, Feichtenhofer, Darrell, and Xie]{liu2022convnet2020s}
Zhuang Liu, Hanzi Mao, Chao-Yuan Wu, Christoph Feichtenhofer, Trevor Darrell, and Saining Xie.
\newblock A convnet for the 2020s, 2022.
\newblock URL \url{https://arxiv.org/abs/2201.03545}.

\bibitem[Sch{\"u}rholt et~al.(2022)Sch{\"u}rholt, Taskiran, Knyazev, Gir{\'o}-i Nieto, and Borth]{schurholtModelZoosDataset2022}
Konstantin Sch{\"u}rholt, Diyar Taskiran, Boris Knyazev, Xavier Gir{\'o}-i Nieto, and Damian Borth.
\newblock Model zoos: A dataset of diverse populations of neural network models.
\newblock In \emph{Thirty-Sixth Conference on Neural Information Processing Systems (NeurIPS) Track on Datasets and Benchmarks}, September 2022.

\bibitem[Radford et~al.(2021)Radford, Kim, Hallacy, Ramesh, Goh, Agarwal, Sastry, Askell, Mishkin, Clark, Krueger, and Sutskever]{Radford2021LearningTV}
Alec Radford, Jong~Wook Kim, Chris Hallacy, Aditya Ramesh, Gabriel Goh, Sandhini Agarwal, Girish Sastry, Amanda Askell, Pamela Mishkin, Jack Clark, Gretchen Krueger, and Ilya Sutskever.
\newblock Learning transferable visual models from natural language supervision.
\newblock In \emph{International Conference on Machine Learning}, 2021.
\newblock URL \url{https://api.semanticscholar.org/CorpusID:231591445}.

\bibitem[Recht et~al.(2019)Recht, Roelofs, Schmidt, and Shankar]{recht2019imagenet}
Benjamin Recht, Rebecca Roelofs, Ludwig Schmidt, and Vaishaal Shankar.
\newblock Do imagenet classifiers generalize to imagenet?
\newblock In \emph{International conference on machine learning}, pages 5389--5400. PMLR, 2019.

\bibitem[Veeling et~al.(2018)Veeling, Linmans, Winkens, Cohen, and Welling]{veeling2018rotation}
Bastiaan~S Veeling, Jasper Linmans, Jim Winkens, Taco Cohen, and Max Welling.
\newblock Rotation equivariant cnns for digital pathology.
\newblock In \emph{Medical image computing and computer assisted intervention--mICCAI 2018: 21st international conference, granada, Spain, September 16-20, 2018, proceedings, part II 11}, pages 210--218. Springer, 2018.

\bibitem[Skreta et~al.(2025)Skreta, Atanackovic, Bose, Tong, and Neklyudov]{skreta2025superpositiondiff}
Marta Skreta, Lazar Atanackovic, Joey Bose, Alexander Tong, and Kirill Neklyudov.
\newblock The superposition of diffusion models using the it\^o density estimator.
\newblock In \emph{The Thirteenth International Conference on Learning Representations}, 2025.
\newblock URL \url{https://openreview.net/forum?id=2o58Mbqkd2}.

\bibitem[Lanzetti et~al.(2024)Lanzetti, Terpin, and Dörfler]{lanzetti2024variationalanalysis}
Nicolas Lanzetti, Antonio Terpin, and Florian Dörfler.
\newblock Variational analysis in the wasserstein space, 2024.
\newblock URL \url{https://arxiv.org/abs/2406.10676}.

\bibitem[Parthasarathy(2005)]{parthasarathy2005probability}
K.R. Parthasarathy.
\newblock \emph{Probability Measures on Metric Spaces}.
\newblock {{AMS}} Chelsea Publishing Series. AMS Chelsea Pub., 2005.
\newblock ISBN 978-0-8218-3889-1.

\bibitem[Ambrosio et~al.(2006)Ambrosio, Gigli, and Savare]{ambrosio2006gradient}
L.~Ambrosio, N.~Gigli, and G.~Savare.
\newblock \emph{Gradient Flows: In Metric Spaces and in the Space of Probability Measures}.
\newblock Lectures in Mathematics. ETH Z{\"u}rich. Birkh{\"a}user Basel, 2006.
\newblock ISBN 9783764373092.
\newblock URL \url{https://books.google.ca/books?id=Hk_wNp0sc4gC}.

\bibitem[Villani(2008)]{villani2008optimal}
C.~Villani.
\newblock \emph{Optimal Transport: {{Old}} and New}.
\newblock Grundlehren Der Mathematischen {{Wissenschaften}}. Springer Berlin Heidelberg, 2008.
\newblock ISBN 978-3-540-71050-9.

\bibitem[Schachter(2017)]{schachter2017eulerian}
Benjamin Schachter.
\newblock \emph{An Eulerian Approach to Optimal Transport with Applications to the Otto Calculus}.
\newblock PhD thesis, University of Toronto, 2017.

\bibitem[Pilanci and Ergen(2020)]{pilanci2020neuralnetworksconvexregularizers}
Mert Pilanci and Tolga Ergen.
\newblock Neural networks are convex regularizers: Exact polynomial-time convex optimization formulations for two-layer networks, 2020.
\newblock URL \url{https://arxiv.org/abs/2002.10553}.

\bibitem[Pooladian et~al.(2023)Pooladian, Ben-Hamu, Domingo-Enrich, Amos, Lipman, and Chen]{pooladian2023multisample}
Aram-Alexandre Pooladian, Heli Ben-Hamu, Carles Domingo-Enrich, Brandon Amos, Yaron Lipman, and Ricky T.~Q. Chen.
\newblock Multisample flow matching: straightening flows with minibatch couplings.
\newblock In \emph{Proceedings of the 40th International Conference on Machine Learning}, ICML'23. JMLR.org, 2023.

\bibitem[Pooladian et~al.(2024)Pooladian, Domingo-Enrich, Chen, and Amos]{pooladian2024neural}
Aram-Alexandre Pooladian, Carles Domingo-Enrich, Ricky T.~Q. Chen, and Brandon Amos.
\newblock Neural optimal transport with lagrangian costs.
\newblock In \emph{The 40th Conference on Uncertainty in Artificial Intelligence}, 2024.
\newblock URL \url{https://openreview.net/forum?id=x4paJ2sJyZ}.

\bibitem[Tong et~al.(2020)Tong, Huang, Wolf, {van Dijk}, and Krishnaswamy]{tong2020trajectorynet}
Alexander Tong, Jessie Huang, Guy Wolf, David {van Dijk}, and Smita Krishnaswamy.
\newblock Trajectorynet: A dynamic optimal transport network for modeling cellular dynamics.
\newblock In \emph{Proceedings of the 37th International Conference on Machine Learning}, 2020.

\bibitem[Denil et~al.(2014)Denil, Shakibi, Dinh, Ranzato, and {de Freitas}]{denilPredictingParametersDeep2014}
Misha Denil, Babak Shakibi, Laurent Dinh, Marc'Aurelio Ranzato, and Nando {de Freitas}.
\newblock Predicting {{Parameters}} in {{Deep Learning}}, October 2014.

\bibitem[Ha et~al.(2017)Ha, Dai, and Le]{ha2017hypernetworks}
David Ha, Andrew~M. Dai, and Quoc~V. Le.
\newblock Hypernetworks.
\newblock In \emph{International Conference on Learning Representations}, 2017.

\bibitem[Peebles et~al.(2022)Peebles, Radosavovic, Brooks, Efros, and Malik]{peebles2022learning}
William Peebles, Ilija Radosavovic, Tim Brooks, Alexei~A. Efros, and Jitendra Malik.
\newblock Learning to learn with generative models of neural network checkpoints, 2022.

\bibitem[Wang et~al.(2025)Wang, Tang, Zhao, Schürholt, Wang, and You]{wang2025recurrentdiffusionlargescaleparameter}
Kai Wang, Dongwen Tang, Wangbo Zhao, Konstantin Schürholt, Zhangyang Wang, and Yang You.
\newblock Recurrent diffusion for large-scale parameter generation, 2025.
\newblock URL \url{https://arxiv.org/abs/2501.11587}.

\bibitem[Hu et~al.(2022)Hu, Li, Stühmer, Kim, and Hospedales]{huPushingLimitsSimple2022}
Shell~Xu Hu, Da~Li, Jan Stühmer, Minyoung Kim, and Timothy~M. Hospedales.
\newblock Pushing the limits of simple pipelines for few-shot learning: External data and fine-tuning make a difference.
\newblock In \emph{2022 IEEE/CVF Conference on Computer Vision and Pattern Recognition (CVPR)}, pages 9058--9067, 2022.
\newblock \doi{10.1109/CVPR52688.2022.00886}.

\bibitem[Zhmoginov et~al.(2022)Zhmoginov, Sandler, and Vladymyrov]{zhmoginovHyperTransformerModelGeneration2022a}
Andrey Zhmoginov, Mark Sandler, and Max Vladymyrov.
\newblock {{HyperTransformer}}: {{Model Generation}} for {{Supervised}} and {{Semi-Supervised Few-Shot Learning}}, July 2022.

\bibitem[Fifty et~al.(2024)Fifty, Duan, Junkins, Amid, Leskovec, Re, and Thrun]{fiftyContextAwareMetaLearning2024}
Christopher Fifty, Dennis Duan, Ronald~Guenther Junkins, Ehsan Amid, Jure Leskovec, Christopher Re, and Sebastian Thrun.
\newblock Context-aware meta-learning.
\newblock In \emph{The Twelfth International Conference on Learning Representations}, 2024.
\newblock URL \url{https://openreview.net/forum?id=lJYAkDVnRU}.

\bibitem[Ravi and Larochelle(2017)]{raviOptimizationModelFewShot2017}
Sachin Ravi and Hugo Larochelle.
\newblock Optimization as a {{Model}} for {{Few-Shot Learning}}.
\newblock In \emph{International {{Conference}} on {{Learning Representations}}}, February 2017.

\bibitem[Lee et~al.(2023)Lee, Xie, Pacchiano, Chandak, Finn, Nachum, and Brunskill]{leeSupervisedPretrainingCan2023}
Jonathan Lee, Annie Xie, Aldo Pacchiano, Yash Chandak, Chelsea Finn, Ofir Nachum, and Emma Brunskill.
\newblock Supervised {{Pretraining Can Learn In-Context Reinforcement Learning}}.
\newblock \emph{Advances in Neural Information Processing Systems}, 36:\penalty0 43057--43083, December 2023.

\bibitem[Kirsch et~al.(2024)Kirsch, Harrison, Sohl-Dickstein, and Metz]{kirsch2024generalpurposeincontextlearningmetalearning}
Louis Kirsch, James Harrison, Jascha Sohl-Dickstein, and Luke Metz.
\newblock General-purpose in-context learning by meta-learning transformers, 2024.
\newblock URL \url{https://arxiv.org/abs/2212.04458}.

\bibitem[Du et~al.(2023)Du, Xiao, Liao, and Snoek]{duProtoDiffLearningLearn2023}
Yingjun Du, Zehao Xiao, Shengcai Liao, and Cees Snoek.
\newblock {{ProtoDiff}}: {{Learning}} to {{Learn Prototypical Networks}} by {{Task-Guided Diffusion}}, November 2023.

\bibitem[Knyazev et~al.(2021)Knyazev, Drozdzal, Taylor, and Romero-Soriano]{knyazev2021parameter}
Boris Knyazev, Michal Drozdzal, Graham~W Taylor, and Adriana Romero-Soriano.
\newblock Parameter prediction for unseen deep architectures.
\newblock In \emph{Advances in Neural Information Processing Systems}, 2021.

\bibitem[Knyazev et~al.(2023)Knyazev, Hwang, and Lacoste-Julien]{knyazev2023canwescale}
Boris Knyazev, Doha Hwang, and Simon Lacoste-Julien.
\newblock Can we scale transformers to predict parameters of diverse imagenet models?
\newblock In \emph{International Conference on Machine Learning}, 2023.

\bibitem[Li et~al.(2021)Li, Malladi, and Arora]{li2021modelingsgdwithsdes}
Zhiyuan Li, Sadhika Malladi, and Sanjeev Arora.
\newblock On the validity of modeling {SGD} with stochastic differential equations ({SDE}s).
\newblock In A.~Beygelzimer, Y.~Dauphin, P.~Liang, and J.~Wortman Vaughan, editors, \emph{Advances in Neural Information Processing Systems}, 2021.
\newblock URL \url{https://openreview.net/forum?id=goEdyJ_nVQI}.

\bibitem[Berlinghieri et~al.(2025)Berlinghieri, Shen, and Broderick]{berlinghieri2025beyond}
Renato Berlinghieri, Yunyi Shen, and Tamara Broderick.
\newblock Beyond schr\"odinger bridges: A least-squares approach for learning stochastic dynamics with unknown volatility.
\newblock In \emph{7th Symposium on Advances in Approximate Bayesian Inference {\textendash} Workshop Track}, 2025.
\newblock URL \url{https://openreview.net/forum?id=GJdo8LlNX1}.

\bibitem[Chen et~al.(2023)Chen, Liu, Tao, and Theodorou]{chen2023deep}
Tianrong Chen, Guan-Horng Liu, Molei Tao, and Evangelos Theodorou.
\newblock Deep momentum multi-marginal schr\"odinger bridge.
\newblock In \emph{Thirty-seventh Conference on Neural Information Processing Systems}, 2023.
\newblock URL \url{https://openreview.net/forum?id=ykvvv0gc4R}.

\bibitem[Shen et~al.(2025)Shen, Berlinghieri, and Broderick]{shen2025multimarginal}
Yunyi Shen, Renato Berlinghieri, and Tamara Broderick.
\newblock Multi-marginal schr\"odinger bridges with iterative reference refinement.
\newblock In \emph{The 28th International Conference on Artificial Intelligence and Statistics}, 2025.
\newblock URL \url{https://openreview.net/forum?id=VcwZ3gtYFY}.

\bibitem[Du et~al.(2024)Du, Plainer, Brekelmans, Duan, Noe, Gomes, Aspuru-Guzik, and Neklyudov]{du2024doob}
Yuanqi Du, Michael Plainer, Rob Brekelmans, Chenru Duan, Frank Noe, Carla~P Gomes, Alan Aspuru-Guzik, and Kirill Neklyudov.
\newblock Doob's lagrangian: A sample-efficient variational approach to transition path sampling.
\newblock \emph{Advances in Neural Information Processing Systems}, 37:\penalty0 65791--65822, 2024.

\bibitem[Krizhevsky and Hinton(2009)]{krizhevsky2009cifar}
A.~Krizhevsky and G.~Hinton.
\newblock Learning multiple layers of features from tiny images.
\newblock \emph{Master's thesis}, 2009.

\bibitem[Coates et~al.(2011)Coates, Ng, and Lee]{coates2011stl10}
Adam Coates, Andrew Ng, and Honglak Lee.
\newblock An analysis of single-layer networks in unsupervised feature learning.
\newblock In \emph{AISTATS}, 2011.

\bibitem[Xiao et~al.(2017)Xiao, Rasul, and Vollgraf]{xiao2017fmnist}
Han Xiao, Kashif Rasul, and Roland Vollgraf.
\newblock Fashion-mnist: a novel image dataset for benchmarking machine learning algorithms.
\newblock \emph{arXiv preprint arXiv:1708.07747}, 2017.

\bibitem[Wightman(2019)]{rw2019timm}
Ross Wightman.
\newblock Pytorch image models.
\newblock \url{https://github.com/rwightman/pytorch-image-models}, 2019.

\bibitem[Falk et~al.(2025)Falk, Meynent, Pfammatter, Schürholt, and Borth]{falk2025modelzoovisiontransformers}
Damian Falk, Léo Meynent, Florence Pfammatter, Konstantin Schürholt, and Damian Borth.
\newblock A model zoo of vision transformers, 2025.
\newblock URL \url{https://arxiv.org/abs/2504.10231}.

\bibitem[Schürholt et~al.(2025)Schürholt, Meynent, Zhou, Lu, Yang, and Borth]{schürholt2025modelzoophasetransitions}
Konstantin Schürholt, Léo Meynent, Yefan Zhou, Haiquan Lu, Yaoqing Yang, and Damian Borth.
\newblock A model zoo on phase transitions in neural networks, 2025.
\newblock URL \url{https://arxiv.org/abs/2504.18072}.

\bibitem[Ainsworth et~al.(2023)Ainsworth, Hayase, and Srinivasa]{ainsworthGitReBasinMerging2023}
Samuel~K. Ainsworth, Jonathan Hayase, and Siddhartha Srinivasa.
\newblock Git {{Re-Basin}}: {{Merging Models}} modulo {{Permutation Symmetries}}, March 2023.

\bibitem[Tran et~al.(2024)Tran, Vo, Huu, Nguyen, et~al.]{tran2024monomial}
Hoang Tran, Thieu Vo, Tho Huu, Tan Nguyen, et~al.
\newblock Monomial matrix group equivariant neural functional networks.
\newblock \emph{Advances in Neural Information Processing Systems}, 37:\penalty0 48628--48665, 2024.

\bibitem[Zhou et~al.(2023)Zhou, Yang, Burns, Cardace, Jiang, Sokota, Kolter, and Finn]{zhou2023permutation}
Allan Zhou, Kaien Yang, Kaylee Burns, Adriano Cardace, Yiding Jiang, Samuel Sokota, J~Zico Kolter, and Chelsea Finn.
\newblock Permutation equivariant neural functionals.
\newblock In \emph{Thirty-seventh Conference on Neural Information Processing Systems}, 2023.
\newblock URL \url{https://openreview.net/forum?id=fmYmXNPmhv}.

\bibitem[Sitzmann et~al.(2020)Sitzmann, Martel, Bergman, Lindell, and Wetzstein]{sitzmann2020implicit}
Vincent Sitzmann, Julien Martel, Alexander Bergman, David Lindell, and Gordon Wetzstein.
\newblock Implicit neural representations with periodic activation functions.
\newblock \emph{Advances in neural information processing systems}, 33:\penalty0 7462--7473, 2020.

\bibitem[Zeng et~al.(2025)Zeng, Yin, Xu, and Liu]{zeng2025weightsgeneralization}
Boya Zeng, Yida Yin, Zhiqiu Xu, and Zhuang Liu.
\newblock Generative modeling of weights: Generalization or memorization?, 2025.
\newblock URL \url{https://arxiv.org/abs/2506.07998}.

\end{thebibliography}
\bibliographystyle{unsrtnat}
}


\clearpage
\appendix
\thispagestyle{empty}

\onecolumn
\aistatstitle{Flows and Diffusions on the Neural Manifold: \\
Supplementary Materials}

\section*{Appendix}
This appendix consist of details left out in the main text. First, we complete the proofs of the results in the text. Next, we perform a more comprehensive review of the related literature. Following, we fill in a few technical results that were promised in the main text, before closing with architecture and training settings.

\section{Further Theory}
\label{app:proofs}

\subsection{Proof of Theorem \ref{thm:ce}}
\mybox{
\textbf{Theorem \ref{thm:ce}}:

Let $\theta_0 \sim p_0$ be the initialized network parameters residing on $\Omega \subset \R^p$ a compact set. Suppose that $\gL : \Omega \to \R$ is $C^1$ and the gradient descent curve $(\theta_t)_{t \geq 0}$ reside in $\Omega$. Define $p_t = \mathrm{Law}(\theta_t)$ and further assume $p_t > 0$ a.e., then it satisfies the continuity equation Eqn. \ref{eq:ce}.
}
\paragraph{Remark.} The idea is to view GD as an iterated minimization scheme on $\gP(\Omega)$ with functional $\gF(p_t) = \int_\Omega\gL(x) \, d p_t(x)$. Alternatively, this may be viewed as the necessary first-order optimality conditions of a JKO scheme \citep{lanzetti2024variationalanalysis, terpin2024learning}. The proof closely follows \citet[Ch. 8]{santambrogio2015optimal}.

\begin{proof}[Proof of Theorem \ref{thm:ce}]
    We follow the reasoning of \citet[Ch. 8]{santambrogio2015optimal}. First, we provide context as for the discretized formulation of a gradient flow. Let $F: \R^d \to \bar \R$ be lower semi-continuous and is bounded below as $F(x) \geq C_1 - C_2 |x|^2$ for some $C_1, C_2 \geq 0$. Consider the formal problem 
    \[
    \begin{cases}
        x'(t) = -\nabla F(x)\\
        x(0) = x_0
    \end{cases}.
    \]
    This is understood as a Cauchy problem which happens to be a gradient flow. If we fix a small time step $\tau > 0$, this problem has a discretization \[
    x_{k + 1}^\tau \in \argmin_{x \in \R^d} F(x) + \frac{|x - x_k^\tau|^2}{2\tau}
    \]
    We now generalize this discretization scheme and show that its limit is a solution to the gradient flow. 
    
    Define $J(p) := \int_\Omega \gL \, dp$. We now have the scheme \[
    p_{k + 1}^\tau \in \argmin_{p \in \gP(\Omega)} J(p) + \frac{W_2^2(p, p_k^\tau)}{2\tau}.
    \]

    \textbf{Claim:}
        The minimization above produces a minimizer $p = p_{k + 1}^\tau$. If we modify 
        \[
        \tilde J(p) = \begin{cases}
            J(p) & p \ll \mathrm{Leb}^p\\
            +\infty & \text{otherwise}
        \end{cases}
        \]
        then this is a unique minimizer.

    \begin{proof}
        As $\Omega$ is compact, it's a well-known result that $\gP(\Omega)$ is compact in the weak topology \citep{parthasarathy2005probability}. Moreover, \citet[Prop. 7.1]{santambrogio2015optimal} gives lower semi-continuity of $J$, which is enough for existence. Moreover, from \citet[Prop. 7.19]{santambrogio2015optimal}, we have: if $p \ll \mathrm{Leb}^p$, then $W_2^2(\cdot, p)$ is strictly convex. Since $J(p)$ is linear in $p$, we have that the minimization objective is strictly convex with the modification to $\tilde J$, thus the minimizer is unique.
     \end{proof}

     \textbf{Note:} due to the nice properties of having $p \ll \mathrm{Leb}^p$, and considering that we do not lose much generality, we will use $\tilde J$ for the rest of the proof.

     \textbf{Claim:} The first variation $\frac{\delta J}{\delta p}(p) = \gL$ for all $ p \in \gP(\Omega)$.

     \begin{proof}
         We note that since $\Omega$ is compact, $\gL$ is presumed continuous, and $p$ a finite measure, we have that $J(p) < \infty$ for all $p \in \gP(\Omega)$. This is sufficient to show that $p$ is always regular for $J$. Moreover, the first variation satisfies:
         \[
         \frac{d}{d\epsilon}\bigg\vert_{\epsilon = 0} J(p + \epsilon\chi) = \int \frac{\delta J}{\delta p}(p) \, d \chi.
         \]
         By linearity: $J(p + \epsilon \chi) = J(p) + \epsilon \int \gL \, d \chi$, and therefore $\frac{\delta J}{\delta p}(p) = \gL$.
     \end{proof}

It was remarked \citet[Remark 7.13]{santambrogio2015optimal} that the first variation of the transport cost $\gT_c$ for a continuous cost $c$ is exactly the Kantorovich potential $\varphi$ \textit{only if it is unique}. In our case, uniqueness stems from \citet[Prop. 7.18]{santambrogio2015optimal}. Given our preparation, we are now ready to state the main Lemma.

\begin{lemma}
    \label{lem:velocity-iterates}
    Let $T_{k + 1}^\tau$ be the optimal transport map from $p_{k + 1}^\tau$ to $p_k^\tau$ (note the reverse), then we have the velocity
    \begin{equation}
        v_{k + 1}^\tau := \frac{\iota - T_{k + 1}^\tau}{\tau} = -\nabla\gL \quad \text{a.e.}
    \end{equation}
\end{lemma}
\begin{proof}
    Since the first variation is linear in the functional argument, we have that the first variation \[
    \frac{\delta(J + W_2^2(\cdot, p_k^\tau) / 2\tau)}{\delta p}(p) = \gL + \varphi / \tau.
    \]
    By \citet[Prop. 7.20]{santambrogio2015optimal}, we have $\gL + \varphi / \tau = C$ a.e. where $\varphi$ is the Kantorovich potential and some constant $C$ (precisely, on all $supp \, p_{k + 1}^\tau$, which is assume $>0$ a.e.). Differentiating, we have 
    \[
    \nabla \varphi = \iota - T_{k + 1}^\tau = -\tau \nabla\gL \quad \text{a.e.}
    \]
\end{proof}

\paragraph{Remark.} The reader may notice that the equation above for $\nabla \varphi$ resembles the JKOnet${}^*$ objective presented in Sec. \ref{sec:approx-ce-in-practice}. \citet{terpin2024learning} uses the first-order necessary conditions for optimality, which bypasses checking that the learned gradient function defines a transport map and the requirement that the scalar function be convex. The original JKOnet \citep{bunne2022proximal}, which requires a bi-level optimization objective, is more explicit in following the transport map, negatively affecting efficiency and stability.

Before proceeding further, we provide a simple bound on the 2-Wasserstein distance of consecutive iterates: by optimality
\[
J(p_{k + 1}^\tau) + \frac{W_2^2(p_{k + 1}^\tau, p_k^\tau)}{2\tau} \leq J(p_k^\tau),
\]
therefore 
\[
\sum_k \frac{W_2^2(p_{k + 1}^\tau, p_k^\tau)}{2\tau} \leq \sum_k 2(J(p_k^\tau) - J(p_{k + 1}^\tau)) \leq 2J(p_0) =: C.
\]
where we telescoped the sum and used $\inf J \geq 0$ as $\gL \geq 0$.

For us to take a limit, we need to fine $p_t^\tau$ for values in $(k\tau, (k + 1)\tau)$. Following \citet[Ch. 8]{santambrogio2015optimal}, take 
\[
p_t^\tau = \left(\frac{k\tau - t}{\tau} v_k^\tau + \iota \right)_\# p_k^\tau \quad \text{ for $t \in ((k-1)\tau, k\tau)$}. 
\]
Moreover, $v_t^\tau$ ought to be defined so that it advects $p^\tau$ over time, and following the intuition of interpolating between discrete values, we require \[
||v_t^\tau||_{L^2(p_t^\tau)} = |(p^\tau)'|(t) = \lim_{t' \to 0} \frac{W_2(p^\tau_{t + t'}, p_t^\tau)}{|t'|} = \frac{W_2(p_{k - 1}^\tau, p_k^\tau)}{\tau}.
\]
In fact, \[
v_t^\tau = \frac{\iota - T_k^\tau}{\tau} \circ((k \tau - t) v_k^\tau + \iota)^{-1}
\]
works. Define the momentum $E^\tau = p^\tau v^\tau$.

\paragraph{Remark.} Note the similarity of the approximation method to the MMFM intermediate densities in Sec. \ref{sec:approx-ce-in-practice}. For simplicity of our objective, the interpolants $p_t^\tau$ are defined by requiring the conditional $p_t^\tau(x | x_{(k-1)\tau}, x_{k\tau}) = \gN(x; (1-\beta_t) x_{(k-1)\tau} + \beta_tx_{k\tau}, \sigma_t I)$ where $\beta_t := \frac{t-(k-1)\tau}{\tau}$. In contrast, the density approximations here require a form on $||v_t^\tau||_{L^2(p_t^\tau)}$.

Now we bound the maximal $E^\tau$:
\[
|E^\tau|([0, 1] \times \Omega) = \int_0^1\int_\Omega |v_t^\tau| \, dp_t^\tau \, dt = \int_0^1 ||v_t^\tau||_{L^1(p_t^\tau)} \, dt \leq \int_0^1 ||v_t^\tau||_{L^2(p_t^\tau)} \, dt.
\]
Using the Cauchy-Schwarz inequality
\begin{equation}
    \label{eq:momentum-bound}
    \int_0^1 ||v_t^\tau||_{L^2(p_t^\tau)} \, dt \leq \int_0^1 ||v_t^\tau||_{L^2(p_t^\tau)}^2 \, dt  = \sum_k \tau \left(\frac{W_2(p_{k - 1}^\tau, p_k^\tau)}{\tau}\right)^2 \leq C^2. 
\end{equation}
Moreover, we see for $0 < s < t < 1$, 
\[
W_2(p_t^\tau, p_s^\tau) \leq \int_s^t |(p^\tau)'|(r) \, dr \leq (t-s)^{1/2} \left(\int_s^t |(p^\tau)'|(r)^2 \, dr\right)^{1/2}
\]
But since $|(p^\tau)'|(r)^2 = ||v_r^\tau||_{L^2(p_t^\tau)}^2$, we have 
\begin{equation}
    \label{eq:w2-holder}
    W_2(p_t^\tau, p_s^\tau) \leq C^2(t-s)^{1/2}.
\end{equation}
This provides a uniform H\"older bound on $t \mapsto p_t^\tau$, which implies uniform boundedness and equicontinuity. Therefore, up to subsequences, both $E^\tau$ and $p^\tau$ converges to a limit as $\tau \to 0$: by Arzela-Ascoli, we have $p^\tau \to p$ uniformly on $W_2$. Moreover, by boundedness of $|E^\tau|$, we have weak compactness on the space of measures, and therefore, there exists a subsequence s.t. $E^\tau \to E$ weakly.

Now, checking the distributional test, e.g. \citet[Definition 4.1]{santambrogio2015optimal}, we have that $\partial_t p_t^\tau + \nabla \cdot E^\tau = 0$. Moreover, the distributional test passes to the limit as $v_t^\tau$ is continuous, $\int\int ||v_t^\tau|^2 \, dp_t^\tau dt \leq C^2$, and $p^\tau \to p$ uniformly, and $E^\tau \to E$ weakly. In other words, we have $\partial_tp_t + \nabla\cdot E = 0$. We are left to compute $E$.

\textbf{Claim:} $E = -p\nabla \gL$.

The trick here is to instead consider simpler curves than the interpolation we defined above. In particular, consider $\tilde p_t^\tau = p_{k + 1}^\tau$ for $t \in (k\tau, (k +1)\tau)$ and $\tilde v_t^\tau = v_{k + 1}^\tau$. Likewise, define $\tilde E^\tau = \tilde p^\tau \tilde v^\tau$. Since \[
\frac{W_2(p_{k + 1}^\tau , p_k^\tau)}{\tau} = \frac{1}{\tau}\left(\int |\iota - T_{k + 1}^\tau|^2 \, dp_{k+1}^\tau \right)^{1/2} = ||\tilde v_{k + 1}^\tau||_{L^2(p_{k + 1}^\tau)},
\]
we have $||\tilde v_{k + 1}^\tau||_{L^2(p_t^\tau)} = ||v_{k + 1}^\tau||_{L^2(p_t^\tau)}$. This implies that $\tilde E^\tau$ has the same bound as in Eqn. \ref{eq:momentum-bound}. Moreover, by using Eqn. \ref{eq:w2-holder}, we see $W_2(p_t^\tau, \tilde p_t^\tau) \leq C^2 \sqrt \tau$. Hence, we have the same convergence: $\tilde p^\tau \to p$ uniformly. One needs to be more careful with the momentum. Let $\tilde E^\tau \to \tilde E$ weakly, we look to prove $\int f \cdot d\tilde E = \int f \cdot d E$ for all $f:[0, 1] \times \Omega \to \R^p$ Lipschitz. We compute:
\[
\left| \int f \cdot d\tilde E^\tau - \int f \cdot dE^\tau \right| \leq \int_0^1 \int_\Omega |f \circ((k\tau-t)v^\tau_{\kappa(t)} + \iota) - f||v_{\kappa(t)}^\tau| \, dp^\tau \, dt \leq Lip(f)\tau \int_0^1 \int_\Omega |\tilde v_t^\tau|^2 \, dp^\tau \, dt,
\]
where $\kappa(t)$ returns the smallest multiple $k\tau \geq t$, and the RHS is bounded by $Lip(f) C^2 \tau$. Thus, as $\tau \to 0$, we see that $E = \tilde E$. 

But now, $\tilde E^\tau = \tilde v^\tau \tilde p^\tau = -\tilde p^\tau \nabla \gL$ by Lemma \ref{lem:velocity-iterates}. Therefore, for any $f \in C_c^1((0, 1) \times \Omega, \R^p)$, we test:
\[
\int f \cdot d\tilde E^\tau = - \int f \cdot \nabla \gL p^\tau.
\]
As $\gL \in C^1$, we know the integrand is continuous, so we can just pass the limit as $\tau \to 0$, meaning that $\tilde E = E = -p\nabla \gL$ invoking the fact that $\tilde p^\tau \to p$ uniformly. 
\end{proof}

\subsection{Proofs of Sec. \ref{sec:approx-ce-in-practice}}

\mybox{
\textbf{Theorem \ref{thm:mmfm-action-gap}}:

Suppose the true marginals evolve according to $\frac{d}{dt}p_t^* = -\nabla \cdot(p_t^* \nabla s_t^*)$ and $t \mapsto p_t^*$ is an absolutely continuous curve. Define $q(z)$ such that marginalizing $q$ with respect to all variables except $x_k, x_{k + 1}$ yields the coupling $p_{t_k} \otimes (T_k^{k +1})_\# p_{t_k}$, where $T_k^{k + 1}$ is the transport map from $p_{t_k}$ to $p_{t_{k + 1}}$. Then,
    \begin{equation*}
        \lim_{|t_k - t_{k + 1}| \to 0} \int_0^1 \E_{p_t(x)} ||u_t(x) - \nabla s_t^*(x)||_2^2 \, dt = 0.
    \end{equation*}

    Replacing $u_t$ with $\frac{x_{t+1}-x_t}{\tau}$, this shows that $\nabla V$ (Eqn. \ref{eq:jkonet-pot-loss}) regresses to the reference action in the limit.
}

\begin{proof}[Proof of Theorem \ref{thm:mmfm-action-gap}]
By absolute continuity of the curve and Brenier's theorem, the Monge map between $p_s, p_t$ exists for $0 \leq s < t \leq 1$.
By \citet[Prop. 8.4.6]{ambrosio2006gradient}, we have that
\begin{equation}
    \label{eq:action-limit}
    \nabla s_t^* = \lim_{h \to 0} \frac{1}{h}(T^*(p_t^*, p_{t + h}^*) - \text{id}),
\end{equation}
where $T^*(p_t^*, p_{t + h}^*)$ is the unique transport map between densities $p_t^*$ and $p_{t + h}^*$. Fixing a small $h > 0$, define \[
d_t(x) := \sum_{m = 0}^{M-1} \mathbf{1}_{[hm, h(m + 1))}(t) \frac{T^*_m(x) - x}{h}, \quad \text{ where } T_m^* = T^*(p_{hm}^*, p_{h(m + 1)}^*).
\]
The idea of our next steps is to instead consider the normed difference between $u_t$ and $d_t$, and later conclude by the triangle inequality.

Recall that
\begin{equation}
    u_t(x) = \E_{q(z)} \frac{u_t(x|z) p_t(x|z)}{p_t(x)} = \sum_{ k =0}^{K-1} \mathbf{1}_{[t_k, t_{k + 1})}(t) \int u_t(x | x_{k}, T_k^{ k + 1}(x_k)) \frac{p_t(x|x_k, T_k^{ k + 1}(x_k))}{p_t(x)} \, dp_{t_k}(x_k),
\end{equation}
where we used the assumed disintegration of $q(z)$. We compute
\begin{align}
    \E_{p_t(x)} ||u_t(x) - d_t(x)||_2^2 &= \int ||u_t(x) - d_t(x)||_2^2  \, dp_t(x)\\
    &=\int \Bigg\vert\Bigg\vert \int \left[ \frac{T_k^{k +1}(x_k) - x_k}{t_{k+1}-t_k} - \frac{T_m^*(x)-x}{h}
    \right] \frac{p_t(x|x_k, T_k^{ k + 1}(x_k))}{p_t(x)} \, dp_{t_k}(x_k) \Bigg\vert\Bigg\vert_2^2 \, dp_t(x)\\
    \label{eq:ut-dt-bound}
    &\leq \int \int \Bigg\vert\Bigg\vert  \frac{T_k^{k +1}(x_k) - x_k}{t_{k+1}-t_k} - \frac{T_m^*(x)-x}{h}
   \Bigg\vert\Bigg\vert_2^2 \frac{p_t(x|x_k, T_k^{ k + 1}(x_k))}{p_t(x)} \, dp_{t_k}(x_k)  \, dp_t(x).
\end{align}
assuming that $t \in [t_k, t_{k + 1})$ and using the fact
\[
d_t(x) = \int d_t(x) \frac{p_t(x|x_k, T_k^{ k + 1}(x_k))}{p_t(x)} \, dp_t(x_k).
\]

\textbf{Claim:} The reference $p_t$ is an absolutely continuous curve in $W_2$-space.

\begin{proof}
    Let $0 \leq s < t \leq 1$ be given. It is known that affine Gaussian paths are absolutely continuous, in particular this means that there exists $g_k \in L^1([0, 1])$ such that $W_2(p_{t_k}, p_{t_{k+1}}) \leq \int_{t_k}^{t_{k + 1}} g_k(\tau) \, d\tau$. If $[s, t] \subset [t_k, t_{k + 1}]$, then we conclude. Otherwise, consider $W_2(p_s, p_{t}) \leq W_2(p_s, p_{t_k}) + W_2(p_{t_k}, p_t)$. 
\end{proof}

The continuity claim allows us to take the limit: $W_2^2(p_{t_k}, p_t) \to 0$ as $|t_k - t| \to 0$. Moreover, by \citet[Theorem 5.10]{santambrogio2015optimal}, we have $p_{t_k} \rightharpoonup p_t$, i.e. narrow convergence.

\textbf{Claim:} Suppose $p_{t_k} \rightharpoonup p_{hm}^*$ and $p_{t_{k+1}} \rightharpoonup p_{h(m+1)}^*$. We have $||T^*(p_{t_k}, p_{t_{k + 1}}) - T^*(p_{hm}, p_{h(m+1)})||_{L^2(p_t)} \to 0$.

\begin{proof}
    Let $t < \min(t_k, h(m+1))$. We first prove this claim for $T^*(p_t, p_{t_{k + 1}}) \to T^*(p_t, p_{h(m+1)})$. \citet[Cor. 5.23]{villani2008optimal} gives convergence in measure (in this case, $p_t$ is the measure). Thus, the argument for $L^2(p_t)$ convergence follows from a typical analysis argument. To simplify notation, let $T \equiv T_t^{h(m+1)}$ and $(T_n)_{n\geq 1}$ the approaching sequence.

    Suppose not, i.e. there exists a subsequence $(T_{n_i})$ such that $||T_{n_i} - T||_{L^2(p_t)} \geq \eps$ for some $\eps > 0$. By convergence in $p_t$-measure, there exists a further subsequence $(T_{n_{i_j}})$ that converges to $T$ pointwise $p_t$-a.e. Since $T_n, T \in L^2(p_t)$, by the Dominated Convergence Theorem $||T_{n_{i_j}} - T||_{L^2(p_t)} \to 0$, and hence $\int |T_{n_{i_j}}(x)| \, dp_t(x) \to \int |T(x)| \, dp_t(x)$. At this point, we have proven that for every subsequence (taken implicitly) of integral $\int |T_n(x)| \, dp_t(x)$, there exists a further subsequence that converges to the limit. It is a well-known result in analysis that this shows convergence of the original sequence of integrals, i.e. $||T_n - T||_{L^2(p_t)} \to 0$.

    We now apply the previous result three times to the following and conclude:
    \begin{align*}
    ||T^*(p_{t_k}, p_{t_{k + 1}}) - T^*(p_{hm}, p_{h(m+1)})||_{L^2(p_t)} &\leq ||T^*(p_{t_k}, p_{t_{k + 1}}) - T^*(p_{t_k}, p_t)||_{L^2(p_t)} \\
    &\hspace{2em} + ||T^*(p_{t_k}, p_t) - T^*(p_{hm}, p_t)||_{L^2(p_t)} \\
    &\hspace{2em} + ||T^*(p_{hm}, p_t) - T^*(p_{hm}, p_{h(m+1)})||_{L^2(p_t)}
    \end{align*}
\end{proof}

To finish, we remind the reader that the densities $p_{t_k} = p_{t_k}^*$, therefore, the supposition of this claim is fulfilled by the observation of narrow convergence: $p_{t_k}^* \rightharpoonup p_{hm}^*$ and $p_{t_{k+1}}^* \rightharpoonup p_{h(m+1)}^*$. Therefore, combining the claim with Eqn. \ref{eq:ut-dt-bound}, we can make $\E_{p_t(x)} ||u_t(x) - d_t(x)||_2^2 < \delta / 2$ for some small $\delta > 0$. Then, if we had chosen a small enough $h$, we would have $\E_{p_t(x)} ||\nabla s_t^*(x) - d_t(x)||_2^2 < \delta/2$. Combining these bounds, we conclude.

\end{proof}

\subsection{Learned proxy matching}
\label{app:extended-proxy-learning}

\paragraph{Overview.} 
In this section, we generalize the regression target in Eqn. \ref{eq:jkonet-pot-loss} and $u_t(x|z)$ in Eqn. \ref{eq:mmfm-regression-target} to encompass methods such as Metric Flow Matching by presenting the notion of \textit{proxy curves}. In particular, we define a family of curves that minimize an objective (such as a data-dependent metric or a Lagrangian) and discuss its fitness as an interpolant (cf. $\mu_t$ in Eqn. \ref{eq:mmfm-mean}) w.r.t. the action gap. The minimization objective of choice in this section is the Lagrangian $L(x_t, \dot x_t, t) = ||\dot x_t||_2^2 + V_t(x_t, \dot x_t)$. This allows some flexibility in the choice of energy functional $V$, which in practice will be data-dependent such as in Metric Flow Matching. In this setting, we seek to characterize choices of energy $V$ to minimize the $W_2$ distance between the proxy probability path, which evolves by $v_\theta$, and the reference $p_t$ in Eqn. \ref{eq:ce}. We start by writing down a continuity equation for the proxy path (cf. Theorem \ref{thm:ce}). Following the action matching discussion, we define a \textit{proxy action gap} in terms of curves. We then show that $W_2^2(p_t, \hat p_t)$ may be bound like in Prop. \ref{prop:w2-bound-action-gap}. Moreover, we note that the smoothness assumption of $\gamma$ is quite restrictive, but in practice, this may be weakened (App. \ref{app:metric-learning}). On the flip side, if the gradient descent path is smooth enough, then we can define an energy functional $V$ such that its minimizing curve $\gamma$ stays close in gradient to $-\nabla \gL$ (see Theorem \ref{thm:sup-action-bound}).

Given its importance as the reference flow we will match during FM training, we discuss how close of a \textit{proxy curve} (generalization of $\mu_t$ in Eqn. \ref{eq:mmfm-mean}) we can obtain within a family of paths that minimizes an energy functional. This will include methods such as Metric Flow Matching \citep{kapusniakMetricFlowMatching2024}, but also shares similarities with GSBM \citep{liu2024generalized}. To motivate this, note that one could "sample" from $p_t$, by saving neural network parameters over the course of training on different initial samples $\theta_0 \sim p_0$. This approach can be used to build a dataset of intermediate weights $\train = \bigcup_{t \in [0, 1]} \train_t$ of weights saved over the course of training. Following \citet{kapusniakMetricFlowMatching2024}, if we define a data-dependent metric $g: \R^p \to \gS_{++}(p)$, which is a smooth map parameterized by the dataset $\train$, we may compute a smooth energy-minimizing curve $\gamma_{\rvx_0, \rvx_1} := \argmin_{\gamma_0 = \rvx_0, \gamma_1 = \rvx_1} \gE_g(\gamma_t)$ between fixed points $(\rvx_0, \rvx_1) \sim \pi$ that can be shown to stay close to the data \citep[Proposition 1]{kapusniakMetricFlowMatching2024}. Further, we may perform this minimization by training geodesic interpolants $\rvx_{t, \eta} \approx \gamma(t; \eta)$, see App. \ref{app:metric-learning} for details. Following, we develop a framework using energy functionals, which are data-dependent in practice, motivated by the connection between the metric and potential approach \citep[App. C.1]{kapusniakMetricFlowMatching2024}: $||\dot x_t||_{g(x_t)} = ||\dot x_t||_2^2 + V(x_t; x_0, x_1)$, where $V$ is a potential function depending on the boundary conditions. In this setting, we seek to characterize choices of energy $V$ to minimize the $W_2$ distance between the proxy probability path, which evolves by $v_\theta$, and the reference $p_t$ in Eqn. \ref{eq:ce}. We start by writing down a continuity equation for the proxy path (cf. Theorem \ref{thm:ce}).

\mybox[orange!20!white]{
\begin{theorem}[Proxy reference path]
\label{thm:learned-ce}
    Suppose the Lagrangian $L(x_t, \dot x_t, t) = ||\dot x_t||_2^2 + V_t(x_t, \dot x_t)$ is Tonelli and strongly convex in velocity. The Lagrangian optimal transport map $T$ exists between $p_0$ and $p_1$.
    Moreover, there exists a locally Lipschitz, locally bounded vector field $w$ s.t.
    \begin{equation}
        \label{eq:learned-ce}
        \partial_t\hat p_t + \nabla \cdot(\hat p_t w_t) = 0
    \end{equation}
    satisfies $\hat p_t = \mathrm{Law}(\gamma_t)$ where $\gamma$ is a random, smooth Lagrangian-minimizing curve and $(\gamma_0, \gamma_1)$ is an optimal coupling of $p_0, p_1$.
\end{theorem}
}
\begin{proof}[Proof of Theorem \ref{thm:learned-ce}]
    First, let us recall the definition of a \textit{Tonelli} Lagrangian. Following \citep{schachter2017eulerian}, it satisfies:
    \begin{enumerate}
        \item $L$ does not depend on time.
        \item $L$ is $C^2$.
        \item $L$ is strictly convex in velocity.
        \item There exists a constant $c_0$ and a function $\theta : \R^p \to \R$ with superlinear growth, i.e. $\lim_{|v|\to \infty} \theta(v)/|v| = \infty$, with $\theta \geq 0$ s.t. $L(x, v) \geq c_0 + \theta(v)$.
    \end{enumerate}
    Then, as noted in \citet[Ch. 3.3]{schachter2017eulerian}, the Lagrangian optimal transport problem has a solution, specifically a map $T: \R^p \to \R^p$ which pushforwards $p_0$ to $p_1$. For our purposes, we should also note that there exists a unique optimal trajectory $\sigma : [0, 1] \times \R^p \to \R^p$ s.t. 
    \[
    \sigma = \text{arginf}_{\sigma: [0, 1] \times \R^p \to \R^p} \left\{\int_0^1 \int_{\R^p} L(\sigma(t,x), \dot \sigma(t,x)) \, dp_0(x) \, dt : \sigma(1, \cdot)_\# p_0 = p_1 \right\}.
    \]
    Using \citet[Prop. 3.4.4]{schachter2017eulerian}, we have a velocity field $w : [0, 1] \times \R^p \to \R^p$ satisfying $\dot \sigma(t,x) = w_t(\sigma(t,x))$ which is locally Lipschitz and locally bounded. Then, by \citet[Prop. 3.4.3]{schachter2017eulerian}, the path defined by $\hat p_t = (\sigma_t)_\# p_0$ and the velocity field $w_t$ satisfies the continuity equation
    \[
    \partial_t\hat p_t + \nabla \cdot(\hat p_t w_t) = 0
    \]
    in the sense of distributions. All that's left is to show that $\hat p_t = \mathrm{Law}(\gamma_t)$ where $\gamma$ is drawn from the Lagrangian-minimizing curves. However, this follows from \citet[Thm. 7.21]{villani2008optimal} as $(\hat p_t)_{t \in [0, 1]}$ minimizes 
    \[
    \mathbb{A}(p) = \inf_\sigma \int_0^1 \int_{\R^p} L(\sigma(t,x), \dot \sigma(t,x)) \, dp_0(x) \, dt,
    \]
    applying the equivalence between (iii) and (i) of Thm. 7.21.
\end{proof}

The following discussion will focus on quantifying closeness of the reference $p$ and $\hat p$. Our objective is to characterize functionals that would induce good proxy trajectories which remain close to Eqn. \ref{eq:ce}. Hence, we ought to assume some regularity for $V$, in particular, we want the Lagrangian to be Tonelli.
In Prop. \ref{prop:w2-equality} below, this definition of the learned path is used to find an expression for the $W_2$ distance that accounts for the closeness of $w_t$ to the loss gradient $-\nabla \gL$. 

\mybox[orange!20!white]{
\begin{proposition}[name=Adapted from Cor. 5.25 \citet{santambrogio2015optimal}]
\label{prop:w2-equality}
    Suppose that $(p, \hat p)$ resides in a compact domain $\Omega \subset \R^p$ and suppose that $p_t, \hat p_t$ are absolutely continuous w.r.t. Lebesgue measure for every $t$. Further, if we assume $p, \hat p$ are absolutely continuous curves in $W_2(\Omega)$, then 
    \begin{equation}
        \label{eq:w2-bound-integral-form}
        W_2^2(p_t, \hat p_t) = 2 \int_0^1 \int_\Omega (x - T_t(x)) \cdot (\nabla \gL(x) + w_t(T_t(x))) \, dt \, dp_t(x),
    \end{equation}
    where $T_t$ is the optimal transport map between $p_t$ and $\hat p_t$ for the cost $|x-y|^2/2$.
\end{proposition}
}

\begin{proof}[Proof of Prop. \ref{prop:w2-equality}] The assumption aligns with \citet[Cor. 5.25]{santambrogio2015optimal}, which gives:
\[
\frac{d}{dt} W_2^2(p_t, \hat p_t) = 2 \int_\Omega (x - T_t(x)) \cdot (\nabla \gL (x) - w_t(T_t(x))) \, dp_t(x).
\]
Integrating both sides and noting that $p_0 = \hat p_0$ yields the desired result.
    
\end{proof}

Following the action matching discussion, we define a \textit{proxy action gap} for the data-dependent energy that will be used for a more intuitive and practical $W_2$-bound; see App. \ref{app:sb-formulation} for an analogue with entropy-regularization.

\begin{equation}
    \label{eq:action-gap}
    AG(p, \hat p) := \frac{1}{2}\int_0^1 \E_\gamma ||\nabla \gL(\gamma_t) + \dot \gamma_t||^2 \, dt,
\end{equation}
where the expectation is taken over all Lagrangian-minimizing curves s.t. $(\gamma_0, \gamma_1)$ is an optimal coupling under the Lagrangian in Theorem \ref{thm:learned-ce}.

We note that this definition is unconventional as it uses curves obtained from $\hat p$ in Theorem \ref{thm:learned-ce}. However, we believe this formulation to better match the interpretation of finding a proxy trajectory that reflects evolution via GD minimization. Intuitively, our Lagrangian-minimizing curves ought to have derivatives close to the gradient descent direction along its length, and the objective is to vary the potential $V$ to minimize $AG(p, \hat p)$.  Next, we adapt a result from \citet{neklyudov2023actionmatching, albergo2023building} to bound the Wasserstein distance in terms of the action gap (cf. Prop. \ref{prop:w2-bound-action-gap})

\mybox[orange!20!white]{
\begin{proposition}
    \label{prop:w2-bound-am}
    Suppose that $\nabla \gL$ is uniformly Lipschitz in $x$ with Lipschitz constant $K$. Then, 
    \begin{equation}
        \label{eq:w2-bound-am}
        W_2^2(p_t, \hat p_t) \leq e^{(1 + 2K)t} \int_0^t \E_\gamma ||\nabla \gL(\gamma_s) + \dot \gamma_s||^2 \, ds.
    \end{equation}
    where the expectation is taken over all Lagrangian-minimizing curves s.t. $(\gamma_0, \gamma_1)$ is an optimal coupling under the Lagrangian in Theorem \ref{thm:learned-ce}.
\end{proposition}
}

\begin{proof}[Proof of Prop. \ref{prop:w2-bound-am}]
     First, note from \citet[Thm. 7.21]{villani2008optimal} that as $p_t, \hat p_t$ are continuous paths, there exists dynamical optimal couplings of $(p_0, p_1)$, as defined in \citet[Def. 7.20]{villani2008optimal}, which we shall denote $\gamma_*, \gamma$ respectively. In particular, we use that $p_t = \mathrm{Law}(\gamma_*(t))$ and $\hat p_t = \mathrm{Law}(\gamma(t))$, and that both $(\gamma(0), \gamma(1))$ and $(\gamma_*(0), \gamma_*(1))$ are distributed according to $\pi$ the optimal coupling between $(p_0, p_1)$. Understanding this, we may define 
     \begin{equation}
         Q_t := \E_{(\gamma(0), \gamma(1)) \sim \pi} ||\gamma_*(t) - \gamma(t)||^2=\E_{(\gamma_*(0), \gamma_*(1)) \sim \pi} ||\gamma_*(t) - \gamma(t)||^2 \geq W_2^2(p_t, \hat p_t).
     \end{equation}
     Furthermore, as the reference path $(p_t)_{t \in [0, 1]}$ follows the continuity equation Eqn. \ref{eq:ce}, $\dot \gamma(t) = -\nabla \gL(\gamma_*(t))$. Thus, considering the time-derivative, we have
     \begin{align*}
         \frac{\partial Q_t}{\partial t} &= 2 \int \inn{\gamma(t) - \gamma_*(t), -\nabla \gL(\gamma_*(t)) - \dot \gamma(t))} \, d\pi(\gamma_0, \gamma_1)\\
         &= 2\int \inn{\gamma(t) - \gamma_*(t), -\nabla \gL(\gamma_*(t)) + \nabla \gL(\gamma(t)) } \, d\pi(\gamma_0, \gamma_1) \\
         &\hspace{3em} + 2\int \inn{\gamma(t) - \gamma_*(t), -\nabla \gL(\gamma(t)) - \dot \gamma(t)} \, d\pi(\gamma_0, \gamma_1) 
     \end{align*}
     The first term may be bounded by Lipschitzness of $\nabla \gL$:\[
     2\inn{\gamma(t) - \gamma_*(t), -\nabla \gL(\gamma_*(t)) + \nabla \gL(\gamma(t)) } \leq 2K||\gamma(t) - \gamma_*(t)||^2.
     \]
     The second term may be bounded by:
     \begin{align*}
         &||\gamma(t) - \gamma_*(t)||^2 - 2\inn{\gamma(t) - \gamma_*(t), -\nabla \gL(\gamma(t)) - \dot \gamma(t)} + ||\nabla \gL(\gamma(t)) + \dot \gamma(t)||^2 \geq 0,\\
         &2\inn{\gamma(t) - \gamma_*(t), -\nabla \gL(\gamma(t)) - \dot \gamma(t)} \leq ||\gamma(t) - \gamma_*(t)||^2 + ||\nabla \gL(\gamma(t)) + \dot \gamma(t)||^2
     \end{align*}
     In summary,
     \[
     \frac{\partial Q_t}{\partial t} \leq (1 + 2K)Q_t + \int ||\nabla \gL(\gamma(t)) + \dot \gamma(t)||^2 \, d\pi(\gamma_0, \gamma_1) ,
     \]
     then by Gronwall's inequality:
     \[
     Q_t \leq \exp(t(1+2K)) \int_0^t \int ||\nabla \gL(\gamma(t)) + \dot \gamma(t)||^2 \, d\pi(\gamma_0, \gamma_1),
     \]
     using the fact that $Q_0 = 0$, and now we conclude by the fact that $W_2^2(p_t, \hat p_t) \leq Q_t$.
\end{proof}

As the loss is arbitrary, it is not guaranteed that the action gap vanishes. For one, the smoothness assumption on $\gamma$ means that eccentric losses cannot be fit exactly. However, as detailed in App. \ref{app:metric-learning}, we can, in practice, weaken the smoothness assumption to match a learned minimizing curve or the simpler cubic splines (motivated by \citet{rohbeck2025modeling}). We also remark that optimizing this functional $V$ is challenging in practice, requiring learned interpolants as discussed in App. \ref{app:metric-learning}, or modeling a drift potential directly as in JKOnet \citep{bunne2022proximal, terpin2024learning}. To conclude, we make use of a natural smoothness assumption on the gradient descent path to prove a bound on $AG$.

\mybox[orange!20!white]{
\begin{theorem}
\label{thm:sup-action-bound}
    Define $g(\gamma,t) := ||\nabla\gL(\gamma_t) + \dot \gamma_t||^2$. Suppose for each $(x_0, x_1) \sim \pi$ there exists a smooth connecting curve $\tilde \gamma$ s.t. $\sup_{0 \leq t \leq 1} g(\tilde \gamma, t) \leq  \delta$ for some $\delta > 0$ and length $\int_0^1 ||\dot \gamma_t||^2 \leq \Gamma$.  If there exists $\eta > \delta > 0$ s.t.
    \begin{enumerate}
        \item $V(\gamma_t, \dot \gamma_t) \geq 2 \Gamma g(\gamma,t)$ for $g(\gamma, t) \geq \eta$, and
        \item $V(\gamma_t, \dot \gamma_t) \leq \min\{\Gamma/\delta, \Gamma\} g(\gamma,t)$ if $g(\gamma,t) \leq \delta$,
    \end{enumerate}
    then $\sup_{0 \leq t \leq 1} g(\gamma_*, t) \leq 2\eta$, where $\gamma_*$ is the Lagrangian-minimizing curve. 
\end{theorem}
}
\begin{proof}[Proof of Theorem \ref{thm:sup-action-bound}]
    Suppose not, so that there exists $t_0 \in [0, 1]$ s.t. $g(\gamma_*, t_0) = ||\dot \gamma_*(t_0) + \nabla \gL(\gamma_*(t_0))||^2 > 2\eta$. By continuity of $\gamma_*$ and $\nabla\gL$ w.r.t. time, we have $0\leq t' \leq t_0$ s.t. for any $t \in [t', t_0]$, $g(\gamma_*, t) \geq \eta$ and $t'$ is chosen s.t. $g(\gamma_*, t') = \eta$. By our assumption on $\eta$, it's natural that we look at the action under two different cases. First, if $g(\gamma_*, t) \geq \eta$, we have:
    \begin{align*}
        A(\gamma_*) &\geq \int_{t'}^{t_0} ||\dot \gamma_*(t)||_2^2 + V(\gamma_*(t), \dot \gamma_*(t)) \, dt \\
        &\geq \int_{t'}^{t_0} ||\dot \gamma_*(t)||_2^2  + 2\Gamma ||\dot \gamma_*(t) + \nabla \gL(\gamma_*(t))||_2^2 \, dt\\
        &>||\gamma_*(t') - \gamma_*(t_0)||^2 + 2\Gamma \eta \geq 2\Gamma\eta.
    \end{align*}
    Otherwise, we have by assumption that there exists a smooth connecting curve $\gamma$ s.t. $d_\nabla(\gamma) \leq \delta < \eta$, hence
    \begin{align*}
        A(\gamma) &= \int_0^1||\dot \gamma(t)||_2^2 + V(\gamma(t), \dot \gamma(t)) \, dt \\
        &\leq \Gamma + \int_0^1 ||\dot \gamma_*(t) + \nabla \gL(\gamma_*(t))||_2^2 \, dt \\
        &\leq \Gamma + \delta \cdot \Gamma / \delta = 2\Gamma.
    \end{align*}
    By minimality of the Lagrangian-minimizing curve $\gamma_*$, we have a contradiction. Hence, for all times, $g(\gamma_*, t) \leq 2\eta$, implying the result.
\end{proof}

\section{Remark on weight initialization}
\label{app:weight-init}
By adapting a well-known result \citep[Prop. 9.3.2]{ambrosio2006gradient}, we can quantify how the choice of weight initialization affects the distribution of converged training weights.
\mybox[orange!20!white]{
\begin{proposition}
\label{prop:w2-bound}
    Further assume that $\gL$ is $\lambda$-convex. Given two different weight initializations $p_0^{(0)}$ and $p_0^{(1)}$ on $\gP(\Omega)$ that evolves according to Eqn. \ref{eq:ce}, we have $W_2(p_t^{(0)}, p_t^{(1)}) \leq e^{-\lambda t} W_2(p_0^{(0)}, p_0^{(1)})$.
\end{proposition}
}
\begin{proof}[Proof of Prop. \ref{prop:w2-bound}]
    If $\gL$ is $\lambda$-convex, then by \citet[Prop. 9.3.2]{ambrosio2006gradient}, we have that $L^\dagger(p) := \int_\Omega \gL(x) \, dp(x)$ is $\lambda$-geodesically convex. Now, we see that \citet[Thm. 11.1.4]{ambrosio2006gradient} applies, and we get the desired inequality.
\end{proof}
The primary hurdle to applying Prop. \ref{prop:w2-bound} is that $\gL$ is rarely convex in the network parameters. For instance, optimization of a multi-layer perceptron (MLP) is highly non-convex due to non-linear activations and the product between hidden and outer layer weights. Interestingly, given a loss minimization problem on a MLP, we have a corresponding convex optimization problem \citep{pilanci2020neuralnetworksconvexregularizers}, i.e. the loss objective is convex in the network parameters and the two problems have identical optimal values. Therefore, with a modified loss function $\gL'$, if $p_1'$ is distributed over its minimizers, e.g. gradient descent is used to solve the convex minimization problem until convergence, we can apply Prop. \ref{prop:w2-bound} with $\lambda = 0$ and $\gL'(\theta_1) = \gL(\theta_1)$ is minimal for $\theta_1 \sim p_1'$.

\section{Related works}
\label{app:rel-works}

\paragraph{Conditional flow matching.}
The CFM objective, where a conditional vector field is regressed to learn probability paths from a source to target distribution, was first introduced in \citet{lipmanFlowMatchingGenerative2023}. The CFM objective attempts to minimize the expected squared loss of a target conditional vector field (which is conditioned on training data and generates a desired probability path) and an unconditional neural network. The authors showed that optimizing the CFM objective is equivalent to optimizing the unconditional FM objective. Moreover, the further work~\citep{tongImprovingGeneralizingFlowbased2024} highlighted that certain choices of parameters for the probability paths led to the optimal conditional flow being equivalent to the optimal transport path between the initial and target data distributions, thus resulting in shorter inference times.  However, the original formulations of flow matching assumed that the initial distributions were Gaussian. \citet{pooladian2023multisample} extended the theory to arbitrary source distributions using minibatch sampling and proved a bound on the variance of the gradient of the objective. \citet{tongImprovingGeneralizingFlowbased2024} showed that using the 2-Wasserstein optimal transport map as the joint probability distribution of the initial and target data along with straight conditional probability paths results in a marginal vector field that solves the dynamical optimal transport problem between the initial and target distributions. 

\paragraph{Flow matching for trajectory inference.} The flow matching framework \citep{albergo2023building, lipmanFlowMatchingGenerative2023, liu2023flow} gives way to a few methods of controlling the trajectory of the inference path, from the simple multi-marginal approach \citep{rohbeck2025modeling}, to approaches with more sophisticated interpolants \citep{neklyudov2023actionmatching, neklyudov2024wassersteinlagrangian, kapusniakMetricFlowMatching2024, pooladian2024neural, rohbeck2025modeling}. A traditional application of trajectory inference is single cell RNA-sequencing \citep{tong2020trajectorynet, neklyudov2024wassersteinlagrangian, kapusniakMetricFlowMatching2024}, however, a similar problem arises in weight generation. For a broad mathematical overview, see \citep{lavenant2024towardstrajectory}.


\paragraph{Neural network parameter generation.}
Due to the flexibility of neural network as function approximators, it is natural to think that they could be applied to neural network weights. \citet{denilPredictingParametersDeep2014} paved the way for this exploration as their work provided evidence of the redundancy of most network parameterizations, hence showing that paramter generation is a feasible objective.
Later, \citet{ha2017hypernetworks} introduced Hypernetworks which use embeddings of weights of neural network layers to generate new weights and apply their approach to dynamic weight generation of RNNs and LSTMs. A significant portion of our paper’s unconditional parameter generation section builds upon the ideas from \citet{peebles2022learning, wangNeuralNetworkDiffusion2024} and the concurrent work of \citet{soroDiffusionbasedNeuralNetwork2024} where the authors employ a latent diffusion model to generate new parameters for trained image classification networks. Direct auto-encoding methods have also seen success, for example \citet{schuerholt2024sane} and \citet{wang2025recurrentdiffusionlargescaleparameter}. 


\paragraph{Weight generation as meta-learning.}
More broadly, we may categorize weight space generation as meta-learning \citep{huPushingLimitsSimple2022, zhmoginovHyperTransformerModelGeneration2022a, fiftyContextAwareMetaLearning2024}, which aims to learn concepts from a few demonstrations. It is therefore natural that the literature has two evaluation settings: in-distribution tasks and out-of-distribution (OOD) tasks.  A prominent example in literature is for the task of few-shot learning. An early example is \citet{raviOptimizationModelFewShot2017} who designed a meta-learner based on the computations in an LSTM cell. Moreover, we may leverage the advancements in generative modeling for weight generation. \citet{leeSupervisedPretrainingCan2023} used transformers for in-context reinforcement learning, but we also see the works of \citet{zhmoginovHyperTransformerModelGeneration2022a, huPushingLimitsSimple2022,
kirsch2024generalpurposeincontextlearningmetalearning, fiftyContextAwareMetaLearning2024} use transformers and foundation models. More similar to our method is the body of work on using diffusion models for weight generation~\citep{duProtoDiffLearningLearn2023, zhang2024metadiff, wangNeuralNetworkDiffusion2024, soroDiffusionbasedNeuralNetwork2024}. These methods vary in their approach, some leveraging a relationship between the gradient descent algorithm and the denoising step in diffusion models to design their meta-learning algorithm. Others rely on the modeling capabilities of conditioned latent diffusion models to learn the target distribution of weights. Most evaluations conducted were in-distribution tasks---with enough training and capacity, it's clear \textit{meta-models} (i.e. models trained on multiple data distributions) should excel at in-distribution tasks. Hence, there is room to explore adaptation for out-of-distribution tasks. However, generalization to novel tasks often presents a challenge to meta-learning and weight generation frameworks \citep{wangNeuralNetworkDiffusion2024, schuerholt2024sane, soroDiffusionbasedNeuralNetwork2024}, including non-diffusion-based approaches \citet{knyazev2021parameter, knyazev2023canwescale}. 

\section{Stochastic formulation of weight evolution}
\label{app:sb-formulation}
\subsection{Setup}
The present formulation of gradient descent as gradient flow ignores the crucial role of noise within typical neural network optimization schemes. Stochastic differential equations (SDEs) provides a way to model SGD as a continuous-time process while taking into account the role of noise. Following \citet{li2021modelingsgdwithsdes}, we write:
\begin{equation}
\label{eq:sgd-as-sde}
    dX_t = -\nabla \gL(X_t) dt + (\alpha \Sigma(X_t))^{1/2}dW_t
\end{equation}
where $W_t$ is the Wiener process, $\alpha > 0$ is the learning rate, and $\Sigma(X) = \E[(\nabla \gL_\xi(X) - \nabla \gL(X))(\nabla \gL_\xi(X) - \nabla \gL(X))^\top]$; here, $\xi$ is a random variable denoting the random batch of training data in the context of SGD. We may then write the Fokker-Planck-Kolmogorov (FPK) equation
\[
\partial_t p_t - \nabla\cdot(p_t \nabla \gL) = \frac{\alpha}{2} \sum_{ij} \frac{\partial^2}{\partial x_i \partial x_j}([\Sigma^{1/2}\Sigma^{\top/2}]_{ij} p_t),
\]
and the corresponding Schr\"odinger bridge (SB) problem
\begin{equation}
    \label{eq:sb-matching}
    \sP^* := \argmin_{\sP_0 = p_0, \sP_1 = p_1} D_{KL}(\sP|\sQ)
\end{equation}
where $\sQ = \mathrm{Law}(X)$ as governed by Eqn. \ref{eq:sgd-as-sde}. 

\subsection{Solution with known SB matching methods}
We run into the issue of well-posedness if the noise covariance is not known a priori, and also if it depends on the state $X_t$. For now, we are aware of the work by \citet{berlinghieri2025beyond}, which only assumes standard SDE regularity conditions to optimize Eqn. \ref{eq:sb-matching}. In particular, given a density over time $h(\cdot)$ and samples from a random time $t_i$, we may construct a state distribution at snapshots $\hat f_{t_i}(\cdot)$. For our approximation, we instead use an empirical time density $\hat h(\cdot)$ approximated from sampled timesteps, and a candidate SDE model parameterized by $\theta$, $f_{\theta, t}$. Training proceeds by using the maximum mean discrepancy (MMD) to quantify the discrepancy between $\hat h \circ f_{\theta, t}$ and $\hat h \circ \hat f_t$.

To simplify the analysis and align with most SB approximation methods, take $\eps_t = \alpha\E_{X_t} [\Sigma(X_t)]$ to form the approximation
\begin{equation}
    \label{eq:approx-sgd-as-sde}
    dY_t = -\nabla \gL(Y_t) dt + \sqrt \eps_t dW_t,
\end{equation}
and obtain the FPK
\begin{equation}
    \label{eq:approx-fpk}
    \partial_t p_t - \nabla\cdot(p_t \nabla \gL) = \frac{\eps_t}{2} \Delta p_t.
\end{equation}
In this form, multi-marginal methods \citep{lavenant2024towardstrajectory, chen2023deep, shen2025multimarginal} may be employed with intermediate weight samples as reference data.

\paragraph{Variational interpolants.} The recent work by \citet{shen2025multimarginal} sheds some light directly onto the problem of modeling gradient flows. This method allows us to specify a family of possible proxies e.g. those induced by the SDE
\begin{equation}
\label{eq:variational-sde-ref}
dX_t = \nabla \Psi_{\alpha}(X_t) dt + \gamma_\beta dW_t,
\end{equation}
where $\alpha, \beta$ are learnable parameters. This method is explored in the context of multi-marginal Schr\"odinger bridges, but it is trivial to modify it for our purposes:
\begin{enumerate}
    \item For $i = 0, \ldots, K-1$, consider data anchors $\{\theta_{t_i}^j\}_{j \in [N_i]}$ and $\{\theta_{t_{i+1}}^j\}_{j \in [N_{i+1}]}$.
    \begin{enumerate}
        \item Simulate forward to $t_{i+1}$ from $\{\theta_{t_i}^j\}_{j \in [N_i]}$ using Eqn. \ref{eq:variational-sde-ref}.
        \item Simulate backward to $t_i$ from $\{\theta_{t_{i+1}}^j\}_{j \in [N_{i+1}]}$ using Eqn. \ref{eq:variational-sde-ref}.
        \item Use simulated samples to estimate the drift between $t_i$ and $t_{i+1}$.
    \end{enumerate}
    \item Concatenate the estimated drifts and use Stage 2 of Alg. 1 in \citet{shen2025multimarginal} to fit $\alpha$ according to the estimated drifts.
\end{enumerate}
The diffusion parameter $\beta$ can be estimated in an outer loop of the above algorithm, as suggested by \citet{shen2025multimarginal}, allowing us to match $\eps_t$. By following this procedure for flow model training, we can vary our interpolant within a natural family of path distributions, using data anchors to better inform training.

\paragraph{Generalized Schr\"odinger Bridge Matching.}
Due to the close relation with our analysis in Sec. \ref{sec:proxies}, we further discuss GSBM \citep{liu2024generalized} which employs entropic action matching as the inner loop objective. In particular, given the reference process Eqn. \ref{eq:approx-fpk}, invoke \citet[Prop. 3.1]{neklyudov2023actionmatching} to obtain a unique entropic action satisfying the FPK equation Eqn. \ref{eq:approx-fpk}. However, in this case, we have a two level optimization: \textbf{(1)} perform action matching to obtain the drift $u_t^\theta$ for a fixed reference path $(\hat p_t)_{t \in [0, 1]}$, and \textbf{(2)} optimize the marginals $\hat p_t$ given the coupling $p_{0, 1}^\theta$ evolved according to the learned drift $u_t^\theta$ in 
\begin{equation}
    \label{eq:sde-drift}
    dX_t = u_t^\theta(X_t) dt + \sqrt \eps_t dW_t.  
\end{equation}
Most relevant to us is the second stage which involves optimizing the marginal distributions. Given $x_0 \sim p_0$ and $x_1$ from Eqn. \ref{eq:sde-drift}, we also obtain the intermediate states $\{X_{t_k}\}$ for $0< t_1 < \dots < t_K < 1$. To parameterize the marginals $p_t$, we assume a Gaussian path $p_t(X_t | x_0, x_1) = \gN(\mu_t, \sigma_t^2 \mI)$, hence deferring our optimization to $\mu_t$ and $\sigma_t$. \citet{liu2024generalized} uses 1-D splines with the control points $\{X_{t_k}\}$ to obtain $\mu_t$ and a uniform sampling of $\sigma_t$ (with boundary conditions $\sigma_0 = \sigma_1 = 0$). Using this parameterization, we can compute the minimization objective 
\begin{equation}
    \label{eq:gsbm-objective}
    \gJ = \int_0^1 \E_{p_t(X_t | x_0, x_1)} \left[  \frac{1}{2}||u_t^\theta(X_t)||_2^2 + V_t(X_t)  \right] \, dt
\end{equation}
to optimize the control points $X_{t_k}$ and $\sigma_{t_k}$, given some choice of $V_t(\cdot)$.

Returning to the marginals $p_t$, we wish to relate its evolution with Eqn. \ref{eq:approx-sgd-as-sde}. Applying, for example, \citet[Prop. 3]{du2024doob}, we can write down the FPK equation
\begin{equation}
    \partial_t\hat p_t(x) = -\nabla \cdot(\hat p_t(x) u_t(x)) + \frac{\eps_t}{2}\Delta \hat p_t \quad \text{where} \quad u_t(x) = \dot \mu_t + \frac{1}{2}\left(\frac{\dot \sigma_t}{\sigma_t} - \frac{\eps_t}{\sigma_t}\right)(x - \mu_t).
\end{equation}
Therefore, our analogy to Eqn. \ref{eq:action-gap} would be to choose our energy functional $V$ to minimize the gap
\[
\int_0^1 \E_{x \sim \hat p_t} ||\nabla \gL(x) + u_t(x)||^2_2 \, dt.
\]


\section{Interpolating paths}
\label{app:metric-learning}
There are many variants of interpolating paths that have been used as reference in the flow matching literature. Typically, the conditional probability path is of the form 
\begin{equation}
    \label{eq:prob-path}
    p_t(x | z) = \gN(x; \mu_t(z), \sigma_t^2(z)\mI)
\end{equation}
with the consideration that the boundary conditions are satisfied, i.e. $\int p_0(x |z) \, dq(z) = p_0(x)$ and $\int p_1(x |z) \, dq(z) = p_1(x)$. The simplest case by \citet{tongImprovingGeneralizingFlowbased2024} considers a linear interpolant $\mu_t(x_0, x_1) = tx_1 + (1-t)x_0$ with a constant, small variance and $q(x_0, x_1) = p_0(x_0) p_1(x_1)$.

Since we are most concerned with the inducing vector field (see Eqn. \ref{eq:cfm}), we would like a simple vector field that induce the desired conditional probability path. The Gaussian path Eqn. \ref{eq:prob-path} is known \citep{rohbeck2025modeling} to have flow
\[
\phi_t(x |z) = \mu_t(z) +\sigma_t(z)\left(\frac{x - \mu_0(z)}{\sigma_0(z)}\right)
\]
which, in fact, has a unique inducing vector field \citep{lipmanFlowMatchingGenerative2023}
\begin{equation}
\label{eq:vec-field}
u_t(x|z) = \frac{\sigma'_t(z)}{\sigma_t(z)}(x_t - \mu_t(z)) + \mu_t'(z).
\end{equation}
The above suggests that if we have a desired interpolating path $\mu_t$, the vector field to match is known from Eqn. \ref{eq:vec-field}. Following, we discuss a few examples from literature.

\paragraph{Lifted curves.} In spline interpolation, it is best if data points correspond to different timesteps to better capture the trajectory over time. If instead we have a sampling of population over time, it is more natural to consider matching marginal path distributions. \citet{lavenant2024towardstrajectory} (gWOT) provides a framework for exactly this. When we have data from $N_i$ samples at various time points $t_i$, $\{\theta_{t_i}^j\}_{i \in [K]}$, we may form the empirical distribution
\[
    \hat\rho_{t_i} = \frac{1}{N_i} \sum_{j = 1}^{N_i} \delta_{\theta_{t_i}^j}.
\]
To produce smoother interpolants, we also introduce a regularizer and Gaussian convolution with a kernel of width $h$ to obtain $\hat\rho_{t_i}^h$. Specifically, we minimize (by gradient descent) the convex functional
\[
F_{K, \lambda, h} (\rmR):= \sigma^2 D_{KL}(\rmR|\rmW^\sigma) + \frac{1}{\lambda}\sum_{i = 1}^K|t_{i + 1}- t_i| D_{KL}(\hat\rho_{t_i}^h|\rmR_{t_i})
\]
over a law on paths $\rmR \in \mathcal P(\Omega)$.

\paragraph{Data-dependent geodesics.} \citet{kapusniakMetricFlowMatching2024} made use of a data-dependent metric to compute a geodesic. In practice, the interpolant is obtained through training a neural network to minimize
\begin{equation}
    \label{eq:mfm}
    L_{\text{mfm}}(\theta) = \E_{\rt \sim U[0, 1], (\rvx_0, \rvx_1) \sim \pi}||v_\theta(\rvx_{t, \eta^*}, t) - \dot \rvx_{t, \eta^*}||^2_{g(\rvx_{t, \eta^*})}.
\end{equation}
Comparing with Eqn. \ref{eq:vec-field}, note that we are using a small, constant variance, so indeed we only match the interpolant derivative. Moreover, as $g$ is a data-dependent metric, its optimization towards a suitable proxy path reduces to learning a parameterization of $g$ w.r.t. the intermediate weights $\gD$. Here, we note the MFM framework fits without issue into the proxy path framework of Sec. \ref{sec:proxies} primarily because geodesics are presumed smooth in time. Intuitively, due to randomness in the training process, some intermediate weights $\theta \in \gD$ ought to be weighed less than others so that the induced geodesic $\gamma$ better matches the true loss gradient. 

\paragraph{Cubic splines.} Following the recent work by \citet{rohbeck2025modeling}, we can fit a cubic spline to conditional data points and use this curve as the reference interpolant. Cubic splines are obtained by optimizing the variational objective
\begin{equation}
    \label{eq:cubic-spline}
    \mu_t(x_0, \ldots, x_k) = \arg\,\min_{\gamma \in \gH^2([t_0, t_K])} \int_{t_0}^{t_K} ||\ddot \gamma(t)||_2^2 \, dt
\end{equation}
where $x_k = \gamma(t_k)$ and $\gH^2([t_0, t_K])$ denotes the class of functions that has absolutely continuous first derivative and weak second derivative on the interval $[t_0, t_K]$. The conditioning data may be sampled, say from our dataset $\gD$. Moreover, \citet{rohbeck2025modeling} considered class-conditional trajectories. In our setting, this could mean weight trajectories from different training datasets. Thus, we sample intermediate weights $(x_0^c, \ldots, x_K^c)$, where $c$ denotes the training set, and we may use techniques such as classifier-free guidance on $v_\theta$ to improve training. Note that as the interpolant is entirely contingent on the intermediate distributions, and specifically the intermediate samples, interpolant optimization cannot be done within this approach. Instead, we rely on a faithful sampling $\gD$ from the reference Eqn. \ref{eq:ce} to provide a good proxy path.

\paragraph{Piecewise linear interpolants.} We conclude by discussing the most straightforward method of incorporating conditional data. In particular, if we have points $z = (x_0, \ldots, x_K)$ at times $(t_0, \ldots, t_K)$, we have the conditional vector field 
\begin{equation}
    \label{eq:lin-interp-vfield}
    u_t(x|z) = \sum_{k = 0}^{K-1} \frac{x_{k + 1} - x_k}{t_{k + 1} - t_k} \mathbf{1}_{[t_k, t_{k + 1})}(t).
\end{equation}
Within our conceptual framework, we may write 
\[
\frac{1}{2}\int_0^1 \E_\gamma ||\nabla \gL(\gamma_t) + \dot \gamma_t||^2 \, dt = \frac{1}{2} \int_0^1 \sum_{k = 0}^{K-1}\mathbf{1}_{[t_k, t_{k + 1})}(t) \E_{(x_k, x_{k + 1}) \sim p_{t_k} \otimes p_{t_{k + 1}}} ||\nabla \gL(\gamma_t) +\frac{x_{k +1} - x_k}{t_{k + 1} - t_k}||^2
\]
as the action gap. As the interpolant is entirely contingent on the intermediate distributions, and specifically the intermediate samples, interpolant optimization cannot be done within this approach. However, this approach has desirable limiting properties w.r.t. the sampling of $\gD$. In particular, recalling Eqn. \ref{eq:ce}, if we let $\sigma:[0, 1] \times \R^p \to \R^p$ be the flow of the drift in Eqn. \ref{eq:ce}, the expression 
\[
\E_{(x_k, x_{k + 1}) \sim p_{t_k} \otimes (\sigma_{t_{k + 1} - t_k})_\#p_{tk}} ||\nabla \gL(\gamma_t) +\frac{x_{k +1} - x_k}{t_{k + 1} - t_k}||^2
\]
goes to zero as $t_{k + 1} - t_k \to 0$. In other words, if the intermediate samples are drawn from the same training trajectory, and are sampled with sufficient time-density, the gap can be made arbitrarily small.

\section{Further method details}
\label{app:arch}

Here, we expound on the implementation of our approach. See Figure~\ref{fig:f2sl} for a schematic of the training and inference process.

\subsection{Pre-trained model acquisition}
\label{app:pre-training-acqui}

\paragraph{Datasets and architectures.} We conduct experiments on a wide range of datasets, including CIFAR-10/100~\citep{krizhevsky2009cifar}, STL-10~\citep{coates2011stl10}, Fashion-MNIST~\citep{xiao2017fmnist}, CIFAR10.1 \citep{recht2019imagenet}, and Camelyon17 \citep{veeling2018rotation}. To evaluate our meta-model's ability to generate new subsets of network parameters, we conduct experiments on ResNet-18~\citep{He2015DeepRL}, ViT-Base~\citep{dosovitskiy2021animage}, ConvNeXt-Tiny~\citep{liu2022convnet2020s}, the latter two are sourced from timm \citet{rw2019timm}. As we shall detail below, small CNN architectures from a model zoo~\citep{schurholtModelZoosDataset2022} are also used for full-model generations.

\paragraph{Model pre-training and checkpointing.} For better control over the target distribution $p_1$, in experiments involving ResNet-18, ViT-Base, and ConvNeXt-Tiny, we pre-train these base models from scratch on their respective datasets. We follow \citet{wangNeuralNetworkDiffusion2024} and train the base models until their accuracy stabilizes (in practice, we train all base models for 100 epochs). Depending on the experiment, we save checkpoints differently. If only the converged weights are needed, we save 200 weights at every iteration past 100 epochs. Otherwise, if we require intermediate weights, then we specify the number of saving epochs and the number of weights to save in such an epoch. For instance, we may have 100 save epochs and 100 saves per epoch, meaning that we save at \textit{every} training epoch and save weights in the first 100 iterations of each epoch.

\paragraph{Model zoo.} The model zoo used for meta-training in the model retrieval setting, was sourced from \citep{schurholtModelZoosDataset2022}. As the base model, we employed their CNN-medium architecture, which consists of three convolutional layers and contains 10,000-11,000 parameters, depending on the number of input channels. We pre-trained these models as described above.

\subsection{Variational autoencoder}
The variational autoencoder follows the implementation of \citet{soroDiffusionbasedNeuralNetwork2024}, which employs a UNet architecture for auto-encoding. In particular, given a set of model weights $\{\mathcal{M}_i\}_{i = 1}^N$, we first flatten the weights to obtain vectors $\vw_i \in \R^{d_i}$. For the sake of uniformity, we always zero-pad vectors to $d = \max_i d_i$. Alternatively, we allow for layer-wise vectorization: set a chunk size $\ell$ which corresponds to the weight dimension of a network layer. Then, zero-pad $\vw_i$ to be a multiple of $\ell$, say $\tilde d$. This allows us to partition into $k$ equal length vectors $\vw_{i, k} \in \R^{\tilde d/k}$. Typically, larger models benefit from layer-wise vectorization.

Subsequently, we train a VAE to obtain an embedding of such vectors by optimizing the objective:
\begin{equation}
\label{eq:vae}
    L_{\text{VAE}}(\theta, \phi) := -\E_{\vz \sim q(\vz | \vw)} [\log p_\theta(\vw | \vz)  + \beta D_{KL}(q_\phi(\vz | \vw) || p(\vz))]
\end{equation}
where $\vw$ is the vectorized weights, $\vz$ is the embedding we are learning, and $p_\theta, q_\phi$ are the reconstruction and posterior distributions respectively. Moreover, we fix the prior $p(\vz)$ to be a $(0, 1)$-Gaussian and the weight is set to be $\beta = 10^{-5}$. For layer-wise vectorization, we simply change the input dimensions to match the chunk size. Upon decoding, we concatenate the chunks to re-form the weight vector.

As for training, the VAE was trained with the objective in \eqref{eq:vae}. Moreover, following p-diff \citep{wangNeuralNetworkDiffusion2024}, we add Gaussian noise to the input and latent vector, i.e. given noise factors $\sigma_{in}$ and $\sigma_{lat}$ with encoder $f_\phi$ and decoder $f_\theta$, we have \[
\vz = f_\phi(\vw + \xi_{in}), \; \hat \vw = f_\theta(\vz + \xi_{lat}) \quad \text{where} \quad  \xi_{in} \sim \mathcal{N}(0, \sigma_{in}^2\mI), \; \xi_{lat} \sim \mathcal{N}(0, \sigma_{lat}^2\mI).
\]

A new VAE is trained at every instantiation of the \fsl-CFM model as architectures often differ in their input dimension for different experiments. However, they are trained with different objectives: the VAE is trained to minimize reconstruction loss. In all experiments, we fix $\sigma_{in} = 0.001$ and $\sigma_{lat} = 0.01$.

\begin{algorithm}[ht]
\caption{Sampling Trajectories from Weight Tensor}
\label{alg:sampling}
\begin{algorithmic}[1]
\Require Number of save epochs $N_{epochs}$, savepoints per epoch $S$, classifier size $D$, tensor of classifier weights $X \in \R^{N_{epochs} \times S \times D}$, number of time samples $K \leq N_{epochs}$.
\Ensure Sampled tensor $W \in \R^{K \times S \times D}$
\State Flatten $X$ to shape $[N_{epochs} \cdot S, D]$ \Comment{Assumes first dim. sorted by training iteration.}
\State Sample $K \cdot S$ indices $I$ uniformly from $[0, N_{epochs} S)$
\State Extract $X[I]$ and reshape to $W \in \R^{K \times S \times D}$
\If{add\_noise}
    \State $W \gets W + \eps$ \Comment{$\eps \sim \gN(0, 10^{-3})$}
\EndIf
\State \Return $W$
\end{algorithmic}
\end{algorithm}

\subsection{Generative meta-model}

\paragraph{Multi-marginal flow matching.} Multi-marginal flow matching (MMFM) proceeds in the same regression paradigm as flow matching models, with the difference being the regression target is now Eqn.~\ref{eq:lin-interp-vfield} (piece-wise linear interpolation). In practice, $z = (x_0, \ldots, x_K)$, where we have $K=3, 4$, or 5, as evaluated in Table \ref{tab:uncond}, and the elements of $z$ are sampled from $p_{t_0} \otimes \dots \otimes p_{t_K}$ constituted by samples obtained from base model pre-training App. \ref{app:pre-training-acqui}. Typically, we save a lot more checkpoints than needed, and so we need to subsample them by Algorithm~\ref{alg:sampling} to create the training dataset. For better training, we may also choose to sample weight initializations, i.e. $x_0 \in z$, for each batch. This can be done easily by using \verb|torch.nn.init| to reset parameters of a module to the desired initialization, e.g. Kaiming uniform. In fact, sampling the weight initialization is all that is required for the validation step.

\paragraph{JKOnet$^{*}$.} The dataloading aspect of JKOnet training is exactly the same as MMFM, so we focus our attention to the training regiment. Following \citet{terpin2024learning}, we may formally write our loss as 
\begin{equation}
    \label{eq:jkonet-loss}
    \sum_{k = 0}^{K-1} \int_{\R^D \times \R^D} ||\nabla V_\theta(x_t, t) + (x_{t + 1}- x_t) / \Delta t ||^2 \, d \gamma(x_t, x_{t+1}).
\end{equation}
Recalling the gradient descent formula, the time argument is not necessary for this loss, but we found it to improve empirical performance. Moreover, the $\gamma(\cdot, \cdot\cdot)$ distribution is traditionally an optimal coupling of consecutive marginals $p_{t_k}, p_{t_{k + 1}}$, however, we have a more natural choice: $x_t$ and $x_{t+1}$ ought to be checkpoints saved consecutively during pre-training. Indeed, we found this to be the superior choice in terms of performance. Moreover, an important note is in order: since JKOnet${}^*$ expects $x_t$ to evolve by a gradient flow, this imposes a condition on how the checkpoints are obtained. We found that meta-training \textbf{only works if pre-training is done using the SGD optimizer}. Modern optimizers such as Adam notably fails for JKOnet$^{*}$. We also connect this point to the footnote in Table \ref{tab:model-retrieval}: just because the weights evolve by a gradient flow does \textit{not} mean that the encoded weights behave the same way. Indeed, model retrieval in latent space fails when we use JKOnet$^{*}$. This suggests an avenue for further research: construct an autoencoder that preserves gradient flow in latent space, where preferably the latent gradient flow can be deduced/estimated from the original gradient steps.

\paragraph{Architecture.}
The neural network used for flow matching is the UNet from D2NWG \citep{soroDiffusionbasedNeuralNetwork2024}. The specific hyperparameters used for the CFM model varies between experiments, so we leave this discussion to the next section. For experiments such as model retrieval where we require a conditioning vector, this is implemented by concatenating a context feature vector (e.g. images are passed into a CLIP encoder \citep{Radford2021LearningTV}, and optionally an attention module if we have a set of context images) to the last dimension (same axis as the neural network parameters); of course, we ignore these extra features after the forward pass.

\begin{table}[t]
\caption{Model architectures and hyperparameters. The JKOnet model uses the same UNet as \fsl-CFM and \fsl-MMFM, but with the up-sampling section replaced by pooling and linear layers.}
\label{tab:model_architectures}
\tablestyle{1.5pt}{1.0}
\centering
\begin{tabular}{lc}
\toprule

\multicolumn{2}{c}{\textbf{Weight Encoder} (UNet)} \\
\midrule
Architecture & VAE \\
Latent Space Size & $4 \times 4 \times 4$ \\
Upsampling/Downsampling Layers & 5 \\
Channel Multiplication (per Downsampling Layer) & (1, 1, 2, 2, 2)\\
ResNet Blocks (per Layer) & 2 \\
KL-Divergence Weight & 1e-5 \\
\midrule
\multicolumn{2}{c}{\textbf{\fsl-CFM and \fsl-MMFM Model} (UNet)} \\
\midrule
Input Size w/ VAE & $4 \times 4 \times 4$ \\
Input Size w/o VAE & variable \\
Upsampling/Downsampling Layers & 4 \\
Channel Multiplication (per Downsampling Layer) & (2, 2, 2, 2) \\
ResNet Blocks (per Layer) & 2 \\
\bottomrule
\end{tabular}

\vskip -0.10in
\end{table}

\subsection{Reward fine-tuning}
\label{app:reft}

\paragraph{Overview.} Continuing from the exposition in Section \ref{sec:methods}, we first provide further motivation for our choice of AM. We also note that our noise injection is different from the noise schedule assumed in the original derivations \citep{domingo-enrichAdjointMatchingFinetuning2024}. As such, we modified the computations slightly. However, since the conventional flow matching algorithm injects a very small noise in the Gaussian path, i.e. $p_t(x_t | x_0, x_1) = \gN(\mu_t, \sigma_t^2 I)$, where practically $\sigma_t = 10^{-3}$, we found that divisions by $\sigma_t$ in the adjoint matching algorithm will explode quantities in the loss. To resolve this issue, we derived a deterministic version of the adjoint matching algorithm, which we found to work much  better with more sensible norms. Further, we derived a multi-marginal variant of the adjoint matching algorithm, however, we focus our resources only on the CFM case. All computations are given below.

\paragraph{Motivation.} Unlike naive fine-tuning strategies that require backpropagation through the inference-time solver, AM reduces the problem to a regression objective, closely resembling standard CFM training (see Algorithm \ref{alg:det-am}). Importantly, it simply renews the drift by learning an additive correction: $v^{ft}_t(x) := v^{base}_t(x) + u(t,x)$ (App. \ref{app:reft}), so the analysis in Sec. \ref{sec:proxies} continues to apply without modification.

\paragraph{Adjoint matching with $\sigma_t = 10^{-3}$.} We start in \citet[App. C.5]{domingo-enrichAdjointMatchingFinetuning2024}. As we are using \fsl-CFM, let $Y_0 \sim p_0$ and $Y_1 \sim p_1$, then we may write
\begin{equation}
    \label{eq:am-cfm-velocity}
    \begin{split}
        v(x,t) &= \E[\alpha_t Y_1 + \beta_t Y_0 \mid x=\alpha_tY_1 + \beta_t Y_0]\\
        &=\E[\frac{\dot \alpha_t(x-\beta_tY_0)}{\alpha_t} + \dot \beta_t Y_0 \mid x=\alpha_tY_1 + \beta_t Y_0]\\
        &=\frac{\dot\alpha_t}{\alpha_t}x + (\dot \beta_t - \frac{\dot\alpha_t}{\alpha_t}\beta_t) \E[Y_0 \mid x=\alpha_tY_1 + \beta_t Y_0],
    \end{split}
\end{equation}
using the fact that $Y_1 = (x-\beta_t Y_0)/\alpha_t$. On a practical note, we typically have $\alpha_t = t, \, \beta_t = 1-t$. A central piece of the theory consists of relating this vector field $v$ with the score function $s(x,t)$. We note that once this relationship is established with $\sigma_t$, then the rest of the adjoint matching derivation follows. The crucial step is in writing the score function; following from \citet[Eqn. 92]{domingo-enrichAdjointMatchingFinetuning2024},
\begin{equation}
    \label{eq:am-score}
    s(x,t) = \frac{\E[p_{t|1}(x|Y_1) \nabla \log p_{t|1} (x|Y_1)]}{p_t(x)}, \quad \text{where} \quad p_{t|1}(x|Y_1) = \frac{\exp(-||x-\alpha_t Y_1||^2/2\beta_t^2)}{(2\pi\beta_t^2)^{D/2}}.
\end{equation}
It suffices to change $p_{t|1}$:
\begin{equation}
    \label{eq:am-new-score}
    p_{t|1}(x|Y_1) = \frac{\exp(-||x-\alpha_t Y_1||^2/(2\sigma_t)^2)}{(2\pi\sigma_t^2)^{D/2}} \implies \nabla \log p_{t|1}(x|Y_1) = -\frac{x-\alpha_tY_1}{\sigma_t^2}
\end{equation}
and combine with Eqn. \ref{eq:am-score} to obtain $s(x,t) = -\beta_t \E[Y_0 \mid x = \beta_t Y_0 + \alpha_t Y_1] / \sigma_t^2$. Further combining with Eqn. \ref{eq:am-cfm-velocity}, we have the correspondence
\begin{equation}
    \label{eq:am-vel-score-equiv}
    v(x,t) = \frac{\dot \alpha_t}{\alpha_t}x + \left( \frac{\dot \alpha_t}{\alpha_t} \beta_t - \dot \beta_t
    \right) \frac{\sigma_t^2}{\beta_t} s(x,t) \iff s(x,t) = \frac{\beta_t\left(v(x,t) - \frac{\dot \alpha_t}{\alpha_t} x \right)}{\sigma_t^2 \left( \frac{\dot\alpha_t}{\alpha_t} \beta_t - \dot\beta_t\right)}.
\end{equation}
Concluding, we may thus write the SDE that mirrors \citet[Eqn. 6]{domingo-enrichAdjointMatchingFinetuning2024}:
\begin{equation}
    \label{eq:am-new-sde}
    \begin{split}
        dX_t &= b(X_t,t) \, dt + \eps(t) \, dB_t\\
        &= \left[\frac{\dot \alpha_t}{\alpha_t}x + s(x,t) \left(\left( \frac{\dot \alpha_t}{\alpha_t}\beta_t - \dot \beta_t\right) \frac{\sigma_t^2}{\beta_t} + \eps(t)^2/2\right)\right] \, dt + \eps(t) \, dB_t.
    \end{split}
\end{equation}

\begin{algorithm}[ht]
\caption{Deterministic adjoint matching for fine-tuning flow models}
\label{alg:det-am}
\begin{algorithmic}[1]
\Require Pre-trained FM velocity field $v^{base}$, step size $h$, number of fine-tuning iterations $N$, trajectory batch size $M$, dataset batch size $m$, initialized $v^{ft} = v^{base}$, cross-entropy loss $\Ls$.
\Ensure Reward fine-tuned FM velocity field $v^{ft}$.
\For{$n \in [0, \ldots, N-1]$}
    \State Sample $M$ trajectories $\mX = (X_t)_{0\leq t \leq 1}$ with an Euler solver with step size $h$ and $X_0 \sim p_0$.
    \State $\{(x_i, y_i)\}_{i = 1}^m \sim \train$ \Comment{Sample from classifier dataset}
    \State For each of the $M$ trajectories, evaluate classifier on predicted weights\\ \hspace{5em} \(\ell(X_1) = \sum_{i = 1}^m \Ls(\textsc{nnet}_{X_1}(x_i), y_i).\) \Comment{$\ell(X_1)$ is a vector of size $M$}
    \State $\tilde a_{t-h} = \tilde a_t + h\tilde a_t ^\top \nabla_{X_t}v^{base}(X_t,t), \quad \tilde a_1 = \nabla_{X_1}\ell(X_1)$ \Comment{Backward solve the lean adjoint ODE}
    \State Detach from computation graphs: $X_t = \verb|stopgrad|(X_t)$ and $\tilde a_t = \verb|stopgrad|(\tilde a_t)$.
    \State $\Ls_{AM}(\theta) = \sum_{t \in \{0, \ldots, 1-h\}} ||v_\theta^{ft}(X_t, t) - (v^{base}(X_t,t) - \tilde a_t)||^2$.
    \State Compute $\nabla_\theta \Ls_{AM}$ and optimize as usual.
\EndFor
\end{algorithmic}
\end{algorithm}

\paragraph{Deterministic in practice.} The key quantity in adjoint matching is the memory-less schedule given as $\eps(t) = \sqrt{2\eta_t}$, where reading from Eqn. \ref{eq:am-new-sde}, $\eta_t = \left( \frac{\dot \alpha_t}{\alpha_t}\beta_t - \dot \beta_t\right)\frac{\sigma_t^2}{\beta_t}$. Plugging in $\alpha_t = t, \, \beta_t = 1-t$, we have 
\[
\eps(t) = \sqrt{\frac{2\sigma_t^2}{1-t}((1-t)/t + 1)} = \sqrt{\frac{2\sigma_t^2}{t(1-t)}} = \sigma_t\sqrt{\frac{2}{t(1-t)}}.
\]
Following the suggestions in \citet[App. H]{domingo-enrichAdjointMatchingFinetuning2024}, we add a small value $h$ to both terms in the denominator. In practice, this means that $\sqrt{\frac{2}{t(1-t)}} \leq 10$, whereas $\sigma_t = 10^{-3}$. Looking at the adjoint matching algorithm, this clearly presents a problem of scale. We found in practice that a deterministic variant, i.e. $\eps(t) = 0$ works well in our evaluations. To formulate this modified algorithm, note that our primary objective is to optimize over the control $u$ in $dX_t = (b(X_t, t) + \eps(t) u(X_t, t)) \, dt + \eps(t) \, dB_t$. The original formulation includes the correction term to the drift, and defines $v^{ft}(x,t)$ as \[
b(x,t) + \eps(t) u(x,t) = 2v^{ft}(x,t) - \frac{\dot \alpha_t}{\alpha_t}x.
\]
We simplify by instead defining $b(x,t) = v^{base}(x,t)$ and $v^{ft}(x,t) = b(x,t) + u(x,t)$. Consequently, this simplifies the adjoint ODE and the regression target, leading to Algorithm \ref{alg:det-am}.

\paragraph{Multi-marginal adjoint matching.} Since this paper touches upon multi-marginal flow matching in its experiments, we think it natural to extend the adjoint matching computations to the multi-marginal setting. In this section, we provide the derivations and leave experimentation to future work. Suppose we are given $z = (Y_0, \ldots, Y_K)$ at times $(t_0, \ldots, t_K) \subset [0, 1]$, we first need to find an expression for a sampled point $x$ analagous to $x = \beta_tY_0 + \alpha_tY_1$. Motivated by Eqn. \ref{eq:lin-interp-vfield}, define
\begin{equation}
    \label{eq:am-mmfm-sample}
    x = \sum_{k = 0}^{K-1}[(1-s_t^{(k)}) Y_k + s_t^{(k)} Y_{k+1}] \mathbf{1}_{[t_k, t_{k + 1})}(t), \quad \text{where } s_t^{(k)} = \frac{t-t_k}{t_{k + 1}-t_k}.
\end{equation}
Differentiating w.r.t. $t$, we indeed get $u_t(x|z)$ in Eqn. \ref{eq:lin-interp-vfield}; moreover, to simplify notatoin, set $r_t^{(k)} = 1-s_t^{(k)}$. The key insight is that only one of the terms in Eqn. \ref{eq:am-mmfm-sample} is non-zero for any $t \in [0, 1]$, therefore 
\begin{equation}
    \label{eq:am-mmfm-vel}
    \begin{split}
        v(x,t) &= \sum_k \mathbf{1}_{[t_k, t_{k + 1})}(t) \E[\dot r_t^{(k)} Y_k + \dot s_t^{(k)}Y_{k+1} \mid x]\\
        &= \sum_k \mathbf{1}_{[t_k, t_{k + 1})}(t) \E\left[\dot r_t^{(k)} Y_k + \dot s_t^{(k)}\frac{x-r_t^{(k)}Y_k}{s_t^{(k)}} \mid x\right]\\
        &= \sum_k \mathbf{1}_{[t_k, t_{k + 1})}(t) \left[ \frac{\dot s_t^{(k)}}{s_t^{(k)}} x + \left(r_t^{(k)} - \frac{\dot s_t^{(k)}}{s_t^{(k)}} r_t^{(k)}\right)
         \E[Y_k \mid x]\right].
    \end{split}
\end{equation}
Like before, the key step is to relate the velocity $v$ and the score function. Note that the backward conditional probabilty is now
\begin{equation}
    \label{eq:am-backward-condprob}
    p_{t|1}(x|z) = \sum_k \mathbf{1}_{[t_k, t_{k + 1})}(t) p_{t|t+1}(x|Y_{k+1}),
\end{equation}
where again we use $\sigma_t$:
\begin{equation}
    \label{eq:am-backward-condprob-summands}
    p_{t|t+1}(x|Y_{k+1}) = \frac{\exp(-||x-s_t^{(k)}Y_{k+1}||^2/2\sigma_t^2)}{(2\pi \sigma_t^2)^{D/2}} \implies \nabla \log p_{t|t+1}(x|Y_{k+1}) = - \frac{x- s_t^{(k)}Y_{k+1}}{\sigma_t^2}.
\end{equation}
Since $\E, \nabla$ are linear, we may move around the sum as we please. Therefore, analagous to the CFM case:
\begin{equation}
    s(x,t) = -\sum_k \mathbf{1}_{[t_k, t_{k + 1})}(t) \E\left[
    \frac{x-s_t^{(k)}Y_{k + 1}}{\sigma_t^2} \Bigg\vert \; x
    \right] = -\sum_k \mathbf{1}_{[t_k, t_{k + 1})}(t) \frac{r_t^{(k)}}{\sigma_t^2} \E[Y_k \mid x].
\end{equation}
At this point, it is a matter of using the same trick as in Eqn. \ref{eq:am-mmfm-vel}: there must be at most one $k \in [0, \ldots, K-1]$ such that $\E[Y_k \mid x] = -\sigma_t^2 s(x,t) / r_t^{(k)}$. Therefore, 
\begin{equation}
    \label{eq:mmfm-vel-score-equiv}
    \begin{split}
    &v(x,t) = \sum_k \mathbf{1}_{[t_k, t_{k + 1})}(t)\left[ \frac{\dot s_t^{(k)}}{s_t^{(k)}} x + \frac{\sigma_t^2}{r_t^{(k)}}\left(\frac{\dot s_t^{(k)}}{s_t^{(k)}} r_t^{(k)} - r_t^{(k)}\right)
    s(x,t)\right] \\
    \iff &s(x,t) = \sum_k \mathbf{1}_{[t_k, t_{k + 1})}(t) \left[ \frac{r_t^{(k)}}{\sigma_t^2 \left( \frac{\dot s_t^{(k)}}{s_t^{(k)}} r_t^{(k)} - r_t^{(k)}\right)}
    \left(v(x,t) - \frac{\dot s_t^{(k)}}{s_t^{(k)}}x \right) \right]
     \end{split}
\end{equation}
Now recall that \[
\frac{dX_t}{dt} = v(X_t, t) \iff dX_t = \left( v(X_t, t) + \frac{\eps(t)^2}{2} s(x,t) \right) \, dt + \eps(t) \, dB_t
\]
as given in \citet[Eqns. 3, 4]{domingo-enrichAdjointMatchingFinetuning2024}. Let $b(X_t, t)$ denote the drift term in the SDE above. Plugging in $v$ and $s$ from Eqn. \ref{eq:mmfm-vel-score-equiv}, we find 
\begin{equation}
    \label{eq:mmfm-eta}
    \eta_t = \sum_{k = 0}^{K-1} \mathbf{1}_{[t_k, t_{k + 1})} (t) \frac{\sigma_t^2 \left( \frac{\dot s_t^{(k)}}{s_t^{(k)}} r_t^{(k)} - r_t^{(k)}\right)}{r_t^{(k)}}.
\end{equation}
Therefore, the memory-less noise schedule $\eps(t) = \sqrt{2\eta_t}$ yields
\[
b(X_t, t) = \sum_{k = 0}^{K-1} \mathbf{1}_{[t_k, t_{k + 1})} (t) \left[ 2v(X_t, t) - \frac{\dot s_t^{(k)}}{s_t^{(k)}}X_t\right]
\]
and by analogy we define the $v^{ft}$ via a control function $u$:
\begin{equation}
    b(x, t) + \eps(t) u(x,t) = \sum_{k = 0}^{K-1} \mathbf{1}_{[t_k, t_{k + 1})} (t)\left[ 2v^{ft}(x, t) - \frac{\dot s_t^{(k)}}{s_t^{(k)}}x\right].
\end{equation}

\subsection{Practical considerations and scaling}
\label{app:practice-and-scaling}
Thus far, our method has been demonstrated on small CNN models (on the order of 10,000 parameters), and a natural direction for future work is extending it to larger network architectures. For instance, the collection of training trajectories, detailed in App. \ref{app:pre-training-acqui}, requires about 2 hours of base model training (NVIDIA A40 GPU) and 1 GB of storage per downstream dataset (e.g. CIFAR10). The meta-model itself comes to about 4M parameters with slight deviations caused by differences in base model size (e.g. MNIST classifiers are slightly smaller than CIFAR10 classifiers); more generally, it scales linearly with the number of parameters of the base model. The modular nature of our framework, however, provides multiple avenues for improving scalability. One promising path is leveraging the growing number of publicly available model zoos for larger architectures \citep{falk2025modelzoovisiontransformers, schürholt2025modelzoophasetransitions}, which would reduce both training time, and possibly storage requirements if models can be streamed on demand. Another direction is the incorporation of encoding schemes such as VAEs, which have proven critical in scaling weight-space learning to architectures like ResNet-18 and Vision Transformers \citep{schuerholt2024sane, wang2025recurrentdiffusionlargescaleparameter}.  Furthermore, when only weight embeddings are needed for training, storage costs can be further reduced by pre-computing and saving embeddings directly, rather than full model weights.

\section{Further experiments}
\label{app:futher-exps}
In this section, we explore in depth the design choices that were made concerning the generative model. In particular, we bring to light some upsides of flow matching over prior diffusion methods \citep{soroDiffusionbasedNeuralNetwork2024}, evaluate each approach's ability to model weight trajectories, and examine the effects of weight space symmetry \citep{hechtnielsen1990algebraicstructure, chen1993geometryfeedforward}.

\subsection{Parameter symmetries}
In this section, we detail two methods for incorporating parameter symmetries into the modeling framework. As mentioned in the main text, the symmetries we investigate involve permutation \citep{hechtnielsen1990algebraicstructure} and scaling \citep{chen1993geometryfeedforward} of neurons in a neural network. We proceed with two approaches: weight alignment and equivariant architectures.

\paragraph{Alignment to a reference.} This approach involves choosing a reference base model and aligning the \textit{permutation} of the layers of base models in the training and evaluation sets; we use Git Re-Basin \citep{ainsworthGitReBasinMerging2023} for this task. 
We perform few-step inference and trajectory modeling with this setup, as shown in Table \ref{tab:few-step-inference} and Figure \ref{fig:w1-tests-aligned}.

\begin{wrapfigure}{r}{0.4\textwidth}
  \begin{center}
    \includegraphics[width=0.38\textwidth]{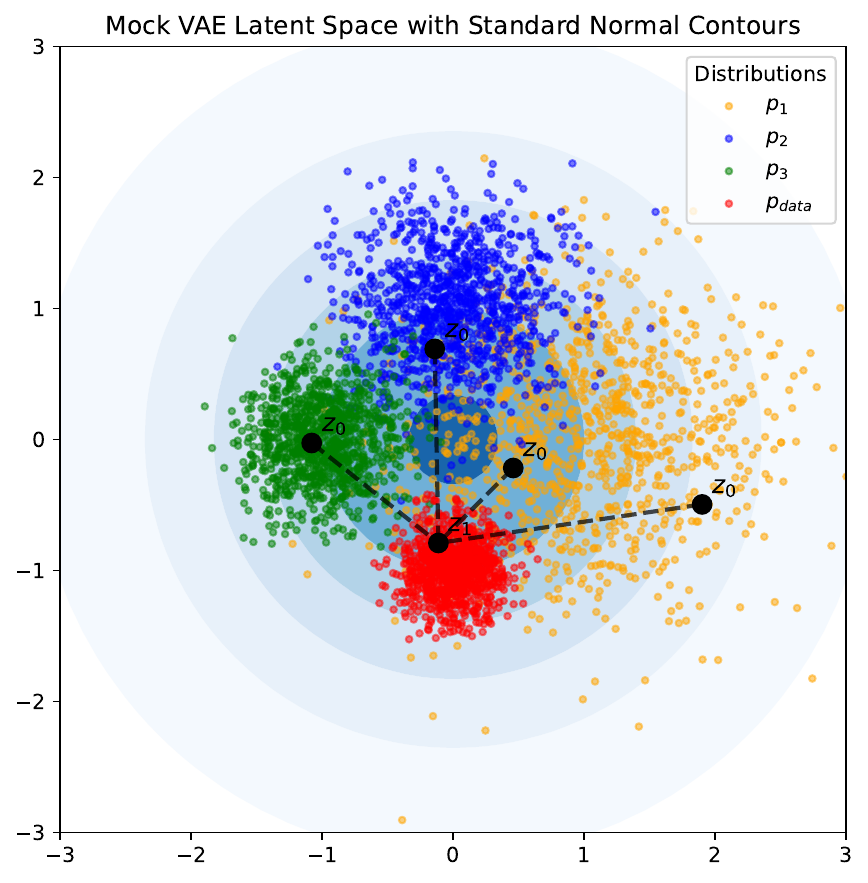}
  \end{center}
  \caption{Mock visualization of VAE latent space in D2NWG. Although the interpolant (dotted line connecting $z_0$ to $z_1$) does not follow the reference trajectory $p_0 \to p_1 \to \dots \to p_{data}$, the points on the line reside within the data manifold.}
  \vspace{-3em}
  \label{fig:vae-latent-space}
\end{wrapfigure}

\paragraph{MonomialNFN as a meta-model.} In this section, we subsitute the UNet for the equivariant architecture MonomialNFN \citep{tran2024monomial}, which accounts for \textit{both} permutation and scaling symmetries. The network design largely follows from the designs in \citet{zhou2023permutation, tran2024monomial}. We start with a Gaussian Fourier Transformation with a mapping size of 16 and scale of 3. Then, we encode the features with a IOSinusoidalEncoding. The encoded features are passed into 4 MonomialNFN layers, employing a residual connection, of hidden dimension 32. Before every layer, we add a time-conditioned HyperNetwork \citep{sitzmann2020implicit}, similar to the time embedding layers of the UNet architecture. Finally, a last MonomialNFN layer is used to reduce the channel down to 1. We used a batch size of 32, and for optimization, we instantiate an AdamW optimizer with learning rate $10^{-3}$ and weight decay $10^{-2}$ for 1000 epochs. As we have explored few-step inference and trajectory modeling with layer alignment above, we focus on reward fine-tuning and evaluating the support of classifier weights as done in App. \ref{app:reft-results}, as shown in Table \ref{tab:corrupt-ft-aug}.

\subsection{Few-step inference}
In this experiment, we wish to evaluate the inference capability of the generative meta-models under strained conditions, in this case, one-step and two-step inference of CNN3 base models. To further distinguish the flow matching framework, we also test \fsl-CFM with a Gaussian source distribution (as opposed to Kaiming uniform). As comparison, we have the diffusion baseline from D2NWG \citep{soroDiffusionbasedNeuralNetwork2024}. 

Table \ref{tab:few-step-inference} shows that \fsl-CFM mostly outperforms D2NWG in the standard and aligned settings (see below for description). Moreover, we see a clear advantage when using Kaiming uniform (the weight initialization used during base model pre-training) over the standard Gaussian. Concerning parameter symmetries, Table \ref{tab:few-step-inference} shows considerable improvements, especially on CIFAR10 tests, indicating that alignment helps in compute-constrained environments. The disparity between source distributions for the \fsl-CFM approach can be explained by the distance between the Gaussian and trained weights distribution (see index 0 of the bottom row plots in Figure \ref{fig:w1-tests}). In contrast, the VAE loss includes a KL divergence term which regularizes the latent space towards a standard Gaussian (Eqn. \ref{eq:vae}).

\begin{table}[h]
\caption{Mean validation accuracy over 10 runs of unconditional generation for CIFAR10 and MNIST using just one or two inference steps. We include results where the \fsl-CFM source distribution is a standard Gaussian, as well as training runs where we aligned the layers to a reference model. Best mean accuracies are \textbf{bolded}, second-best \underline{underlined}.}
\label{tab:few-step-inference}
\tablestyle{3pt}{1.0}
\begin{center}
\begin{adjustbox}{max width=\textwidth}
\begin{tabular}{l @{\hspace{10pt}}cc cc}
\toprule
&  \multicolumn{2}{c}{One-step} & \multicolumn{2}{c}{Two-step} \\
\cmidrule(lr){2-3} \cmidrule(lr){4-5} 
& CIFAR10 & MNIST & CIFAR10 & MNIST \\
\midrule

D2NWG & \pmval{51.36}{4.64} & \pmval{93.42}{2.39} & \pmval{50.54}{4.75} & \pmval{93.44}{2.35} \\
D2NWG-Aligned & \underline{\pmval{55.47}{5.06}} & \pmval{95.94}{0.80} & \pmval{55.23}{5.24} & \pmval{95.96}{0.81} \\
\fsl-CFM w/ Gauss & \pmval{26.72}{5.81} & \pmval{50.83}{13.71} & \pmval{47.46}{10.99} & \pmval{88.60}{21.9} \\
\fsl-CFM-Aligned w/ Gauss & \pmval{25.63}{4.04} & \pmval{55.75}{17.72} & \pmval{49.15}{5.18} & \pmval{96.07}{3.02}\\
\fsl-CFM & \pmval{49.75}{3.07} & \underline{\pmval{96.64}{1.06}} & \underline{\pmval{62.98}{0.47}} & \underline{\pmval{98.11}{0.19}} \\
\fsl-CFM-Aligned & \textbf{\pmval{57.18}{2.34}} & \textbf{\pmval{97.89}{0.53}} & \textbf{\pmval{63.27}{0.34}} & \textbf{\pmval{98.62}{0.04}} \\
\bottomrule
\end{tabular}
\end{adjustbox}
\vspace{-1em}
\end{center}
\end{table}

\begin{figure}
    \centering
    \includegraphics[width=0.49\linewidth]{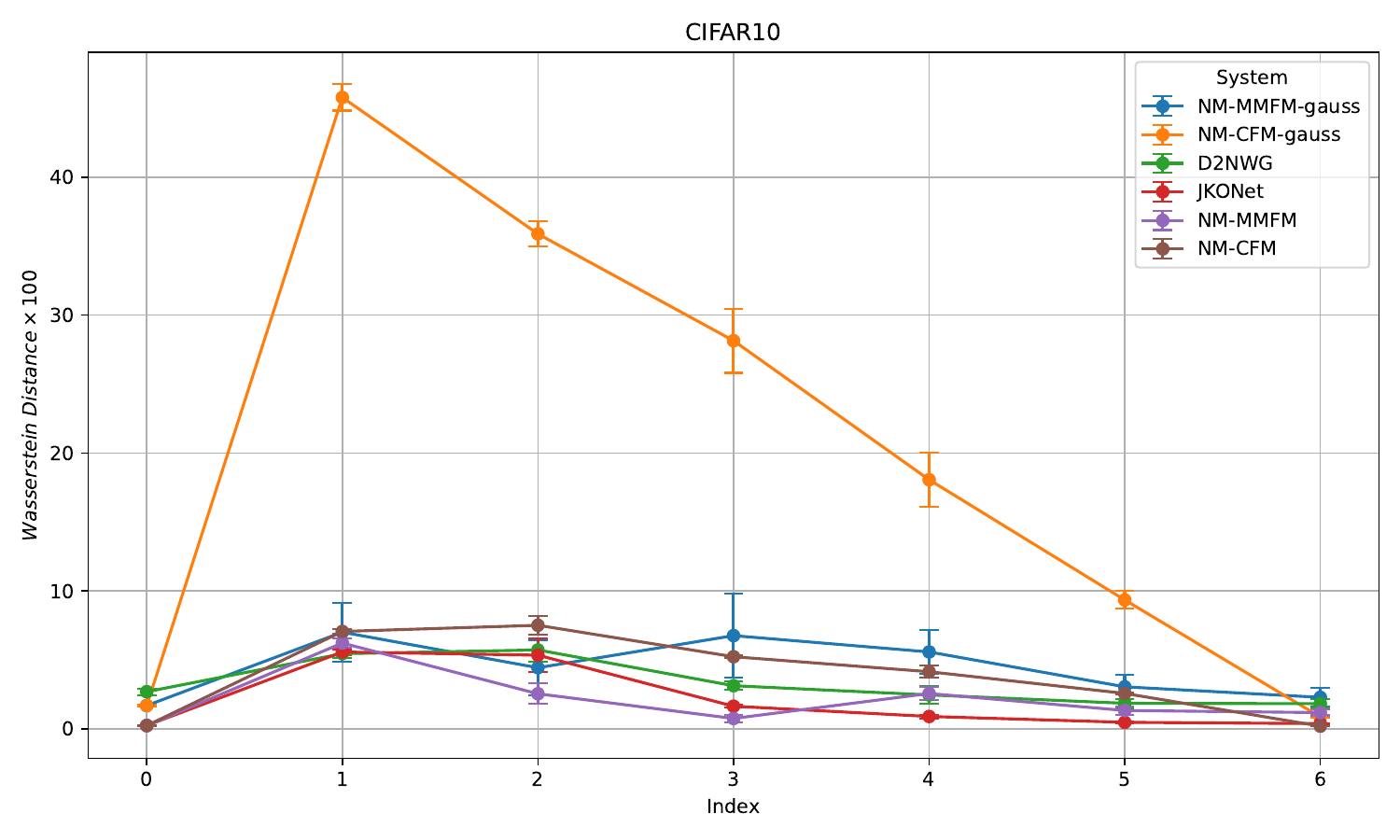}
    \includegraphics[width=0.49\linewidth]{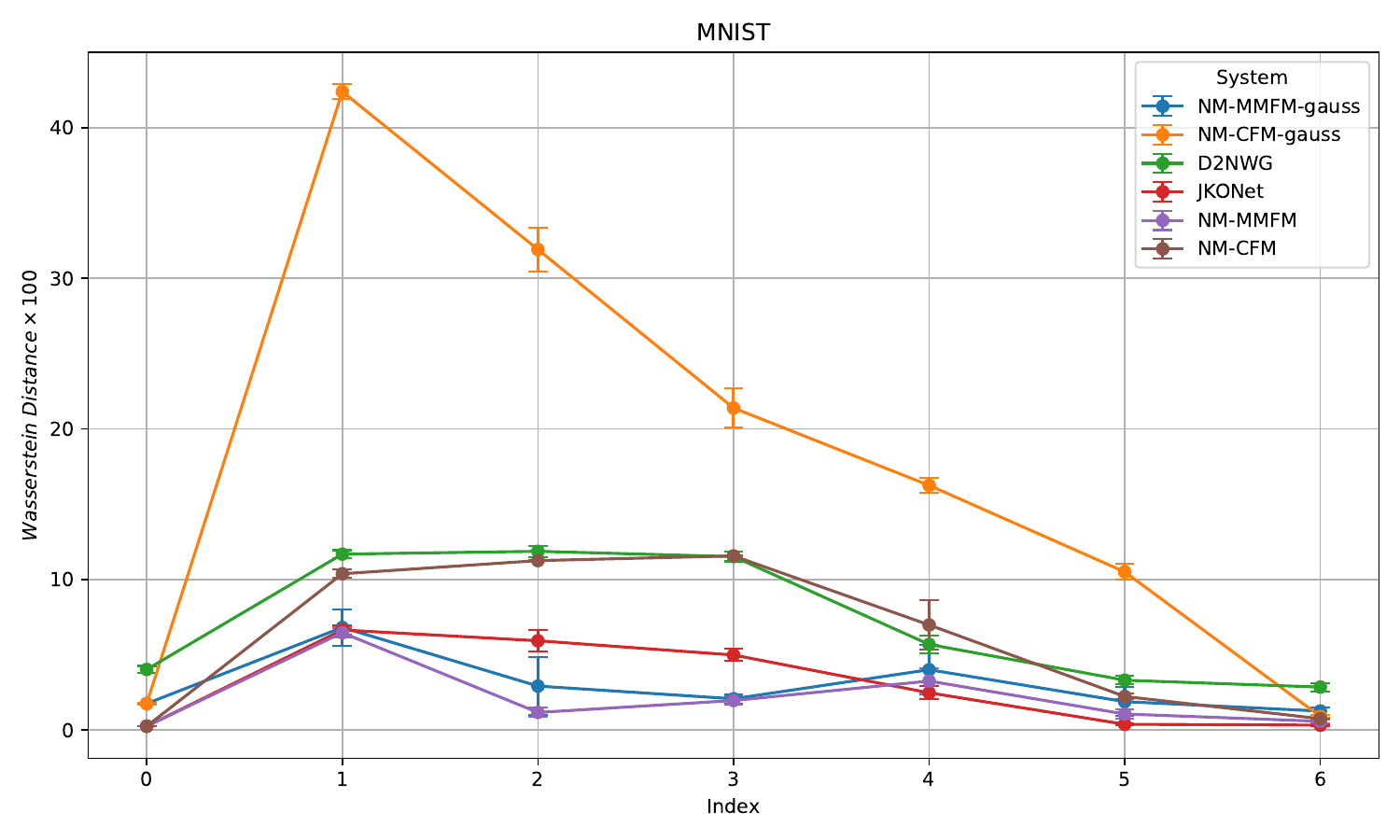}
    \includegraphics[width=0.49\linewidth]{figures/cifar10_zoomed_wasserstein_plot.pdf}
    \includegraphics[width=0.49\linewidth]{figures/mnist_zoomed_wasserstein_plot.pdf}
    \caption{Mean Wasserstein-1 distance ($\times 100$) between reference and generated intermediate marginals over 5 seeds of unconditional generation. We also provide plots excluding \fsl-CFM w/ Gauss due to its deviation from the rest.}
    \label{fig:w1-tests}
\end{figure}

\subsection{Trajectory modeling}
In this experiment, we evaluate the ability of different approaches to model the weight trajectory. As decided in Section \ref{sec:experiments}, we used \fsl-MMFM(3) and \fsl-JKO(4) as representatives for this method. Moreover, in the interest of fairness, we divide the trajectory into 5 buckets, and so the MMFM and JKO methods would need to interpolate between training distributions. Once again, our baseline is D2NWG where the VAE is trained on full trajectory weights at each batch iteration. See Figure \ref{fig:w1-tests} for results.

Interestingly, we find that the $W_1$-distance of the generated trajectory to be consistently lower in D2NWG vs. \fsl-CFM, but both methods lag behind MMFM and JKO which explicitly models the weight trajectory. We suspect the D2NWG performance is because of the latent space which is trained to encode trajectory data. Thus, even if the interpolated weights do not follow the expected trajectory, it still lands on the data manifold; see Figure \ref{fig:vae-latent-space}. Concerning parameter symmetries, Figure \ref{fig:w1-tests-aligned} shows a more modest improvement in the intermediate indices, e.g. indices 1-3, where D2NWG improves significantly, but happens to degrade in later indices, crucially where downstream evaluation takes place. We hypothesize that the improvement for D2NWG arises due to autoencoder underfitting. Indeed, the autoencoder only take parameters as input and is trained to reconstruct it, however, the \fsl-CFM models take in an extra time parameter that helps distinguish between the different parameter distributions.

\begin{figure}
    \centering
    \includegraphics[width=0.49\linewidth]{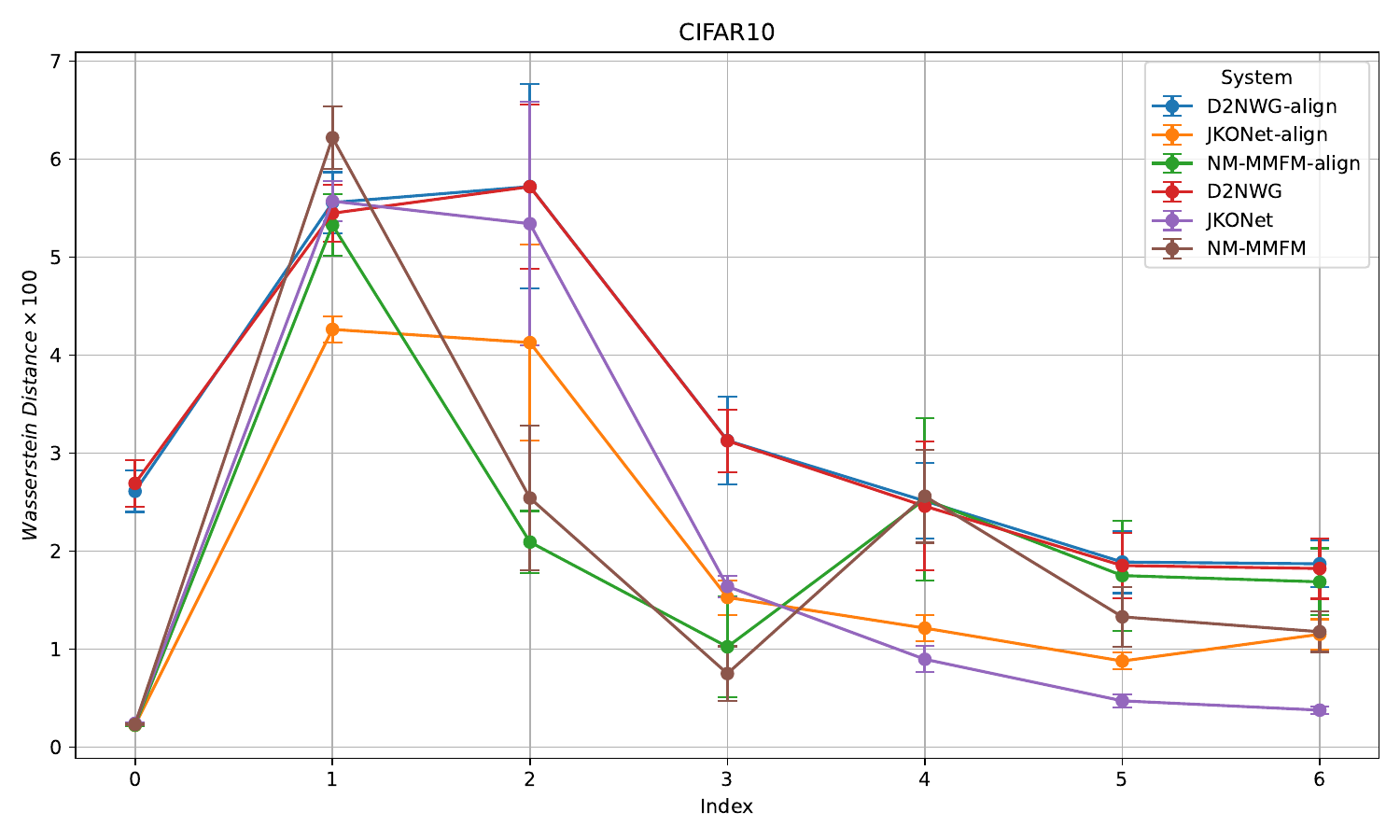}
    \includegraphics[width=0.49\linewidth]{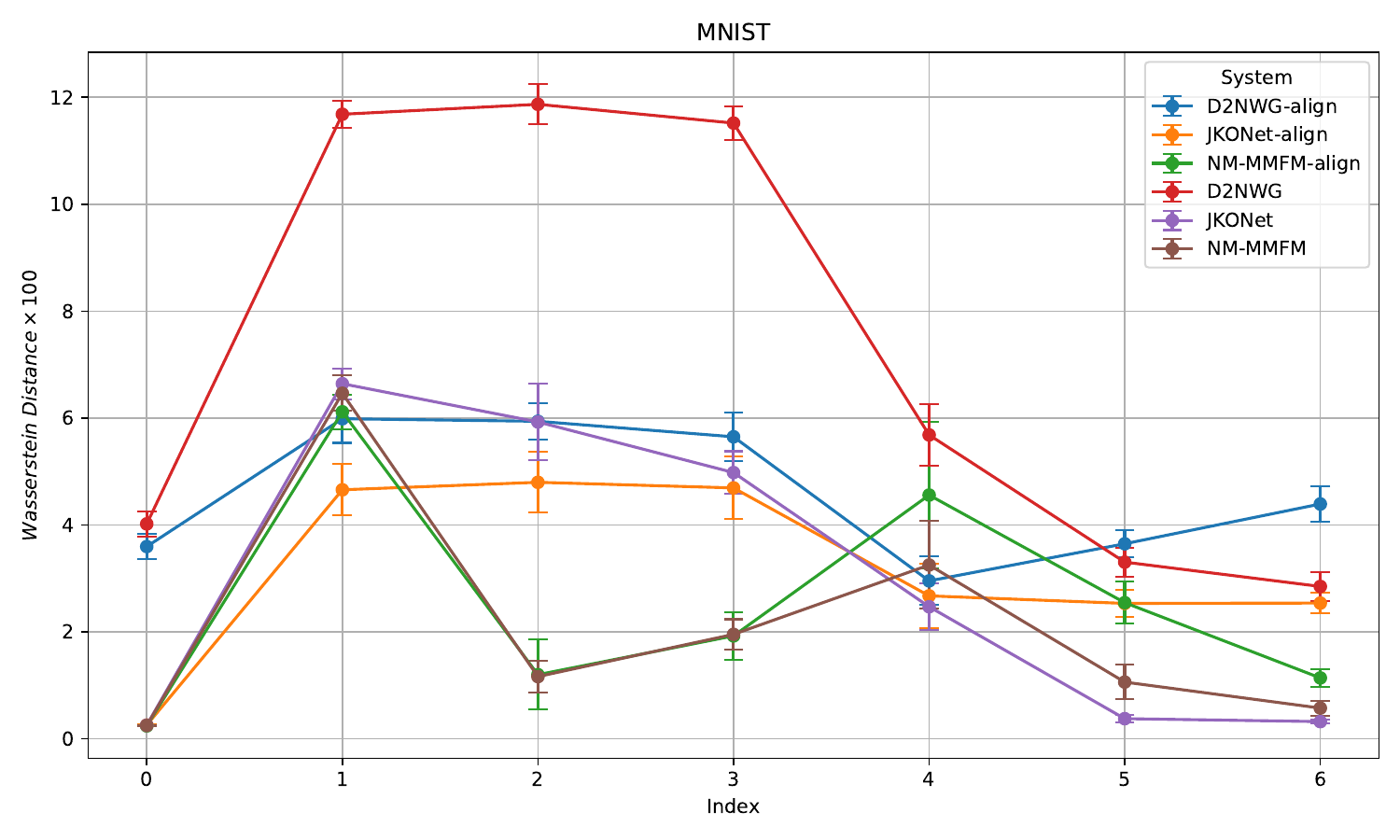}
    \includegraphics[width=0.49\linewidth]{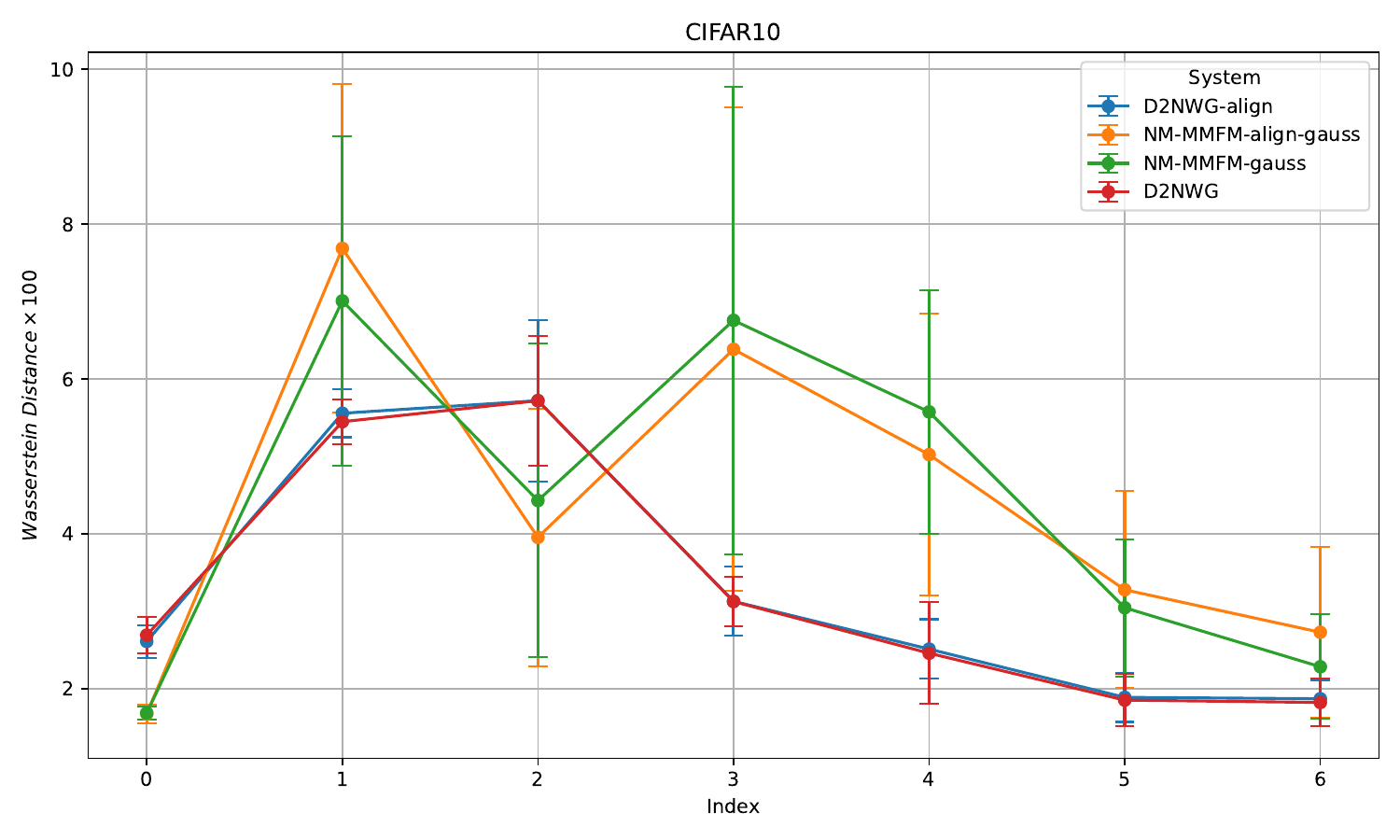}
    \includegraphics[width=0.49\linewidth]{figures/cifar10_aligned_gauss_wasserstein_plot.pdf}
    \caption{Mean Wasserstein-1 distance ($\times 100$) between reference and generated intermediate marginals over 5 seeds of unconditional generation, excluding \fsl-CFM results, illustrating the effect of permutation alignment on the training and validation data. Due to the high $W_1$ distances resulting from \fsl-CFM, we exclude them in the comparison.}
    \label{fig:w1-tests-aligned}
\end{figure}

\subsection{Reward fine-tuning: support of classifier weights}
We first note that for symmetry to hold, we had to remove the \verb|MaxPool2d| layers that were previously included in the CNN architecture, hence the different validation accuracy. This can be seen as a downside of current equivariant architectures.

Intuitively, if we equate data in the same equivalence class, then the effective support of classifier weights ought to expand, thus improving generalization to prediction on corrupted data. This intuition turns out to be false, as seen in Table \ref{tab:corrupt-ft-aug}. It shows that the new architecture in fact degrades the validation accuracy, which is most evident in the Corruption Level 2 column. We hypothesize this is because the members of the equivalence class are separated by a considerable distance and they do not correspond to meaningful regions of weight space that influence generalization on corrupted datasets. Empirically, this reinforces the findings of a recent work \citep{zeng2025weightsgeneralization}. Moreover, when using equivariant architectures, we forgo the basic Gaussian noise augmentation, which may be more beneficial for slight corruption in downstream task data. 

\begin{table}[h]
\caption{Mean validation accuracy over five generated classifiers after reward fine-tuning on increasingly corrupted datasets. The -MoNFN suffix indicates that network made use of the MonomialNFN \citep{tran2024monomial} architecture instead of the UNet. The arrow '$\rightarrow$' indicates the accuracy before (left) and after (right) reward fine-tuning.}

\label{tab:corrupt-ft-aug}
\tablestyle{3pt}{1.0}
\begin{center}
\begin{adjustbox}{max width=\textwidth}
\begin{tabular}{l @{\hspace{10pt}}ccc}
\toprule
& \multicolumn{3}{c}{CIFAR10} \\
\cmidrule(lr){2-4} 
Corruption Level & 0 & 1 & 2 \\
\midrule
SGD fine-tuning & 55.68 $\rightarrow$ 55.96 & 52.47 $\rightarrow$ 54.00 & 19.79 $\rightarrow$ 42.54 \\
\fsl-CFM-GaussAug & \pmval{54.52}{0.56} $\rightarrow$ \pmval{54.65}{0.57} & \pmval{51.62}{0.56} $\rightarrow$ \pmval{52.43}{0.72} & \pmval{19.45}{1.06} $\rightarrow$ \pmval{30.06}{0.80}  \\
\fsl-CFM-MoNFN & \pmval{54.09}{0.73} $\rightarrow$ \pmval{54.17}{1.70} & \pmval{50.94}{2.02} $\rightarrow$ \pmval{51.50}{1.01} & \pmval{19.76}{1.62} $\rightarrow$ \pmval{27.03}{1.65} \\

\midrule

& \multicolumn{3}{c}{MNIST} \\
\cmidrule(lr){2-4}
Corruption Level & 0 & 1 & 2 \\
\midrule
SGD fine-tuning & 92.55 $\rightarrow$ 93.65 & 86.92 $\rightarrow$ 91.19 & 9.51 $\rightarrow$ 87.64 \\
\fsl-CFM-GaussAug & \pmval{92.63}{0.18} $\rightarrow$ \pmval{92.65}{0.18} & \pmval{86.22}{0.73} $\rightarrow$ \pmval{88.74}{0.29} & \pmval{9.51}{0.01} $\rightarrow$ \pmval{30.36}{5.78} \\
\fsl-CFM-MoNFN & \pmval{92.06}{0.33} $\rightarrow$ \pmval{92.23}{3.48} & \pmval{83.52}{2.05} $\rightarrow$ \pmval{85.61}{5.58} & \pmval{9.51}{0.00} $\rightarrow$ \pmval{18.09}{2.52} \\

\bottomrule
\end{tabular}
\end{adjustbox}
\vspace{-1em}
\end{center}
\end{table}

\subsection{Discussion}

\subsubsection{Failure cases}
In this section, we note specific failure cases and discuss possible explanations. 
    
\paragraph{Gaussian \fsl-JKO.} In prior experiments, we also attempted to use the Gaussian distribution as the source distribution in the \fsl-JKO approach, but we observed failure (i.e. validation accuracy no better than chance) in all trials. We suspect this has to do with JKOnet's sensitivity to changes in scale given that the model output is simply a scalar value. Indeed, since the standard deviation and norm of parameters distributed by a Gaussian is considerably higher (about $10-100\times$) than the initialization (Kaiming uniform) or the converged weights, JKOnet would need to effectuate a large gradient $\nabla_xV(x,t)$ at small times $t$, and suddenly transition to small adjustments after the first intermediate distribution (for \fsl-JKO(4) this would be $t=0.2$).
    
\paragraph{Stochasticity levels.} When employing \fsl-CFM on a latent space created by a VAE, we observe good performance with a standard deviation $\sigma = 0.1$ when sampling the interpolant $x_t \sim p_t(x_t | x_0, x_1)$. However, this fails completely when applied to the raw weight space, where we instead set $\sigma=10^{-3}$. This clearly indicates that to generate performant base models, the parameters cannot deviate by much from the converged parameters (at least for the CNN architecture). 

\paragraph{Diffusion models on weight space.} We also attempted to use the diffusion model from D2NWG directly on weight space (i.e. D2NWG without the VAE component). We observe some decrease in the loss but found it to be at least an order of magnitude higher than the \fsl-CFM loss and the validation accuracy did not exceed 20\%. We hypothesize that this relates to the issue of stochasticity discussed above. With \fsl-CFM, one has greater control over the level of stochasticity as opposed to a diffusion model. Indeed, the forward process of diffusion requires noising towards a Gaussian distribution and so there is a clear tradeoff when we decrease the beta noise schedule $(\beta_t)_{t \in [0, T]}$: if $\beta_t$ is made small for a greater number of timesteps, then the forward process will not reach a proper Gaussian distribution. Consequently, this adversely affects the reverse (inference) process as the model will have a poor understanding of the source distribution.

\subsubsection{Conclusion}
From these results, several key observations emerge. First, given that diffusion models fail when applied directly in weight space---and considering the results in Table \ref{tab:few-step-inference}---we see clear advantages in end-generation precision when using an FM model over prior diffusion-based approaches. In particular, the ability to control the level of stochasticity appears important for achieving high base-model validation accuracy (see failure cases above). Second, the flexibility in choosing the source distribution also substantially affects the accuracy with which weight-space trajectories are modeled, as illustrated in Figure \ref{fig:w1-tests}. This further supports the use of (MM)FM, which can accurately model weight-space data without requiring an autoencoder. Lastly, we observe that although layer alignment of permutation states helps simplify the training distribution resulting in easier training (see Table \ref{tab:few-step-inference} and Figure \ref{fig:w1-tests-aligned}), in the usual case of inference, where we use 100 steps instead of 1 or 2, and when modeling weight trajectory, the improvements are quite modest as our current architecture has the capacity to fit the training data well. In fact, when we move on to using equivariant architectures, we find that in the case of reward fine-tuning, this degrades downstream performance.

\begin{table}[h]
\caption{Best validation accuracy of unconditional \fsl{} generation for various datasets. \textit{orig} denotes base models trained conventionally by SGD and \textit{p-diff} those generated with p-diff \citep{wangNeuralNetworkDiffusion2024}. We focus on generating just the batch norm parameters.}
\label{tab:uncond-app}
\tablestyle{3pt}{1.0}
\begin{center}
\begin{adjustbox}{max width=\textwidth}
\begin{tabular}{l @{\hspace{20pt}}ccc ccc ccc ccc}
\toprule
&  \multicolumn{3}{c}{CIFAR100} & \multicolumn{3}{c}{CIFAR10} & \multicolumn{3}{c}{MNIST} & \multicolumn{3}{c}{STL10} \\
\cmidrule(lr){2-4} \cmidrule(lr){5-7} \cmidrule(lr){8-10} \cmidrule(lr){11-13}
Base Models  & orig. & \fsl{} & p-diff. & orig. & \fsl{} & p-diff. & orig. & \fsl{} & p-diff & orig. & \fsl{} & p-diff\\
\midrule
ResNet-18 & 71.45 & 71.42 & 71.40 & 94.54 & 94.36 & 94.36 & 99.68 & 99.65 & 99.65 & 62.00 & 62.00 & 62.24 \\
ViT-base  & 85.95 & 85.86 & 85.85 & 98.20 & 98.11 & 98.12 & 99.41 & 99.38 & 99.36 & 96.15 & 95.77 & 95.80 \\
ConvNext-tiny  & 85.06 & 85.12 & 85.17 & 98.03 & 97.89 & 97.90 & 99.42 & 99.41  & 99.40 & 95.95 & 95.63 & 95.63  \\
\bottomrule
\end{tabular}
\end{adjustbox}
\vspace{-1em}
\end{center}
\end{table}

\section{Experimental details and further results}
\label{app:experimental-details}
In this section, we provide further experimental details such as hyperparameters and computation times, alongside some extra results.

\subsection{Unconditional generation}
Unconditional generation involves two stages: first is the training of base models. We choose a Resnet18, ViT-B, ConvNext-tiny, and CNN3 for our base models and provide the training parameters in Table~\ref{tab:task-training}. Next, we train the generative meta-model; Table~\ref{tab:generative} lists the training settings.
We found that in most cases, the autoencoder and \fsl-CFM converge after 1000 epochs. One should note that when using Kaiming uniform for the source $p_0$, it is also necessary to train the autoencoder on this distribution. On a NVIDIA A40 card, we estimate training time to be around 2 hours when the VAE is used. Un-encoded training generally takes double that time. Inference requires less than 1 minute to complete for both encoded and un-encoded runs, and it takes a few seconds to go through the validation set to compute the classification accuracy. Notably, the JKO runs were about \textbf{$2\times$ as fast} since the model output is a scalar. This means we can compress inputs through the downsampling layers of the UNet, leading to a smaller parameter count. Extra results are presented in Table \ref{tab:uncond-app}.

\paragraph{Remark on generating batch norm parameters.} In order to reduce the target parameter count, we restrict ourselves to the batch norm parameters for larger architectures as in Table \ref{tab:uncond-app}. We train base models in the same way as App. \ref{app:pre-training-acqui} suggests, but for training we only save \verb|state_dict| tensors with \verb|bn| in the key. The rest of the model is saved for evaluation: after generating batch norm parameters, we impute these parameters back into the trained base models and validate as usual.

\begin{table}[t]
\caption{Task training settings. Note that only CNN3 was evaluated for MMFM and JKOnet experiments, so it's the only model type with non-zero \textit{Num. of save epochs} and \textit{Savepoints per epoch}.}
\label{tab:task-training}
\tablestyle{1.5pt}{1.0}
\centering
\begin{tabular}{lccc}
\toprule
Model Type & ResNet18 & ViT \& ConvNext & CNN \\
\midrule
Optimizer & SGD & AdamW & SGD\\
Initial training LR & 0.1 & $1\times 10^{-4}$ & 0.1\\
Training scheduler & MultiStepLR & CosineAnnealingLR & MultiStepLR \\
Layer params. saved & Last 2 BN layers & Last 2 BN layers & All layers\\
Num. of save epochs & 0 & 0 & 100 \\
Savepoints per epoch & 0 & 0 & 100 \\
Num. final weights saved & 200 & 200 & 200 \\
Saved parameter count & 2048 & 3072 & [10565, 12743]\\
Training epochs & 100 & 100 & 100\\
Batch Size & 64 & 128 & 128\\
\bottomrule
\end{tabular}

\vskip -0.10in
\end{table}

\begin{figure}
    \centering
    \includegraphics[width=0.8\linewidth]{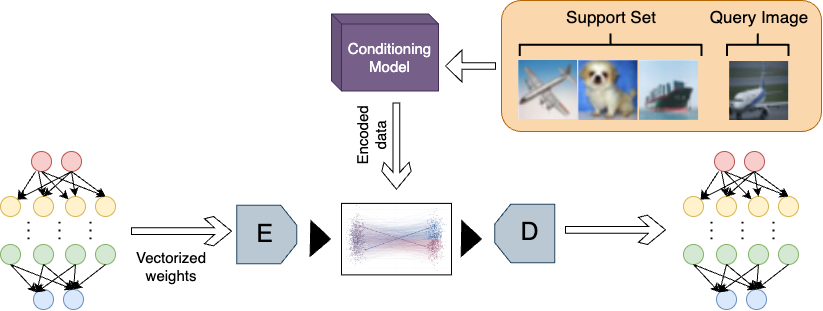}
    \caption{A schematic of the training process of \fsl{}-CFM w/ VAE for conditional generation. Given a set of pre-trained target weights and a support set of images, we apply the conditioned flow model to pushforward a sample of the latent prior towards encoded target weights. The decoder is used during inference where we start from a sample of the latent prior and pushforward towards the target distribution with a trained vector field $v_\theta(\cdot, t; \vy)$ where $\vy$ is the support set embedding.}
    \label{fig:f2sl}
\end{figure}

\subsection{Model retrieval} 
For each dataset, we first sample 30 images as representatives, and then pre-compute its CLIP \citep{Radford2021LearningTV} embeddings. We then have a choice of how to aggregate the 30 embeddings. In our case, we use a light multi-head attention module and a linear layer to compress the context into a vector. Generally, we maintain the feature dimension---the version of CLIP we use returns a 768-dimensional vector. To condition the flow/diffusion model with this context, we concatenate to the channel dimension if we are using a VAE, or to the last dimension otherwise. One should note that the number of parameters differ between datasets; for instance, the three-channel datasets have slightly more parameters than their one-channel counterparts. Hence, we apply a simple zero-padding to standardize the input dimension. On a NVIDIA A40 card, we estimate a training time of 2 hours for the VAE and 3 hours for the CFM to achieve our level of accuracy (w/ VAE) and about 5 hours un-encoded. Inference times remain the same as in \textit{unconditional generation}.

\begin{table}[t]
\caption{Training settings for modules in the unconditional generation experiment. The number of JKOnet inference steps depends on the number of intermediate marginal distributions we are modeling.}
\label{tab:generative}
\vspace{0.5em}
\tablestyle{1.5pt}{1.0}
\centering
\begin{tabular}{l cccc}
\toprule
Model Type & Autoencoder & CFM w/ AE & MMFM & JKOnet\\
\midrule
Optimizer & AdamW & AdamW & AdamW & AdamW\\
(LR, weight decay) & (1e-4, 2e-6) & (1e-4, 2e-6) & (3e-4, 2e-6) & (5e-3, 2e-6) \\
Num. inference steps & n/a & 100 & 100 & variable \\
Weight initialization & n/a & Kaiming uniform & Kaiming uniform & Kaiming uniform \\
Epochs & 500 & 1000 & 1000 & 1000 \\
Batch size & 64 & 64 & 64 & 64 \\
\bottomrule
\end{tabular}
\vskip -0.10in
\end{table}

\subsection{Downstream initialization}
In this evaluation, seen in Table \ref{tab:ft-retrieval}, we obtain weights from model retrieval (albeit trained with weights before convergence), and fine-tune on their corresponding dataset, i.e. if we retrieve an MNIST classifier, we fine-tune on MNIST and $\reallywidetilde{\text{MNIST}}$. Fine-tuning is done conventionally with the task training settings in Table \ref{tab:task-training}. The corruption of datasets is done by applying the following transformations: random horizontal flip, random rotation (max $15$ deg), color jittering, Gaussian blur (kernel size of 3, $\sigma \in [0.1, 2.0]$).

\begin{table*}[h]
\vspace{-1em}
\caption{Mean validation accuracy over top-5 fine-tuned generated weights post-retrieval. Setup follows Table \ref{tab:model-retrieval}, excepting the target weight data. $\widetilde{\text{tilde}}$ indicates corrupted datasets, on which the model was not trained.}
\label{tab:ft-retrieval}
{\fontsize{10}{10}\selectfont
\tablestyle{2.2pt}{1.0}
\renewcommand{\arraystretch}{0.9}
\setlength{\tabcolsep}{4pt}
\begin{center}
\begin{adjustbox}{max width=\textwidth}
    \begin{tabular}{l l c c c c c c c c}
    \toprule
    Epoch & Method &
    $\reallywidetilde{\text{CIFAR10}}$ & $\reallywidetilde{\text{STL10}}$ & $\reallywidetilde{\text{MNIST}}$ & $\reallywidetilde{\text{FMNIST}}$ &
    CIFAR10 & STL10 & MNIST & FMNIST\\
    \midrule
    \raisebox{-1.5\normalbaselineskip}[0pt][0pt]{0}
     & RandomInit & $\sim 10\%$ & $\sim 10\%$ & $\sim 10\%$ & $\sim 10\%$ & $\sim 10\%$ & $\sim 10\%$ & $\sim 10\%$ & $\sim 10\%$\\
     & \fsl-CFM w/ VAE & \pmval{39.66}{0.08} & \pmval{36.85}{0.36} & \pmval{90.93}{0.35} & \pmval{72.44}{0.27} & \pmval{44.63}{0.05} & \pmval{40.06}{0.02} & \pmval{95.57}{0.05} & \pmval{83.68}{0.24}\\
     & \fsl-MMFM w/ VAE & \pmval{39.08}{0.22} & \pmval{21.98}{ 0.89} & \pmval{90.55}{0.23} & \pmval{74.08}{0.08} & \pmval{41.89}{0.37} & \pmval{36.41}{1.54} & \pmval{90.02}{1.09} & \pmval{83.38}{0.29}\\
     & \fsl-JKO & \pmval{40.53}{0.27} & \pmval{18.83}{0.30} &\pmval{90.80}{0.20} & \pmval{74.45}{0.15} & \pmval{45.04}{0.19} & \pmval{41.76}{0.21} & \pmval{95.42}{0.22} & \pmval{83.90}{0.04}\\
    \midrule
    
    \raisebox{-1.5\normalbaselineskip}[0pt][0pt]{1}
     & RandomInit & \pmval{34.05}{1.13} & \pmval{16.64}{1.61} & \pmval{83.21}{0.57} & \pmval{66.81}{0.74}	& \pmval{36.27}{2.05} & \pmval{22.13}{2.11} & \pmval{96.52}{0.20} & \pmval{77.21}{0.26} \\
     & \fsl-CFM w/ VAE & \pmval{44.68}{0.16}	& \pmval{38.18}{0.40} & \pmval{94.02}{0.42} & \pmval{76.43}{0.22} & \pmval{48.28}{0.14} & \pmval{41.66}{0.19} & \pmval{97.38}{0.10} & \pmval{85.27}{0.30} \\
     & \fsl-MMFM w/ VAE & \pmval{43.96}{0.37} & \pmval{29.44}{1.25} &	\pmval{94.71}{0.35} & \pmval{79.85}{0.47} & \pmval{47.98}{0.32} & \pmval{39.93}{1.45} & \pmval{97.68}{0.13} & \pmval{85.50}{0.12}\\
     & \fsl-JKO & \pmval{45.04}{0.31} & \pmval{22.26}{0.08} & \pmval{94.94}{0.17} & \pmval{80.63}{0.20} & \pmval{49.61}{0.29} & \pmval{42.41}{0.31}  & \pmval{97.39}{0.02} & \pmval{85.48}{0.13} \\
     
    \midrule
    \raisebox{-1.5\normalbaselineskip}[0pt][0pt]{5}
     & RandomInit & \pmval{46.55}{0.80} & \pmval{25.08}{1.13} & \pmval{92.53}{0.28} & \pmval{79.08}{0.93} & \pmval{47.74}{1.33}	 & \pmval{31.55}{2.00} & \pmval{98.24}{0.03} & \pmval{84.87}{0.40}\\
     & \fsl-CFM w/ VAE & \pmval{47.41}{0.13}	& \pmval{39.98}{0.29} & \pmval{95.22}{0.09} & \pmval{79.88}{0.61} & \pmval{51.69}{0.14} & \pmval{42.62}{0.08} & \pmval{98.04}{0.05} & \pmval{86.76}{0.15}\\
     & \fsl-MMFM w/ VAE & \pmval{47.06}{0.45} & \pmval{35.24} {0.72} & \pmval{95.92}{0.19} & \pmval{82.08}{0.11} & \pmval{51.70}{0.21} & \pmval{41.35}{0.65} & \pmval{98.51}{0.02} & \pmval{86.64}{0.14}\\
     & \fsl-JKO & \pmval{48.14}{0.08} & \pmval{26.00}{0.28} & \pmval{95.92}{0.11} & \pmval{82.87}{0.19} & \pmval{53.01}{0.11} & \pmval{43.38}{0.29} & \pmval{98.13}{0.05} & \pmval{86.75}{0.04} \\
     
    \midrule
    \raisebox{-1.5\normalbaselineskip}[0pt][0pt]{20}
     & RandomInit & \pmval{50.28}{0.43} & \pmval{33.63}{0.99} & \pmval{95.81}{0.18} & \pmval{82.36}{0.42} & \pmval{51.35}{1.21} & \pmval{44.16}{1.28} & \pmval{98.51}{0.05} &  \pmval{88.25}{0.69}\\
     & \fsl-CFM w/ VAE & \pmval{52.25}{0.18}	& \pmval{41.18}{0.28} & \pmval{96.41}{0.08} & \pmval{83.42}{0.23} & \pmval{55.66}{0.23}	& \pmval{44.38}{0.16} & \pmval{98.25}{0.05} &  \pmval{88.02}{0.17}\\
     & \fsl-MMFM w/ VAE & \pmval{52.57}{0.73} & \pmval{39.80}{0.48} & \pmval{97.01}{0.21} & \pmval{84.38}{0.08} & \pmval{55.85}{0.76}	& \pmval{44.10}{0.26} & \pmval{98.85}{0.03} & \pmval{88.29}{0.03}\\
     & \fsl-JKO & \pmval{52.59}{0.02} & \pmval{34.53}{0.35} & \pmval{96.82}{0.10} & \pmval{84.89}{0.20} & \pmval{56.65}{0.27} & \pmval{45.06}{0.21} & \pmval{98.43}{0.03} & \pmval{88.10}{0.14} \\
    \midrule
    30
     & RandomInit & \pmval{52.99}{0.55} & \pmval{37.79}{0.55} & \pmval{96.55}{0.22} & \pmval{84.16}{0.66} & \pmval{	56.05}{1.21} & \pmval{45.80}{1.19} & \pmval{98.55}{ 0.05} & \pmval{88.55}{0.66}\\
    \bottomrule
    \end{tabular}
\end{adjustbox}
\end{center}
}
\vspace{-1em}
\end{table*}

\subsection{Reward fine-tuning}
\label{app:reft-results}
\paragraph{Training settings.}
For the corresponding experiments, we used an AdamW optimizer with learning rate $2\times 10^{-5}$ and weight decay $5\times 10^{-4}$. We use a trajectory batch size of 8 (denoted $M$ in Algorithm \ref{alg:det-am}) and a dataset batch size $m = 128/512$ depending on the dataset. We also clip gradients at norm 1.0 and set the number of fine-tuning iterations to 150. We use a cosine annealing scheduler with $\eta_{\min} = 10^{-6}$. The step size was set to $h=0.025$, meaning our trajectory consists of 40 timesteps. As suggested in \citet[App. H]{domingo-enrichAdjointMatchingFinetuning2024}, we only evaluated gradients at 20 out of 40 timesteps: 10 of the last timesteps and a uniform sample of 10 from the first 30 timesteps.

\paragraph{Training tricks.} We also introduce a few tricks. First, Algorithm \ref{alg:det-am} suggests that we must take a fine-tuning step every time we sample a dataset batch. We instead opt to average the starting lean adjoint $\tilde a_1$ over 3 batch samples; we found this to result in more stable training losses. In fact, this is all the iterations done per epoch, so if $N=150$, we only have 450 batch iterations total. Moreover, as $\tilde a_1$ is the gradient of a classifier loss, we experimented with treating it like a stochastic gradient descent step, which means including a learning rate, momentum, and weight decay parameter. To clarify, this involves saving the $\tilde a_1$ values from previous training iterations. It does not seem to help much, other than the learning rate. We set the reward learning rate to 1.5 and momentum to 0.01. 

\paragraph{Padding regularization.}
Another trick specific to \fsl-CFM-All is to use padding regularization. As mentioned in the model retrieval section, zero-padding is applied to standardize the input dimension given the variability between dataset classifiers. This trick indexes the padded elements in the input tensor and adds its $\ell_2$-norm to the loss, thus coercing it towards zero. Notably, this is an implicit method of conditioning on the dataset. For example, if we are regularizing more elements of the input, that suggests we are classifying on a one-channel dataset since these classifiers require more padding.

\paragraph{Weight augmentation.} In our experiments, we also tried augmenting the network weight data acquired from pre-training in an attempt to expand $\text{supp } \hat p_1$. This is done by simple Gaussian noise ($\sigma = 5 \times 10^{-3}$), dropout $(p=0.02$), and mix-up. Recall that mix-up involves sampling a data pair $(w_1, w_2)$ and an interpolation parameter $\alpha \sim \text{Uniform}[0, 1]$, and returning an interpolation $(1-\alpha)w_1 + \alpha w_2$.

\paragraph{Corruption levels.} As part of our experiments to get a sense of the width of classifier supports, we applied increasing corruption to the base datasets. The transformations are as follows:
\begin{enumerate}
    \item Level 1: random horizontal flip, random rotation (max $15$ deg).
    \item Level 2: Level 1 + color jittering, and Gaussian blur (kernel size of 3 and $\sigma \in [0.1, 2.0]$).
    \item Level 3: Level 2 + random erasing with $p = 0.5$, scale in [0.2, 0.5], and ratio in [0.3, 3.3].
\end{enumerate}

\begin{table}[h]
\caption{Mean validation accuracy over five generated classifiers after reward fine-tuning on increasingly corrupted datasets. The -All suffix indicates that the CFM was trained on classifiers of CIFAR10, STL10, MNIST, and FMNIST, whereas +A indicates weight augmentation, and +P indicates regularization on padded values. The arrow '$\rightarrow$' indicates the accuracy before (left) and after (right) reward fine-tuning.}

\label{tab:corrupt-ft}
\tablestyle{3pt}{1.0}
\begin{center}
\begin{adjustbox}{max width=\textwidth}
\begin{tabular}{l @{\hspace{10pt}}ccc}
\toprule
& \multicolumn{3}{c}{CIFAR10} \\
\cmidrule(lr){2-4} 
Corruption Level & 0 & 1 & 2 \\
\midrule
SGD fine-tuning & 63.38 $\rightarrow$ 63.38 & 59.93 $\rightarrow$ 60.91 & 24.18 $\rightarrow$ 49.90  \\
 \fsl-CFM & \pmval{62.53}{0.02} $\rightarrow$ \pmval{63.33}{0.08} & \pmval{58.65}{0.22} $\rightarrow$ \pmval{60.34}{0.76} & \pmval{24.84}{0.93} $\rightarrow$ \pmval{34.15}{0.74}  \\

\fsl-CFM-All & \pmval{28.92}{11.19} $\rightarrow$ \pmval{61.90}{0.22} & \pmval{23.57}{16.43} $\rightarrow$ \pmval{57.65}{0.91} & \pmval{22.72}{1.56} $\rightarrow$ \pmval{34.59}{1.88}  \\

\fsl-CFM-All+A & \pmval{45.07}{15.64} $\rightarrow$ \pmval{56.24}{1.65} & \pmval{37.53}{15.16} $\rightarrow$ \pmval{55.99}{1.17} &  \pmval{19.89}{4.11} $\rightarrow$ \pmval{32.44}{1.73} \\

\fsl-CFM-All+P & \pmval{50.18}{18.49} $\rightarrow$ \pmval{60.29}{0.62} & \pmval{51.36}{14.51} $\rightarrow$ \pmval{56.64}{1.14} & \pmval{21.84}{3.69} $\rightarrow$ \pmval{33.97}{1.23}  \\

\fsl-CFM-All+A+P & \pmval{39.25}{17.68} $\rightarrow$ \pmval{58.85}{1.33} & \pmval{41.59}{15.05} $\rightarrow$ \pmval{55.91}{1.79} & \pmval{19.21}{3.30} $\rightarrow$ \pmval{33.79}{2.00} \\

\midrule

& \multicolumn{3}{c}{MNIST} \\
\cmidrule(lr){2-4}
Corruption Level & 0 & 1 & 2 \\
\midrule
SGD fine-tuning & 98.93 $\rightarrow$ 98.93 & 96.58 $\rightarrow$ 97.78 & 18.8 $\rightarrow$ 97.55 \\
 \fsl-CFM & \pmval{98.52}{0.01} $\rightarrow$ \pmval{98.79}{0.04} & \pmval{95.87}{0.01} $\rightarrow$ \pmval{97.01}{2.27} & \pmval{15.68}{0.17} $\rightarrow$ \pmval{91.21}{3.05} \\

\fsl-CFM-All & \pmval{60.77}{27.58} $\rightarrow$ \pmval{97.56}{1.20} & \pmval{35.74}{27.14} $\rightarrow$ \pmval{59.32}{27.88} & \pmval{26.68}{22.11} $\rightarrow$ \pmval{65.40}{34.17} \\

\fsl-CFM-All+A & \pmval{37.67}{37.07} $\rightarrow$ \pmval{95.84}{0.53} & \pmval{49.42}{37.93} $\rightarrow$ \pmval{92.50}{8.99} & \pmval{11.14}{4.65} $\rightarrow$ \pmval{35.32}{28.25} \\

\fsl-CFM-All+P & \pmval{30.28}{33.54} $\rightarrow$ \pmval{95.72}{1.03} & \pmval{49.24}{38.96} $\rightarrow$ \pmval{94.65}{0.41} & \pmval{28.54}{24.07} $\rightarrow$ \pmval{78.87}{21.20} \\

\fsl-CFM-All+A+P & \pmval{27.57}{34.27} $\rightarrow$ \pmval{95.88}{0.35} & \pmval{58.32}{38.92} $\rightarrow$ \pmval{94.59}{0.30} & \pmval{12.68}{6.16} $\rightarrow$ \pmval{88.64}{3.92} \\
\bottomrule
\end{tabular}
\end{adjustbox}
\vspace{-1em}
\end{center}
\end{table}

\begin{table}
\caption{Mean validation accuracy over five generated classifiers after reward fine-tuning on four datasets. The -All suffix indicates that the CFM was trained on classifiers of CIFAR10, STL10, MNIST, and FMNIST, whereas +A indicates weight augmentation, and +P indicates regularization on padded values. The arrow '$\rightarrow$' indicates the accuracy before (left) and after (right) reward fine-tuning.}
\vspace{-1em}
\label{tab:corrupt-ft-extra}
\tablestyle{2.0pt}{1.1}
\begin{center}
\begin{adjustbox}{max width=\textwidth}
\begin{tabular}{l @{\hspace{10pt}} cccc}
\toprule
  &  \multicolumn{1}{c}{{CIFAR10}} & \multicolumn{1}{c}{{STL10}} & \multicolumn{1}{c}{{MNIST}} & \multicolumn{1}{c}{{FMNIST}}\\
\midrule
\fsl-CFM-All & \pmval{28.92}{11.19} $\rightarrow$ \pmval{61.90}{0.22} & \pmval{42.21}{11.68} $\rightarrow$ \pmval{52.63}{0.14} & \pmval{60.77}{27.58} $\rightarrow$ \pmval{97.56}{1.20} & \pmval{43.12}{34.22} $\rightarrow$ \pmval{88.49}{1.09} \\

\fsl-CFM-All+A & \pmval{45.07}{15.64} $\rightarrow$ \pmval{56.24}{1.65} & \pmval{27.21}{11.44} $\rightarrow$ \pmval{50.66}{1.71} & \pmval{37.67}{37.07} $\rightarrow$ \pmval{95.84}{0.53} & \pmval{55.49}{33.85} $\rightarrow$ \pmval{85.65}{2.76} \\
\fsl-CFM-All+P & \pmval{50.18}{18.49} $\rightarrow$ \pmval{60.29}{0.62} & \pmval{22.05}{6.32} $\rightarrow$ \pmval{52.11}{0.36} & \pmval{30.28}{33.54} $\rightarrow$ \pmval{95.72}{1.03} &\pmval{32.07}{29.62} $\rightarrow$ \pmval{86.93}{1.97} \\
\fsl-CFM-All+A+P & \pmval{39.25}{17.68} $\rightarrow$ \pmval{58.85}{1.33} & \pmval{21.27}{5.06} $\rightarrow$ \pmval{50.49}{0.77} & \pmval{27.57}{34.27} $\rightarrow$ \pmval{95.88}{0.35} & \pmval{46.20}{34.08} $\rightarrow$ \pmval{84.21}{5.30} \\
\bottomrule
\end{tabular}
\vspace{-2em}
\end{adjustbox}
\end{center}
\end{table}

\paragraph{Experiments on the support of classifier weights.}
To get a sense of how the weight distributions change as the dataset changes, see Table \ref{tab:corrupt-ft}. In this experiment, we reward fine-tuned the \fsl-CFM meta-model on increasingly corrupted versions of the base training dataset. The affect of the corruption is noticeable on the support as reward fine-tuning, which is constrained within the support set, fails to reach the accuracy of conventional fine-tuning. Indeed, we find accuracies to be bounded above, often far below the validation accuracy obtained from SGD fine-tuning for the most corrupted data. This holds true even for mild corruption schemes, suggesting the ideal classifier support on the corrupted set has little intersection with the original support, indicating narrowness of the set. To see that different classifiers have mostly disjoint supports, we try expanding $\text{supp }p_1^{\text{base}}$ by training on classifiers for different datasets. To verify this, we trained a new \fsl-CFM model on classifiers for different datasets \textit{without} any context conditioning. Since the fine-tuned target distribution ought to classify only one dataset, the hope is for fine-tuning to redirect the velocity field towards this one classifier distribution. This intuition turns out to be insufficient, as shown by the \fsl-CFM-All rows in the table, further supporting our hypothesis. Moreover, this result holds with weight augmentations. The results also suggest that context conditioning is necessary for consistent validation accuracies, as convincingly shown in the MNIST case. Indeed, the padding regularization is an implicit form of context conditioning as the MNIST classifiers---expecting one-channel images---are slightly smaller than the 3-channel dataset classifiers. Further reinforcing our hypothesis, we also provide results of a \fsl-CFM model trained on all four datasets fine-tuned to generate classifiers for each in Table \ref{tab:corrupt-ft-extra}.

\paragraph{Computation times.} On a NVIDIA A40 GPU, one full training run takes about 2 and a half hours. During evaluation, we sample generated weights 5 times and validate on the test dataset; this test completes in under 5 minutes.

\begin{table}[h]
\caption{True positive rate at the 5\% significance level (TPR@5) and area under receiver operating characteristic curve (AUROC) for detection of harmful covariate shift on CIFAR10.1 and Camelyon17. We test on both the disagreement rate (DAR) and the entropy, setting $\lambda = \kappa/(|\tQ| + 1)$. The best result for each column and our method are \textbf{bolded}.}

\label{tab:extra-meta-detectron}
\tablestyle{3pt}{1.0}
\begin{center}
\begin{adjustbox}{max width=\textwidth}
\begin{tabular}{l @{\hspace{10pt}}ccc ccc}
\toprule
\textbf{TPR@5} & \multicolumn{3}{c}{CIFAR10} & \multicolumn{3}{c}{Camelyon}\\
\cmidrule(lr){2-4} \cmidrule(lr){5-7}
 $|\tQ|$ & 10 & 20 & 50 & 10 & 20 & 50 \\
\midrule
Detectron (DAR), $\kappa = 1$ & 0 & 0 & 0 & 0 & \pmval{.10}{.10} & 0 \\
Detectron (DAR), $\kappa$ match & 0 & 0 & 0 & \pmval{.30}{.15} & \pmval{.20}{.13} & \textbf{\pmval{.50}{.17}} \\

\textbf{Meta-detectron (DAR)} & \pmval{.33}{.13} & \pmval{.47}{.13} &	\pmval{.27}{.12} & \textbf{\pmval{.80}{.11}} & \textbf{\pmval{.40}{.13}} & \pmval{.42}{.15}\\

Detectron (Entropy), $\kappa = 1$ & \textbf{\pmval{.60}{.17}} & \pmval{.40}{.16} & \pmval{.50}{.17} & 	0	& 0	 & 0 \\
Detectron (Entropy), $\kappa$ match & \pmval{.50}{.17} & \pmval{.10}{.10} & \pmval{.20}{.13} & \pmval{.10}{.10} & 	0	& 0 \\
\textbf{Meta-detectron (Entropy)} & \pmval{.27}{.12}	 & \textbf{\pmval{.93}{.07}} & \textbf{\pmval{.93}{.07}} & 0 &  0 & \pmval{.27}{.12} \\

\toprule

\textbf{AUROC} & \multicolumn{3}{c}{CIFAR10} & \multicolumn{3}{c}{Camelyon}\\
\cmidrule(lr){2-4} \cmidrule(lr){5-7}
$|\tQ|$ & 10 & 20 & 50 & 10 & 20 & 50 \\
\midrule
Detectron (DAR), $\kappa = 1$ & 0.515	& 0.595	 & 0.485 &	0.590 & 0.595 & 0.795 \\
Detectron (DAR), $\kappa$ match & 0.495 & 0.485 & 0.560	& 0.690 & 0.795	& \textbf{0.935} \\
\textbf{Meta-detectron (DAR)} & \textbf{0.849} & 0.838 &	0.938 & \textbf{0.900}	& 0.760 & 0.806 \\

Detectron (Entropy), $\kappa = 1$ & 0.740	& 0.695 &	0.850 &	0.345 & 	0.610 &	0.720 \\
Detectron (Entropy), $\kappa$ match & 0.735 &	0.730 &	0.820 & 	0.510 & 	0.455 & 0.600 \\
\textbf{Meta-detectron (Entropy)} & 0.716 &\textbf{0.987} &	\textbf{0.996}	& 0.747 &	\textbf{0.836} & 0.847 \\

\bottomrule
\end{tabular}
\end{adjustbox}
\vspace{-1em}
\end{center}
\end{table}

\begin{table}[h]
\caption{In-distribution validation accuracy before and after reward fine-tuning.}

\label{tab:meta-detectron-val-acc}
\tablestyle{3pt}{1.0}
\begin{center}
\begin{adjustbox}{max width=\textwidth}
\begin{tabular}{l @{\hspace{10pt}}ccc}
\toprule
 & \multicolumn{3}{c}{CIFAR10} \\
\cmidrule(lr){2-4} 
 $|\tQ|$ & 10 & 20 & 50 \\
\midrule
Meta-detectron ($\tP^*$) & \pmval{61.81}{0.17} $\rightarrow$ \pmval{61.01}{0.14} & \pmval{60.97}{0.16} $\rightarrow$ \pmval{60.49}{0.25} & \pmval{60.64}{0.19} $\rightarrow$ \pmval{60.32}{0.30} \\
Meta-detectron ($\tP^*$), $\kappa$ vary & \pmval{60.78}{0.21} $\rightarrow$ \pmval{61.09}{0.17} & \pmval{60.97}{0.16} $\rightarrow$ \pmval{60.49}{0.25} & \pmval{60.78}{0.09} $\rightarrow$ \pmval{60.38}{0.24} \\

Meta-detectron ($\tQ$) & \pmval{60.84}{0.27} $\rightarrow$ \pmval{60.85}{0.27} & \pmval{61.61}{0.30} $\rightarrow$ \pmval{61.11}{0.13} & \pmval{60.90}{0.15} $\rightarrow$ \pmval{61.15}{0.28} \\
Meta-detectron ($\tQ$), $\kappa$ vary & \pmval{60.97}{0.15} $\rightarrow$ \pmval{60.76}{0.26} & \pmval{61.61}{0.30} $\rightarrow$ \pmval{61.11}{0.13} & \pmval{61.14}{0.15} $\rightarrow$ \pmval{60.59}{0.26} \\

\midrule

 & \multicolumn{3}{c}{Camelyon} \\
\cmidrule(lr){2-4}
 $|\tQ|$ & 10 & 20 & 50 \\
\midrule
Meta-detectron ($\tP^*$) & \pmval{92.75}{0.24} $\rightarrow$ \pmval{91.38}{0.72} & \pmval{92.84}{0.21} $\rightarrow$ \pmval{91.71}{0.59} & \pmval{92.63}{0.15} $\rightarrow$ \pmval{92.39}{0.22} \\	
Meta-detectron ($\tP^*$), $\kappa$ vary & \pmval{92.10}{0.35} $\rightarrow$ \pmval{92.12}{0.46} & \pmval{92.09}{0.54} $\rightarrow$ \pmval{91.95}{0.48} &  \pmval{92.31}{0.21} $\rightarrow$ \pmval{90.43}{0.82} \\	

Meta-detectron ($\tQ$) & \pmval{92.85}{0.27} $\rightarrow$ \pmval{92.47}{0.21} & \pmval{92.60}{0.19} $\rightarrow$ \pmval{92.44}{0.34} & \pmval{93.03}{0.18} $\rightarrow$ \pmval{92.70}{0.42} \\
Meta-detectron ($\tQ$), $\kappa$ vary & \pmval{92.36}{0.34} $\rightarrow$ \pmval{91.89}{0.66} & \pmval{90.61}{1.45} $\rightarrow$ \pmval{91.38}{0.23} &  \pmval{92.66}{0.24} $\rightarrow$ \pmval{91.01}{0.36}\\

\bottomrule
\end{tabular}
\end{adjustbox}
\vspace{-1em}
\end{center}
\end{table}

\subsection{Detecting harmful covariate shifts}
\label{app:meta-detectron-training}

\paragraph{Training.} The only settings that were changed from reward fine-tuning is the learning rate, which is now $1.5 \times 10^{-5}$ and the number of fine-tuning epochs $N=100$. The computation time varies between CIFAR10 and Camelyon17. The former completes in 2 hours, whereas the latter requires 3 and a half hours. The difference stems from the higher image resolution of Camelyon17, resulting in more parameters in the classifier.

\paragraph{Choosing $\lambda$.} Following the exposition in \citet{ginsberg2023harmfulcov}, the choice of $\lambda$ can be motivated by a counting argument. We suppose that agreeing with the base classifier on a sample of $\tP$ incurs a reward of 1 and disagreeing on a sample of $\tQ$ incurs a reward of $\lambda$. Originally, to encourage agreement of $\tP$ as the primary objective, $\lambda$ is set so that the reward obtained from disagreeing on \textit{all} samples of $\tQ$ is less than agreeing on just \textit{one} sample of $\tP$, i.e. $\lambda|\tQ| < 1$, giving $\lambda = \frac{1}{|\tQ| + 1}$. However, this argument can be generalized slightly. As reward fine-tuning is a more conservative approach, we want to increase the reward for disagreeing on $\tQ$. For instance, we may want the reward obtained from disagreeing on \textit{all} samples of $\tQ$ to be about the same as agreeing on $\kappa > 0$ samples of $\tP$. This gives $\lambda = \kappa/(|\tQ|+1)$ as the $\ell_{dce}$ weight. 

\paragraph{Choosing $\kappa$.}
We tried $\kappa=1, 3, 6, 9, |\tQ|+1$ for Camelyon17 and $\kappa=1, 32, 40, 50, |\tQ|+1$ for CIFAR. We ran two experiments with different weight settings. In our fixed run, seen in Tables \ref{tab:extra-meta-detectron} and \ref{tab:meta-detectron-val-acc}, we used $\kappa=4$ for Camelyon17 and $\kappa=32$ for CIFAR10. The reason for the lower $\kappa$ values for Camelyon17 stems from the number of classes: recall that Camelyon17 is requires a binary classifier, whereas there are 10 classes in the CIFAR10 dataset. The dataset batch size also matters: we used a batch size of 128 for Camelyon17 (as the images are larger) and 512 for CIFAR10. In the run where we varied $\kappa$ over the sample size $|\tQ|$, we started at a reference point: $\kappa=32$ for CIFAR10 at $|\tQ|=20$ and scaled naturally, i.e. $\kappa=16$ when $|\tQ|=10$ and $\kappa=80$ when $|\tQ|=50$. Likewise, we used the reference $\kappa=4$ for Camelyon17 at $|\tQ|=20$ and scaled accordingly.

\paragraph{Shift detection test.} We used a standard two-sample test identical to \citet{ginsberg2023harmfulcov}. Given the original $\tP$ and the unknown $\tQ$, we would like to rule out the null hypothesis $\tP = \tQ$ at the 5\% significance level by comparing two statistics: \textit{entropy} and the \textit{disagreement rate}. The definition of entropy we use measures uncertainty over classes in the logits. 
\begin{equation}
    \label{eq:detectron-entropy}
    \text{Entropy}(x) = \sum_{c=1}^N \hat p_c \log \hat p_c \quad \text{where }\hat p = \frac{f(x) + g(x)}{2},
\end{equation}
where $f$ is our base classifier and $g$ is the generated classifier. In contrast to Detectron, we do not use CDC ensembles in our method. We use a Kolmogorov-Smirnov test to compute the p-value for covariate shift on the entropy distributions, comparing $g_{\tP^*}$ and $g_{\tQ}$. Intuitively, when $\tQ$ is out-of-distribution, the generated classifier predicts with high entropy on $\tQ$ and low on $\tP^*$.

Regarding the disagreement rate, the null hypothesis is represented by $\E[\phi_\tQ] \leq \E[\phi_{\tP^*}]$ where $\phi$ is the disagreement rate and expectation is taken over trial seeds. This comes from the idea that its easier to learn to reject from a distribution that is not in the training set (since the base classifier $f$ will also be unsure). The test result is considered significant at $\alpha\%$ when $\phi_\tQ$ is greater than the $(1-\alpha)$ quantile of $\phi_{\tP^*}$. In practice $\alpha = 5\%$.

\begin{figure}
    \centering
    \includegraphics[width=0.49\linewidth]{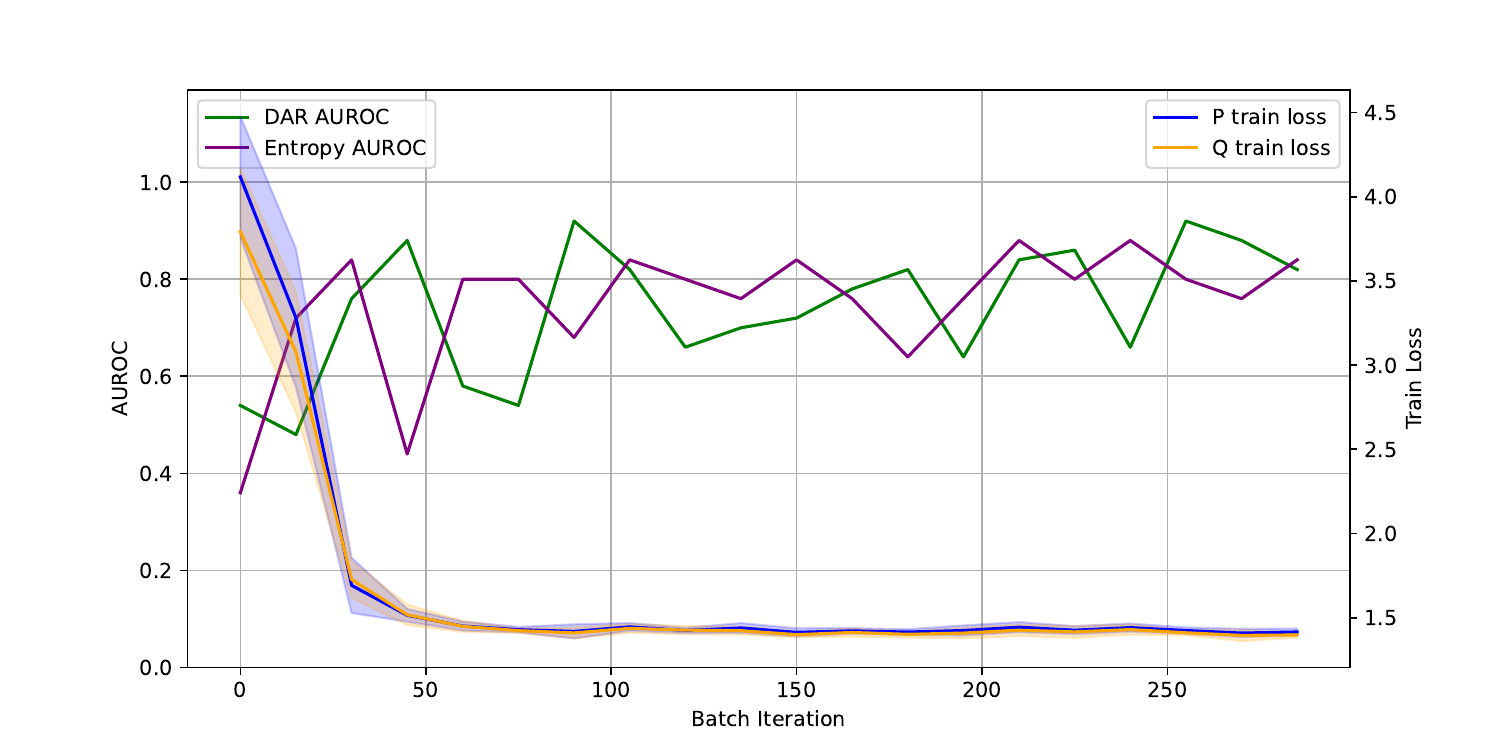}
    \includegraphics[width=0.49\linewidth]{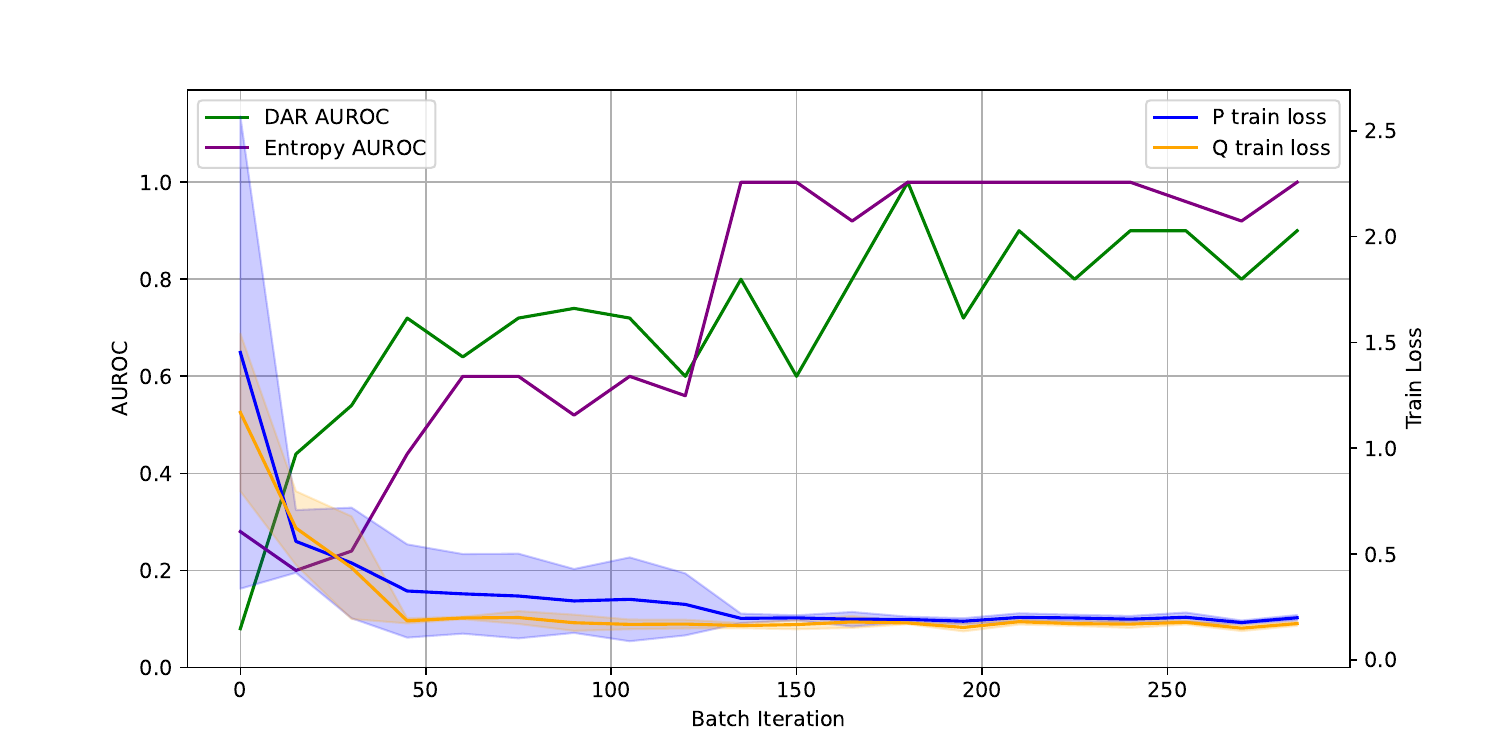}
    \includegraphics[width=0.49\linewidth]{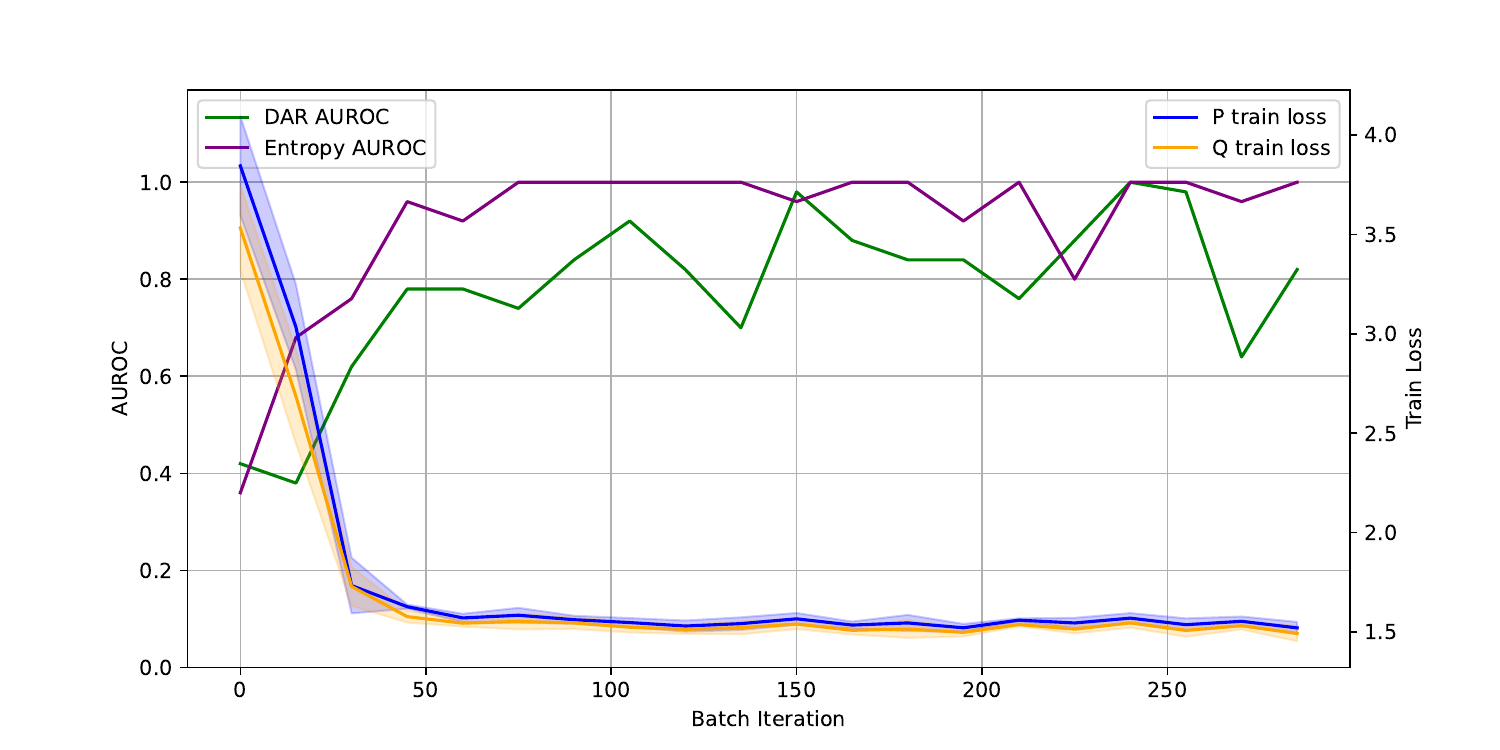}
    \includegraphics[width=0.49\linewidth]{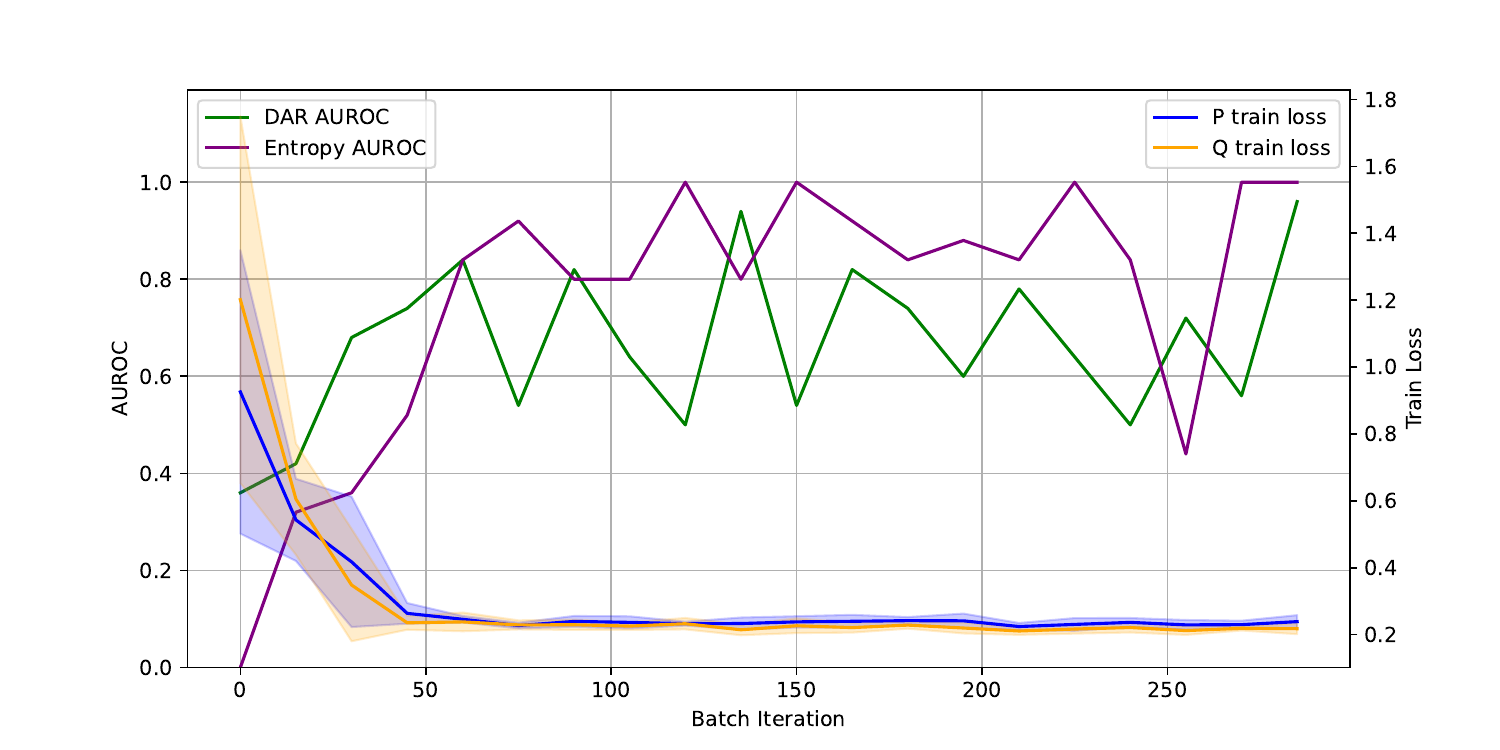}
    \includegraphics[width=0.49\linewidth]{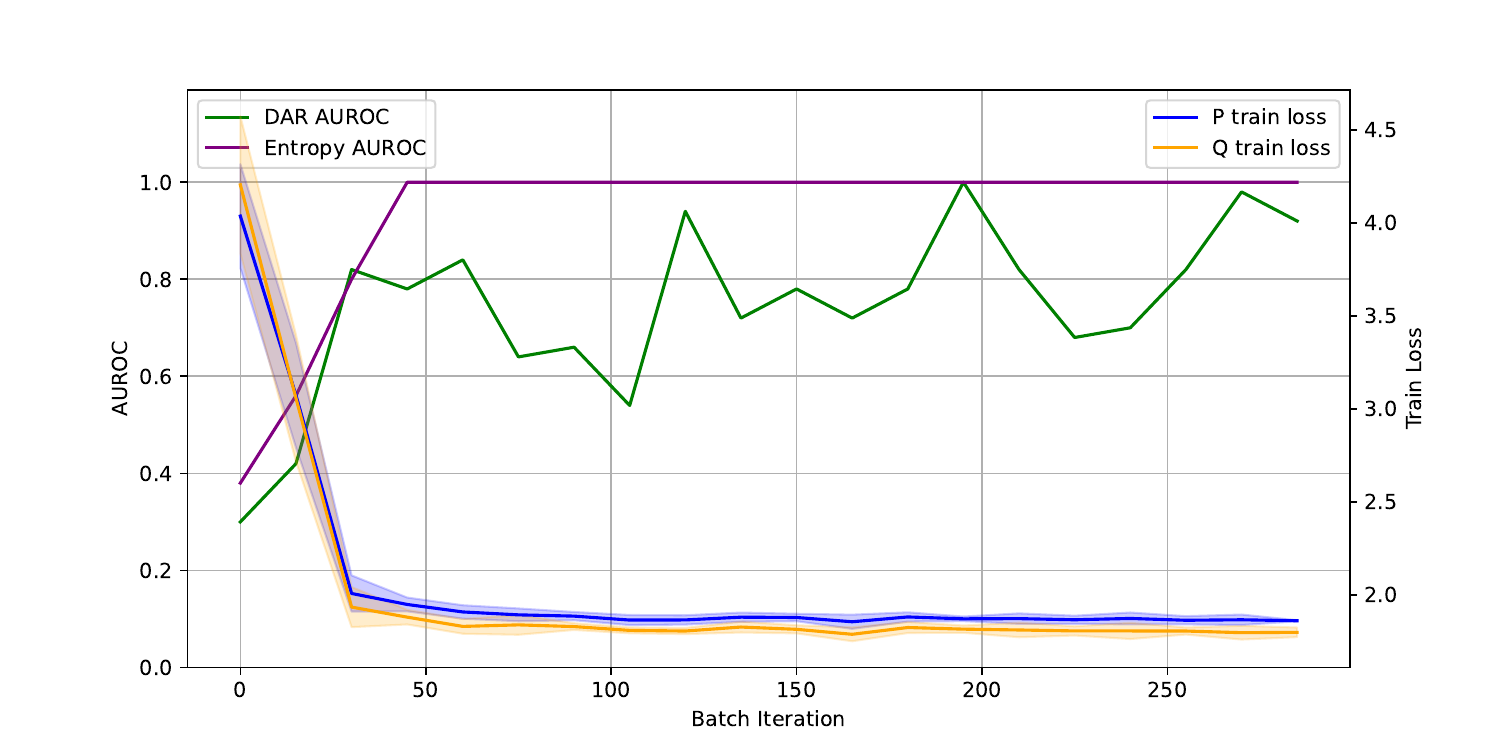}
    \includegraphics[width=0.49\linewidth]{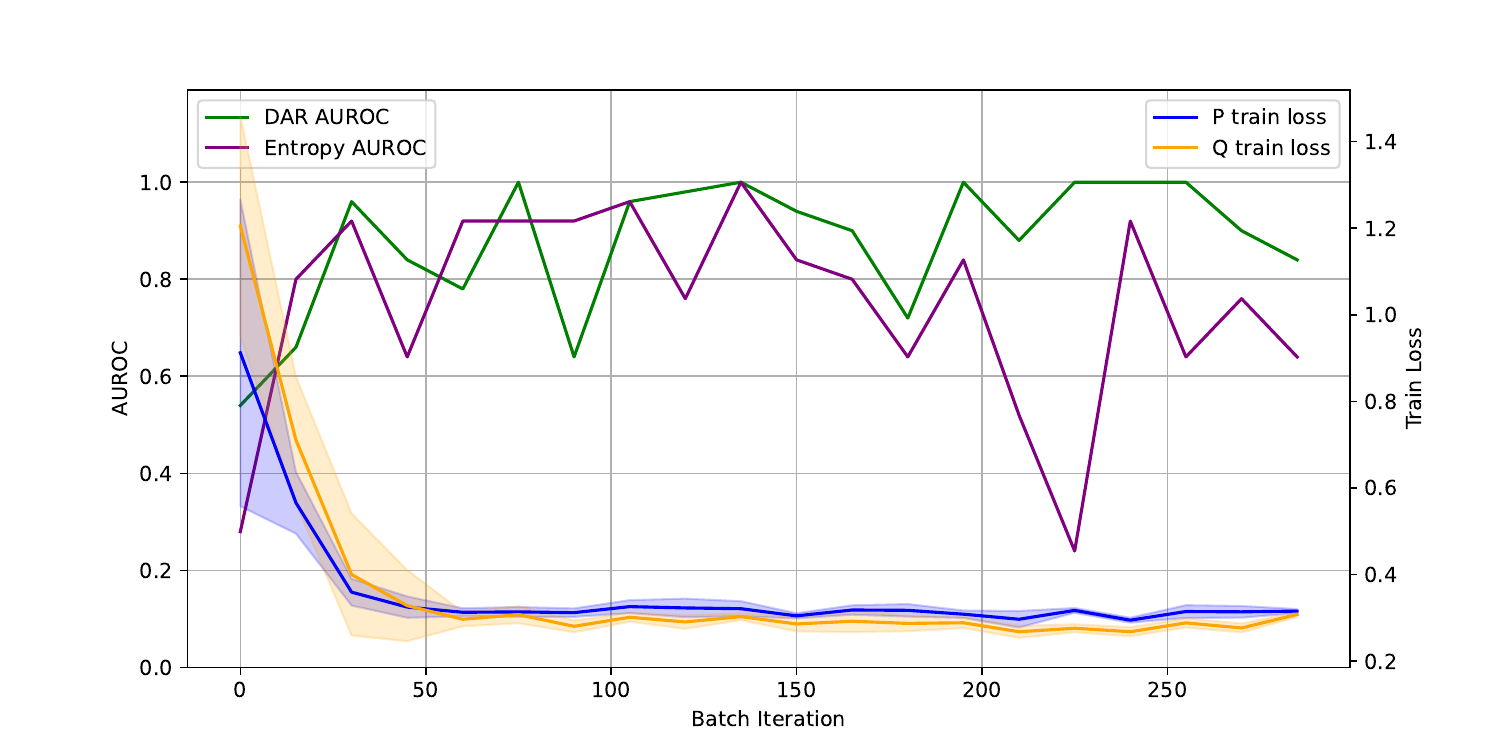}
    \caption{Plots illustrating how AUROC and $\ell_{cdc}$ evolves over reward fine-tuning iterations for CIFAR10 and Camelyon17 when $|\tQ|=10, 20, \, 50$ for \textit{varying} $\kappa$ ordered from top to bottom.}
    \label{fig:extra-vary-auroc-and-loss}
\end{figure}

\begin{figure}
    \centering
    \includegraphics[width=0.49\linewidth]{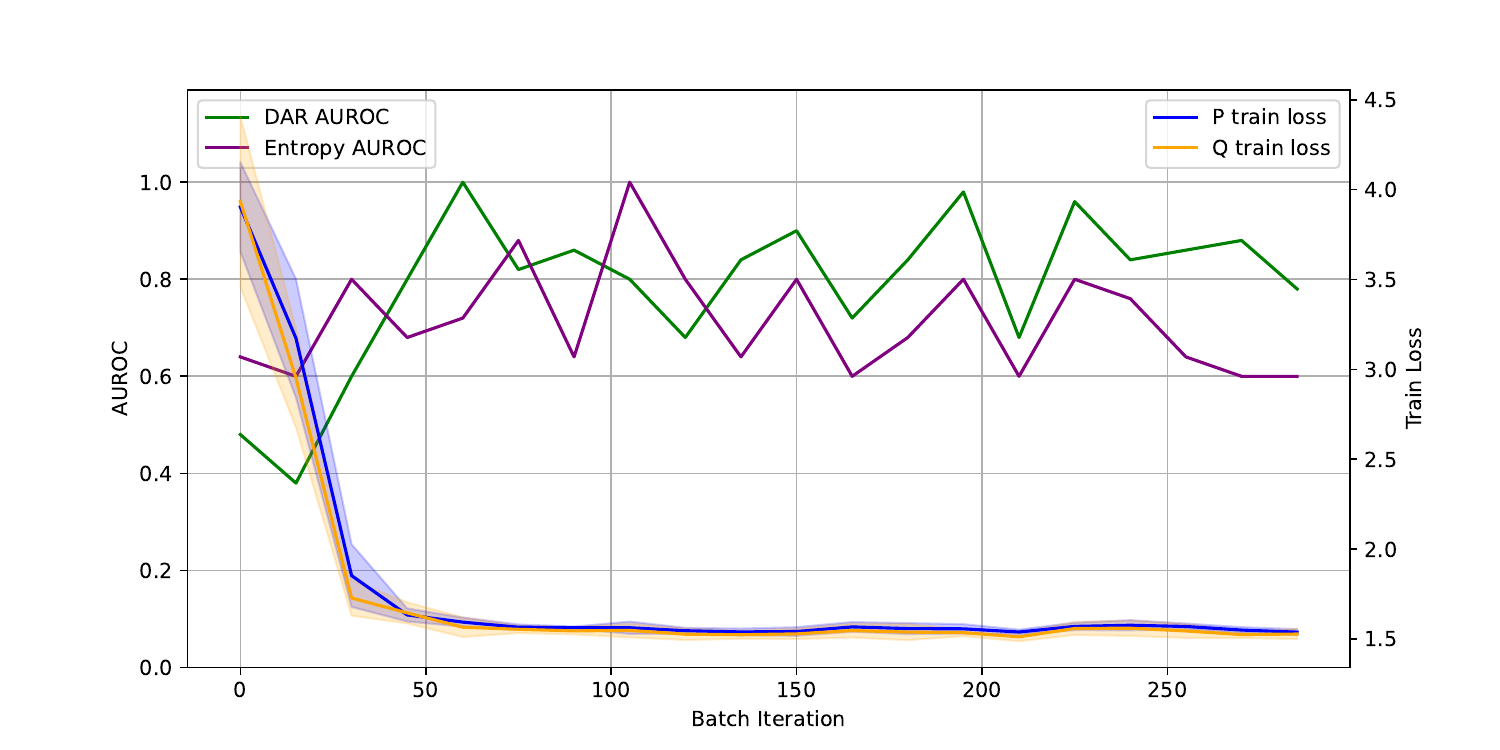}
    \includegraphics[width=0.49\linewidth]{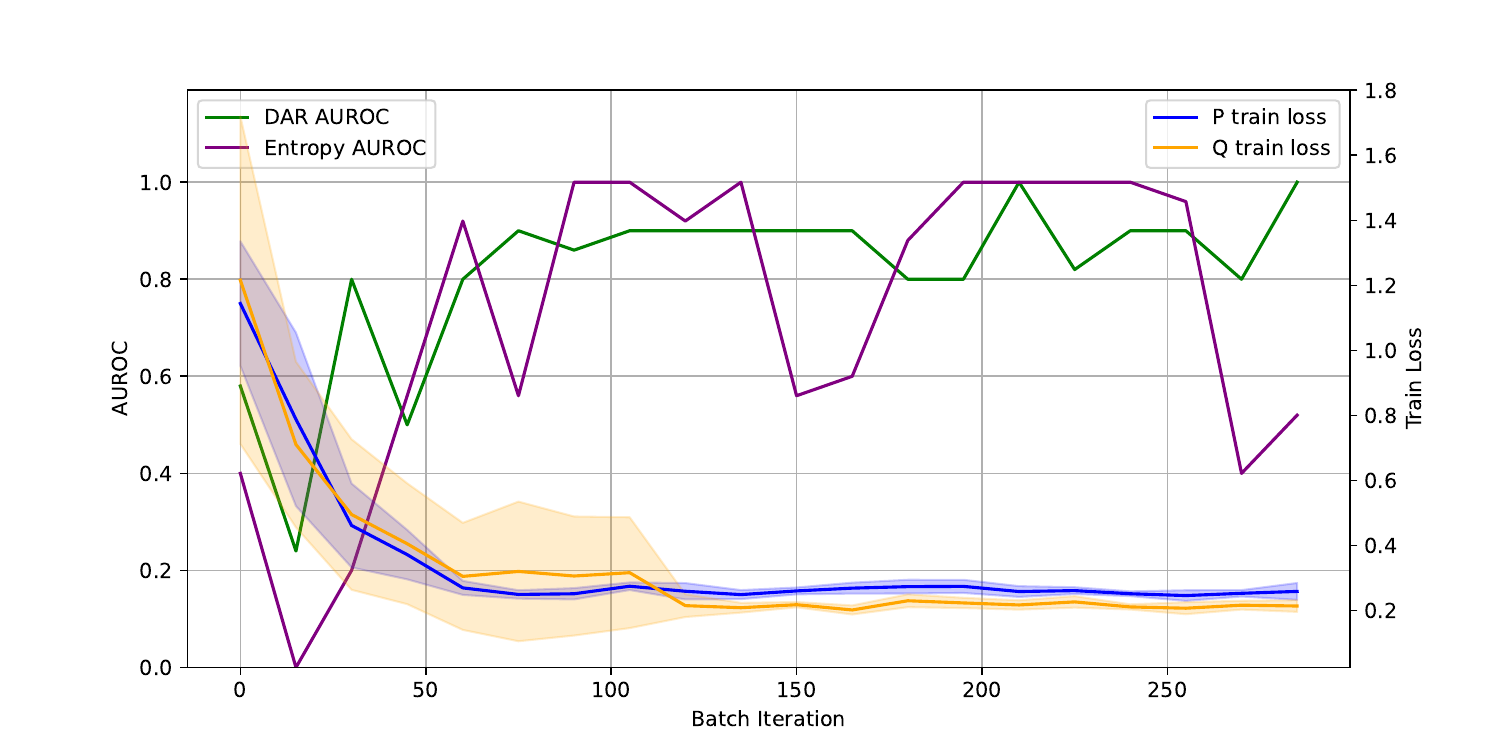}
    \includegraphics[width=0.49\linewidth]{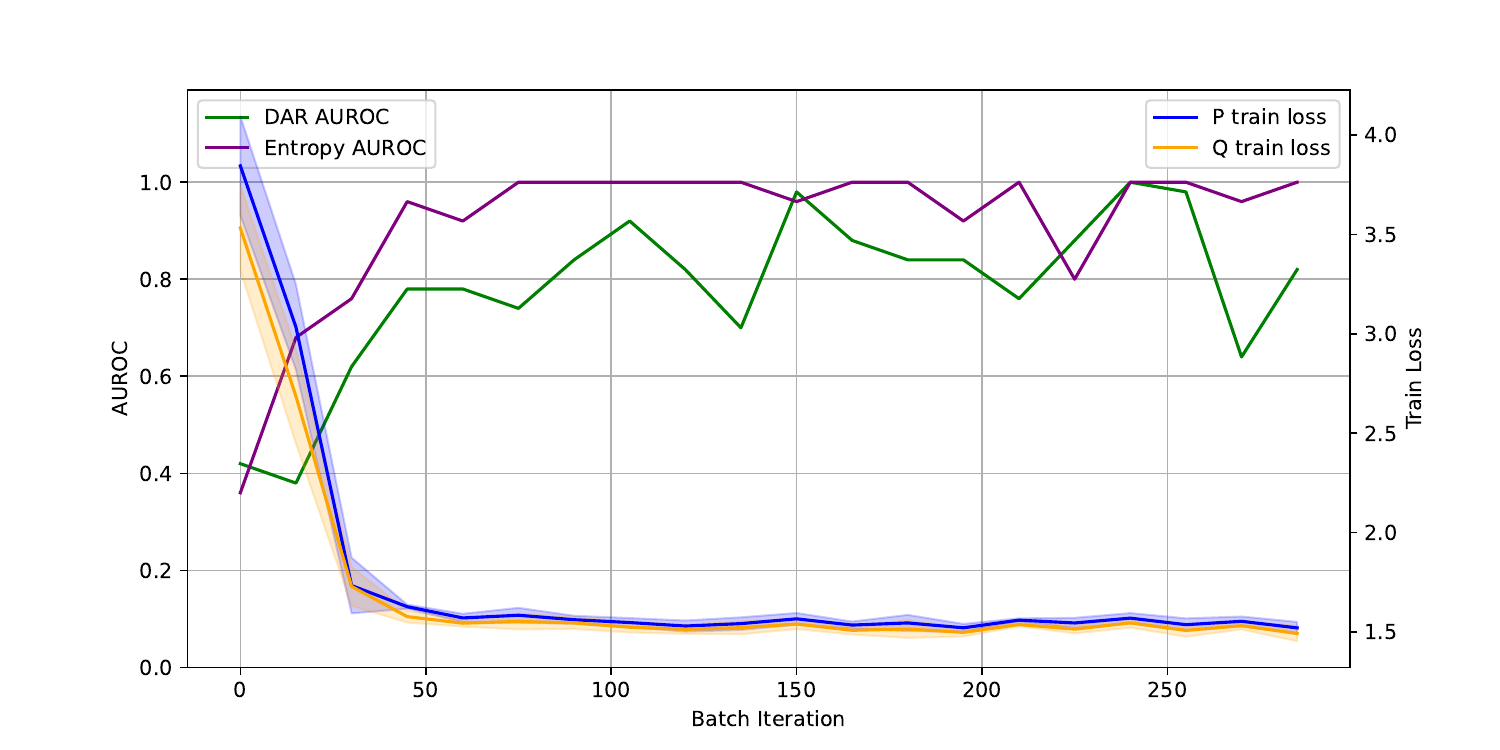}
    \includegraphics[width=0.49\linewidth]{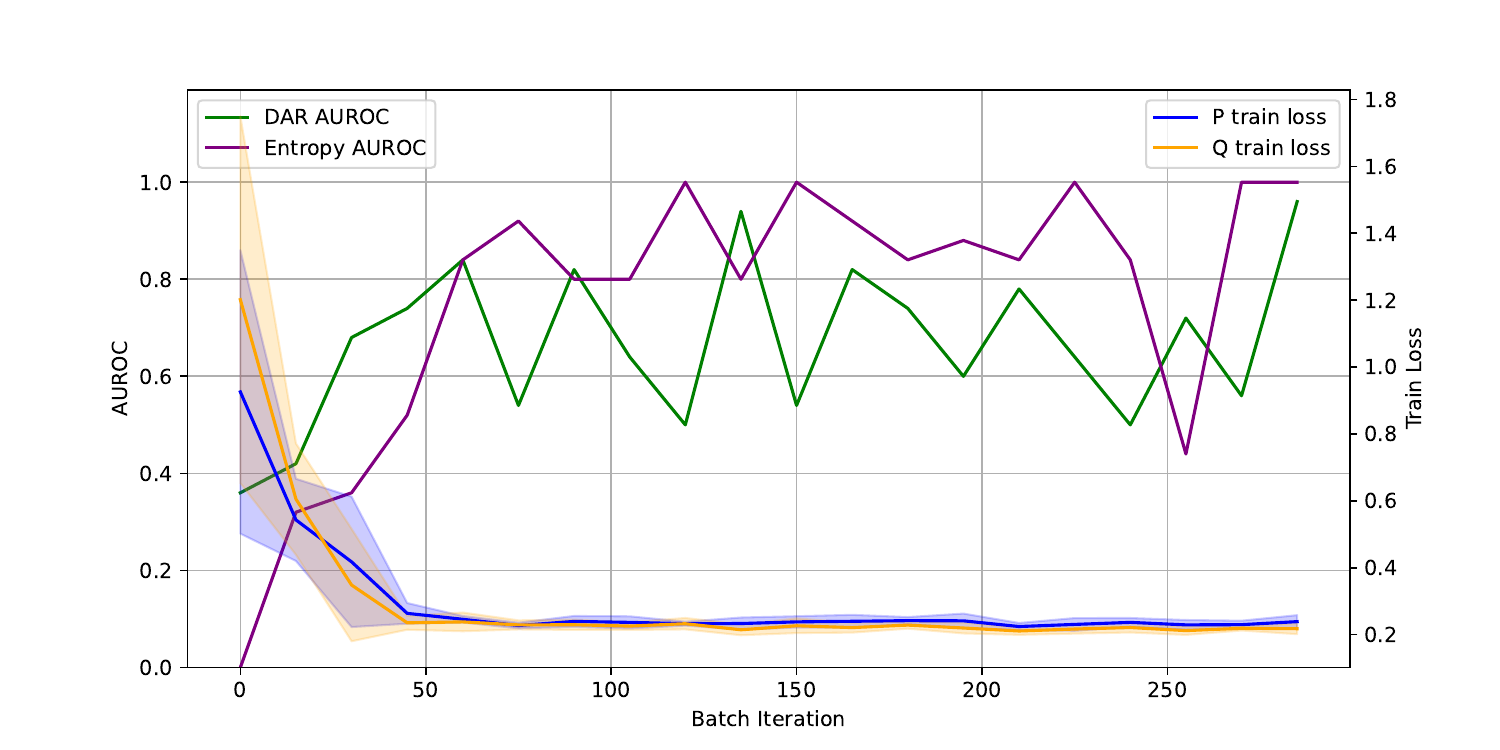}
    \includegraphics[width=0.49\linewidth]{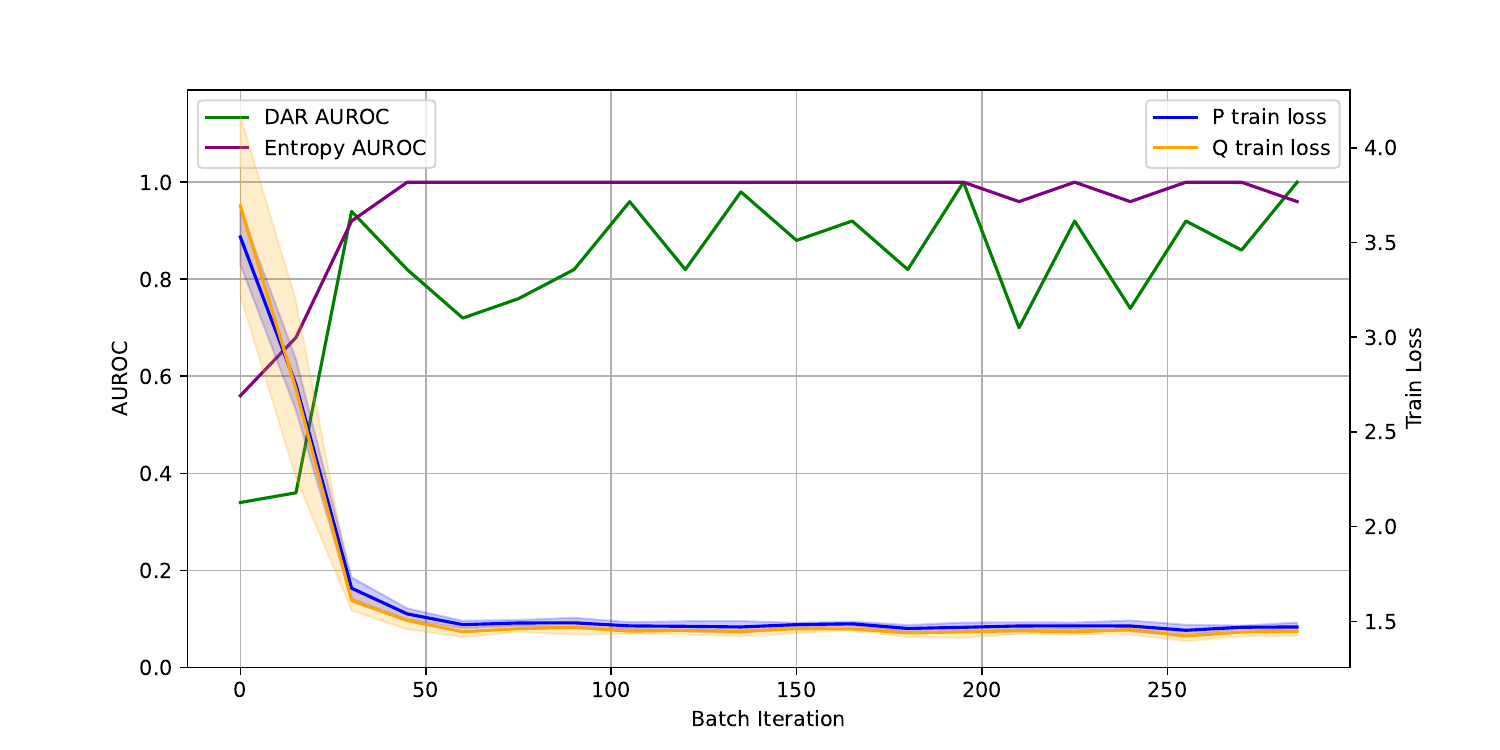}
    \includegraphics[width=0.49\linewidth]{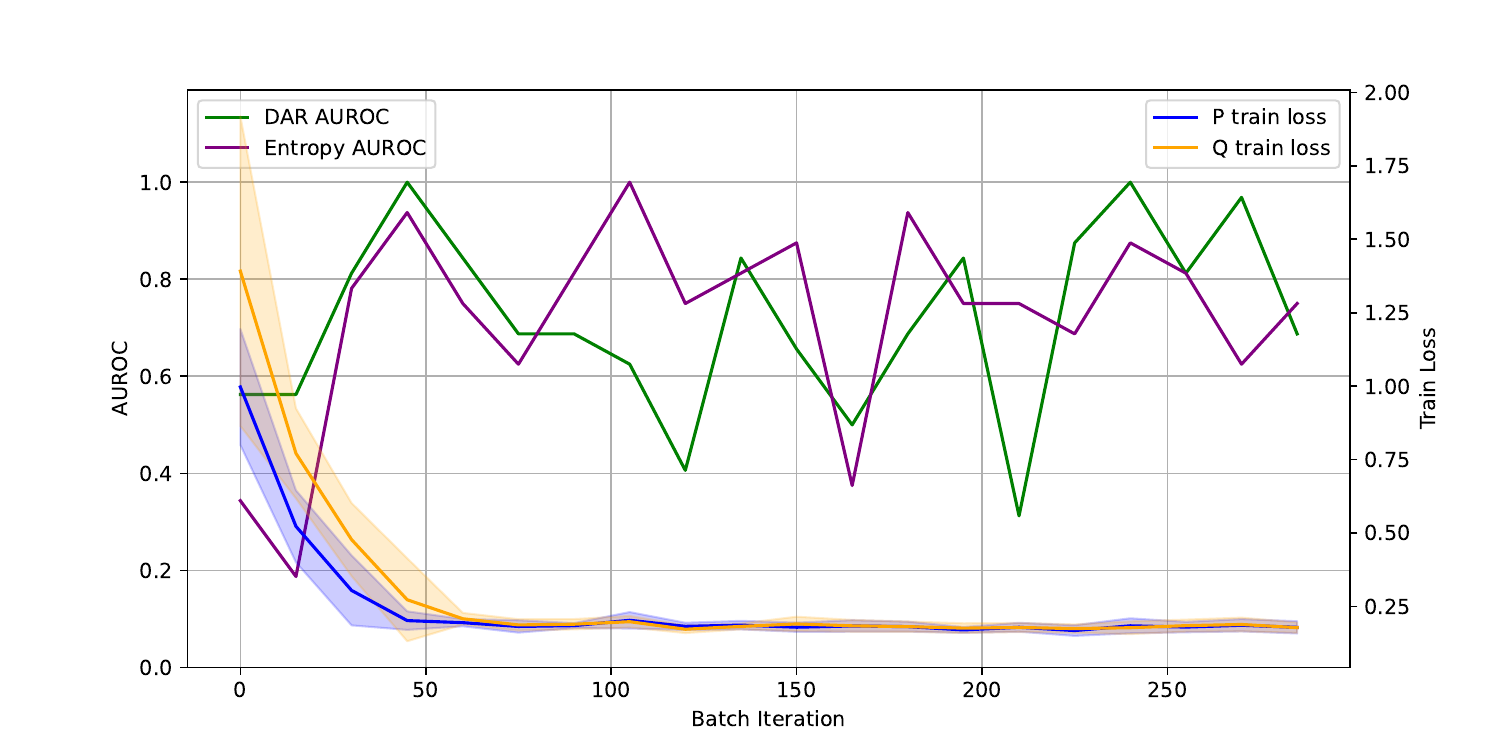}
    \caption{Plots illustrating how AUROC and $\ell_{cdc}$ evolves over reward fine-tuning iterations for CIFAR10 and Camelyon17 when $|\tQ|=10, 20, \, 50$ for \textit{fixed} $\kappa$ ordered from top to bottom.}
    \label{fig:extra-fixed-auroc-and-loss}
\end{figure}

\paragraph{Further results.} See Tables \ref{tab:extra-meta-detectron} and \ref{tab:meta-detectron-val-acc} for the results when $\kappa$ is fixed. We also show more plots in Figures \ref{fig:extra-vary-auroc-and-loss} and \ref{fig:extra-fixed-auroc-and-loss}.

\begin{figure}
    \centering
    \includegraphics[width=0.49\linewidth]{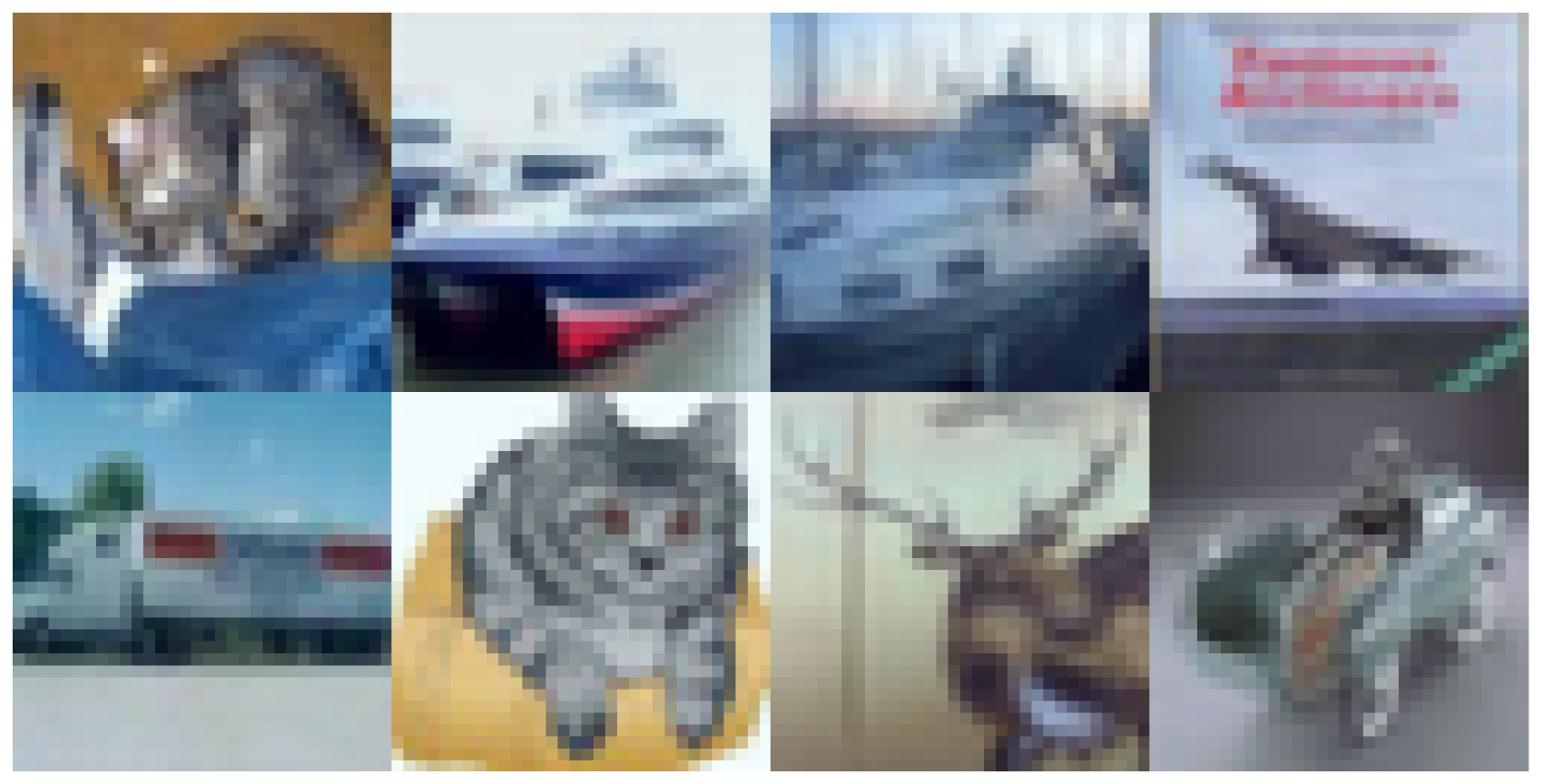}
    \includegraphics[width=0.49\linewidth]{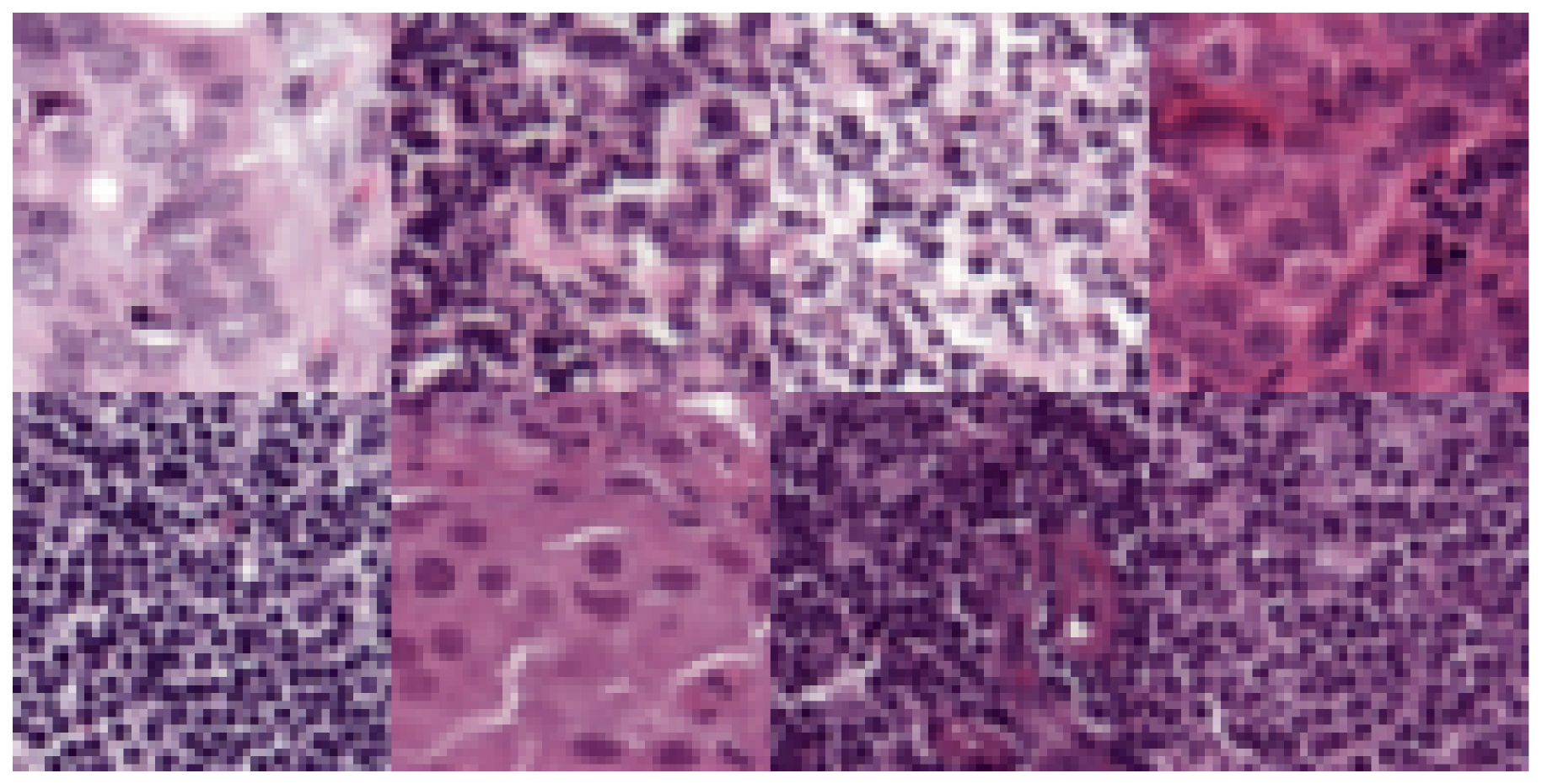}
    \caption{Samples of $\tP^*$ (\textbf{top}) and $\tQ$ (\textbf{bottom}) from the CIFAR10 (\textbf{left}) and Camelyon17 (\textbf{right}) evaluation of Meta-Detectron.}
    \label{fig:meta-detectron-samples}
\end{figure}

\end{document}